\documentclass{article}

\usepackage{microtype}
\usepackage{graphicx}
\usepackage{diagbox}
\usepackage{footnote}
\makesavenoteenv{figure}

\usepackage{hyperref}

\usepackage[accepted]{icml2023}
\usepackage{xcolor}
\usepackage{algorithm,algorithmic}
\usepackage{amssymb,multirow,paralist,color,amsthm}
\usepackage{amsmath}
\usepackage{subfigure}
\usepackage{epstopdf}
\usepackage{booktabs} 
\usepackage{makecell}

\def \S {\mathbf{S}}
\def \A {\mathcal{A}}

\def \O {\mathcal{O}}

\def \R {\mathbb{R}}

\def \w {\mathbf{w}}
\def \v {\mathbf{v}}
\def \t {\mathbf{t}}
\def \x {\mathbf{x}}

\def \E {\mathrm{E}}

\def \x {\mathbf{x}}
\def \p {\mathbf{p}}
\def \a {\mathbf{a}}

\def \b {\mathbf{b}}
\def \e {\mathbf{e}}

\def \d {\mathbf{d}}

\def \1 {\mathbf{1}}

\def \z {\mathbf{z}}
\def \s {\mathbf{s}}

\def \u {\mathbf{u}}

\def \B {\mathcalB}

\def \E {\mathrm{E}}
\def \x {\mathbf{x}}

\def \D {\mathcal{D}}
\def \z {\mathbf{z}}
\def \u {\mathbf{u}}

\def \w {\mathbf{w}}
\def \s {\mathbf{s}}

\def \R {\mathbb{R}}
\def \S {\mathcal{S}}

\def \W {\mathcal{W}}

\def \A {\mathcal{A}}
\def \q {\mathbf{q}}
\def \v {\mathbf{v}}

\def \d {\mathbf{d}}
\def \p {\mathbf{p}}
\def \q {\mathbf{q}}

\def \a {\mathbf{a}}
\def \b {\mathbf{b}}

\def \B {\mathcal{B}}

\def \s {\mathbf{s}}

\def \t {\mathbf{t}}

\def \T {\mathcal{T}}

\def \E {\mathbb{E}}
\def \bftau {\boldsymbol{\tau}}

\newtheorem{thm}{Theorem}

\newtheorem{ass}{Assumption}
\newtheorem{lemma}{Lemma}
\newtheorem{definition}{Definition}

\DeclareMathOperator*{\argmin}{arg\,min}

\usepackage[shortlabels]{enumitem}
\setlist[itemize]{noitemsep}
\setlist[itemize]{nolistsep}

\usepackage[textsize=tiny]{todonotes}

\icmltitlerunning{}

\begin{document}

\twocolumn[
\icmltitle{Not All Semantics are Created Equal: Contrastive Self-supervised Learning with Automatic Temperature Individualization}



\begin{icmlauthorlist}
\icmlauthor{Zi-Hao Qiu}{nju}
\icmlauthor{Quanqi Hu}{tamu}
\icmlauthor{Zhuoning Yuan}{ui}
\icmlauthor{Denny Zhou}{google}
\icmlauthor{Lijun Zhang}{nju}
\icmlauthor{Tianbao Yang}{tamu}
\end{icmlauthorlist}

\icmlaffiliation{nju}{National Key Laboratory for Novel Software Technology, Nanjing University, Nanjing, China}
\icmlaffiliation{tamu}{Computer Science and Engineering, Texas A\&M University, College Station, USA}
\icmlaffiliation{ui}{Department of Computer Science, the University of Iowa, Iowa City, USA}
\icmlaffiliation{google}{Google Research, USA}

\icmlcorrespondingauthor{Tianbao Yang}{tianbao-yang@tamu.edu}

\icmlkeywords{Machine Learning, ICML}

\vskip 0.3in
]



\printAffiliationsAndNotice{Most work of Z.H. Qiu was done when visiting the OptMAI lab at TAMU. }  

\begin{abstract}
In this paper, we aim to optimize a contrastive loss with \emph{individualized} temperatures in a principled and systematic
manner for self-supervised learning. The common practice of using a \emph{global} temperature parameter $\tau$ ignores the fact that ``not all semantics are created equal", meaning that different anchor data may have different numbers of samples with similar semantics, especially when data exhibits long-tails. First, we propose a new {\it robust contrastive loss} inspired by distributionally robust optimization (DRO), providing us an intuition about the effect of $\tau$ and a mechanism for automatic temperature individualization. Then, we propose an efficient stochastic algorithm for optimizing the robust contrastive loss with a provable convergence guarantee without using large mini-batch sizes. Theoretical and experimental results show that our algorithm automatically learns a suitable $\tau$ for each sample. Specifically, samples with frequent semantics use large temperatures to keep local semantic structures, while samples with rare semantics use small temperatures to induce more separable features. Our method not only outperforms prior strong baselines (e.g., SimCLR, CLIP) on unimodal and bimodal datasets with larger improvements on imbalanced data but also is less sensitive to hyper-parameters. To our best knowledge, this is the first methodical approach to optimizing a contrastive loss with individualized temperatures. 
\end{abstract}

\section{Introduction}
\label{sec:introduction}
Self-supervised learning (SSL) is a promising way to learn data representations that generalize across downstream tasks. Specifically, contrastive learning (CL) has laid the foundation for state-of-the-art SSL models due to its effectiveness~\citep{chen2020simple,he2020momentum,ReLiCv2,huang2022contrastive}. CL aims to push the similarity scores between ``positive" pairs (e.g., augmented views of the same image) to be higher than that between ``negative" pairs (e.g., augmented views from different images), which has great promises in leveraging large amount of unlabelled data~\citep{goyal2021self,radford2021learning}. Moreover, CL has been extended to a broader scope, e.g., bimodal image-text SSL~\citep{zhang2020contrastive,radford2021learning}, where images and language text descriptions can be regarded as multi-modal views of the same underlying concept. The well-known CLIP~\citep{radford2021learning} method shows that models learned from millions of image-text pairs can attain impressive recognition performance for a wide range of visual understanding tasks.

\begin{figure*}[t]
\begin{center}
\begin{minipage}[c]{0.48\textwidth}
\centering\includegraphics[width=1\textwidth]{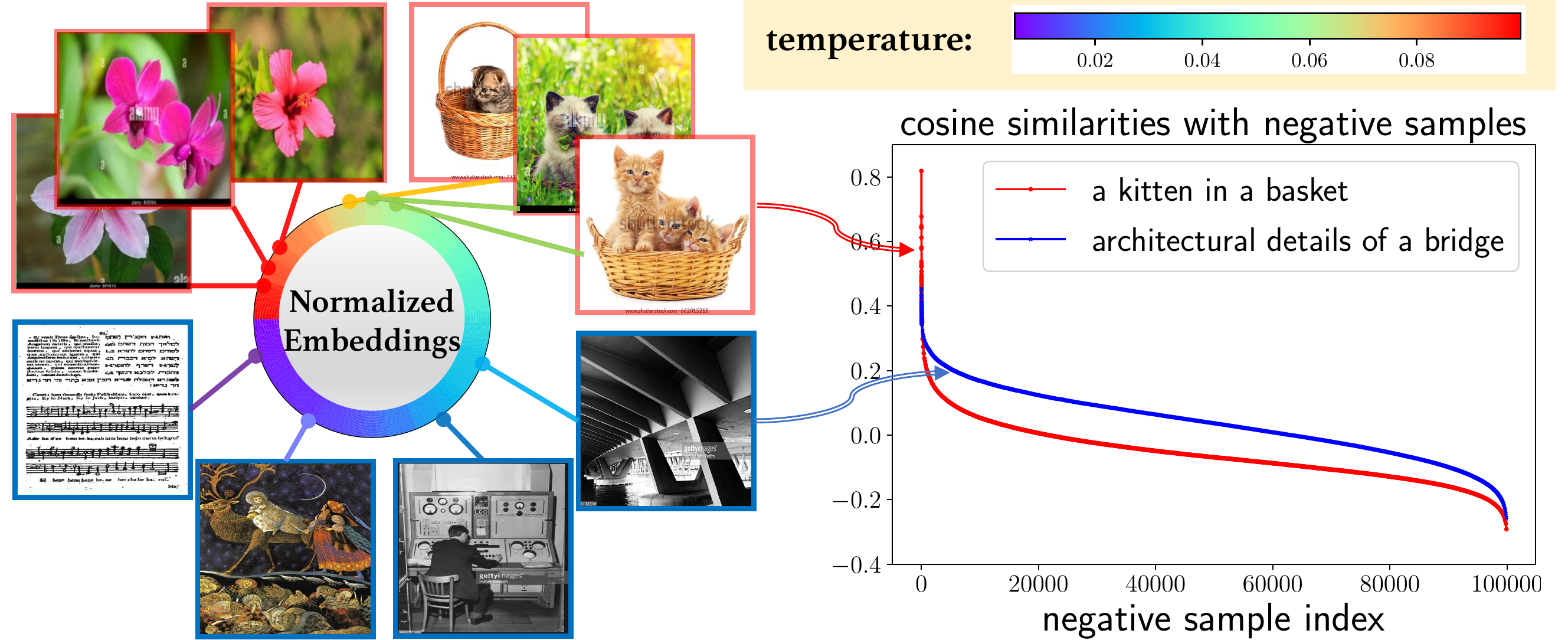}
\end{minipage}
\begin{minipage}[c]{0.24\textwidth}
\centering\includegraphics[width=1\textwidth]{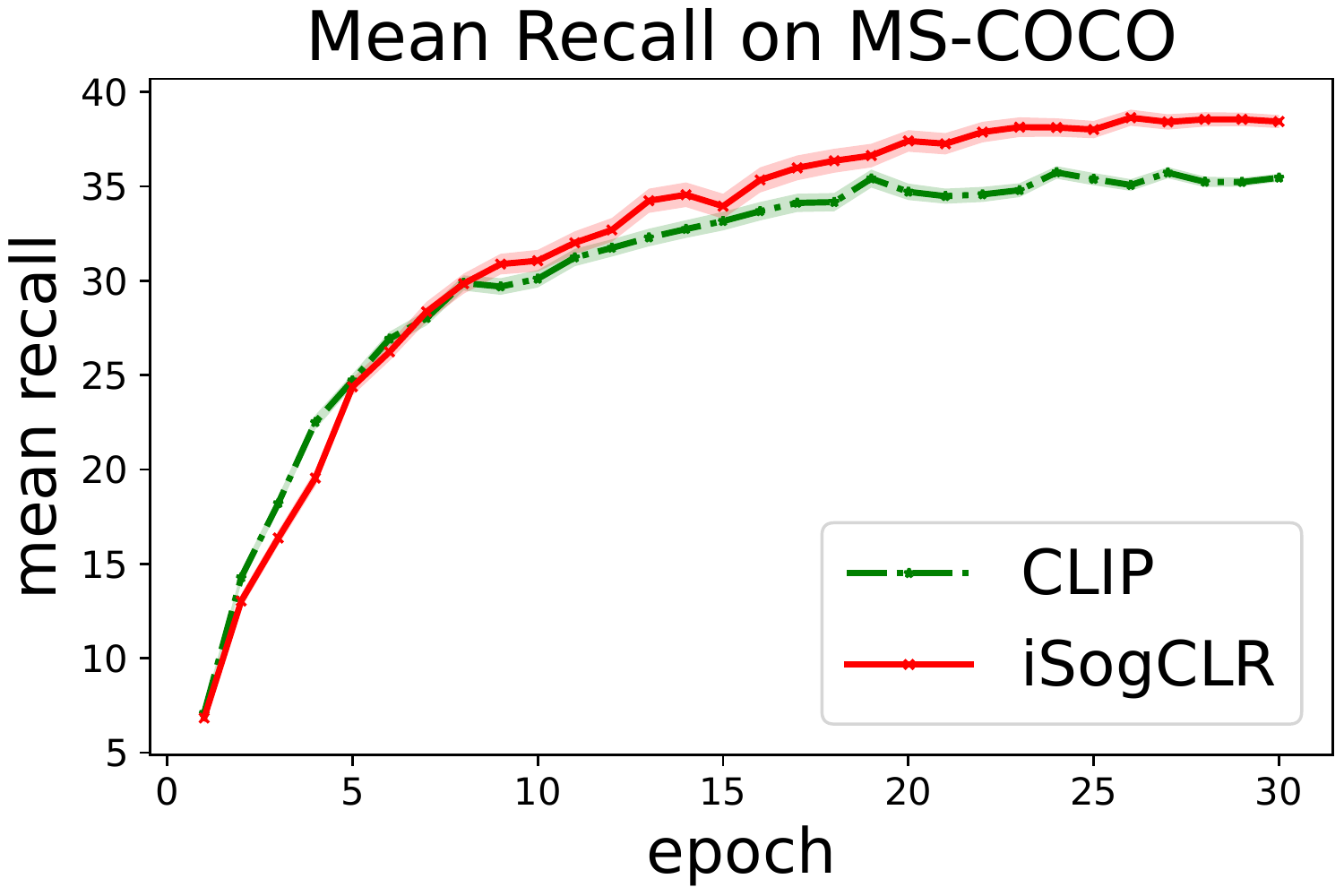}
\end{minipage}
\begin{minipage}[c]{0.24\textwidth}
\centering\includegraphics[width=1\textwidth]{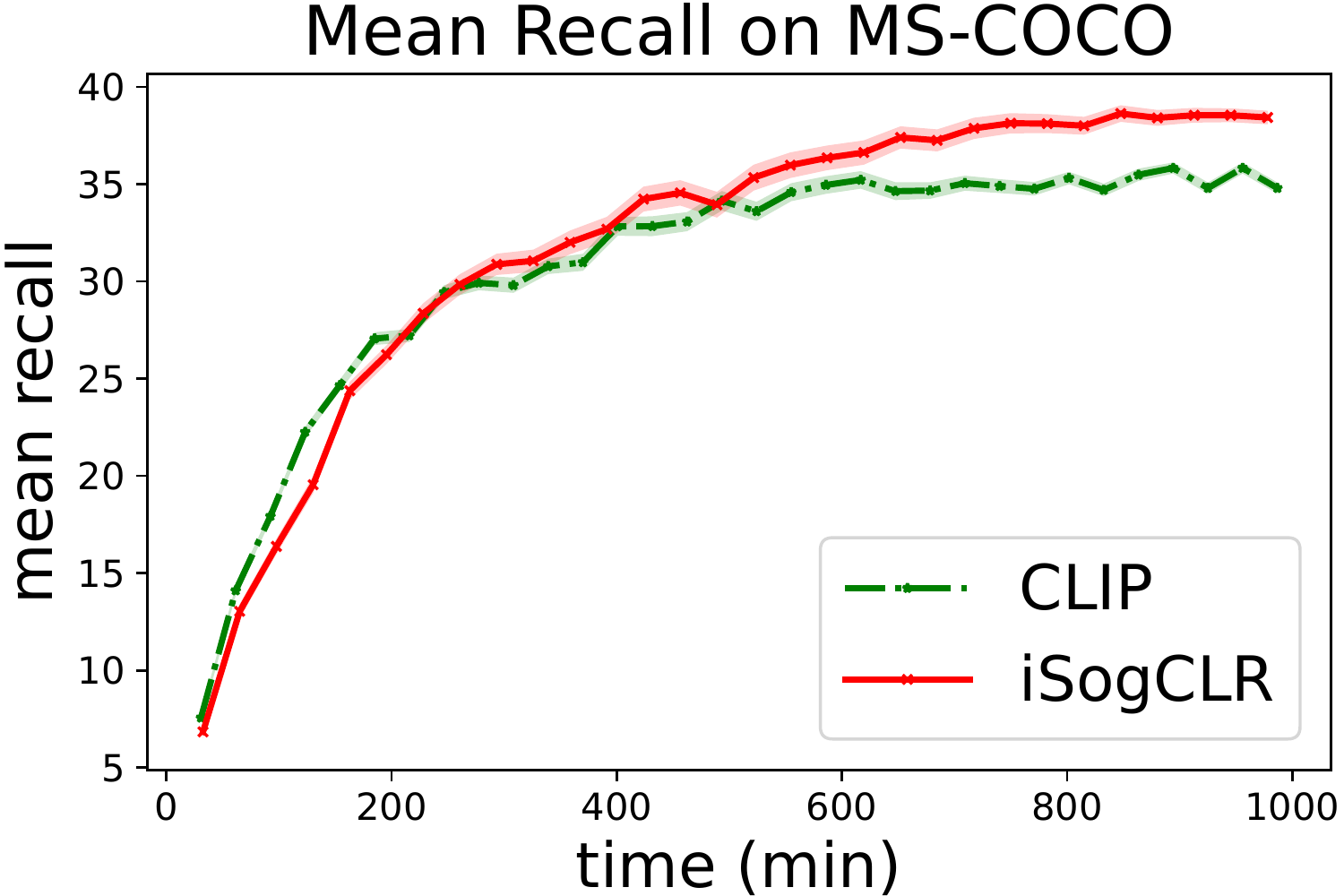}
\end{minipage}
\vspace*{-0.1in}
\caption{Left: Samples with \emph{frequent} semantics (e.g., a kitten in a basket) have many more similar samples than that with \emph{rare} semantics (e.g., architectural details of a bridge). Middle: An illustration of temperature individualization by our algorithm named iSogCLR, making ``hot" images with frequent semantics use a higher temperature to keep semantic structures, and making  ``cold" images  with rare semantics use a lower temperature for inducing more separable features. The circular heatmap is plotted using learned temperatures of 100 random images from the CC3M dataset. Right: Convergence curves of CLIP and iSogCLR, where we use the CLIP implementation and training settings from open-clip~\protect\footnotemark. We train the models on CC3M and evaluate image retrieval performance on MS-COCO.}
\label{fig:motivating_example}
\end{center}
\vspace*{-0.28in}
\end{figure*}
\footnotetext{\url{https://github.com/mlfoundations/open_clip}}

In general, contrastive methods share a common design of the softmax-based loss function,
For a given anchor data $i$, a contrastive loss can be written as:
{\setlength\abovedisplayskip{0pt}
\setlength\belowdisplayskip{0pt}
\begin{equation}\label{eqn:bc}
    \mathcal{L}_{\text{con}}^{i} = - \log\frac{\exp{(\text{sim}(z_i,z_i^+)}/\tau)}{\sum_{k\neq i}\exp{(\text{sim}(z_i,z_k)}/\tau)},
\end{equation}}where $z_i$ is the feature of the anchor data, $z_i^+$ is the feature of a different `view' of data $i$ and called a \emph{positive} sample, $z_k (k\neq i)$ are the features of other samples and called \emph{negative} samples, $\tau$ is the temperature parameter, and $\text{sim}(\cdot,\cdot)$ measures the similarity between two input vectors. The positive pair can be also added to the denominator, which does not affect our discussion here. A significant property of the loss is the \textbf{hardness-aware} property~\citep{wang2021understanding,zhang2022dual,xia2022progcl}. Consider the gradient of $\tau\mathcal{L}_{\text{con}}^{i}$ w.r.t. model parameters $\w$, i.e., $\tau\nabla_{\w}\mathcal{L}_{\text{con}}^{i}$:
{\setlength\abovedisplayskip{0pt}
\setlength\belowdisplayskip{0pt}
\begin{align*}
\small
    -\nabla_{\w}\text{sim}(z_i,z_i^+)\!+\!\sum_{k\neq i}\! \frac{\exp{(\text{sim}(z_i,z_k)}/\tau)}{\underset{k\neq i}{\sum}\exp{(\text{sim}(z_i,z_k)}/\tau)}\!\nabla_{\w}\text{sim}(z_i,z_k).
\end{align*}}Note that the weight for the gradient of a negative pair $(z_i,z_k)$ is proportional to $\exp{(\text{sim}(z_i,z_k)}/\tau)$. Thus, the contrastive loss automatically penalizes negative pairs according to their hardness (hard means $\text{sim}(z_i,z_k)$ is large). The temperature $\tau$ plays a critical role in controlling the penalty strength on negative samples~\citep{wang2021understanding,zhang2022dual}. Specifically, a small $\tau$ penalizes much more on hard negative samples (i.e., the degree of hardness-awareness is high), causing separable embedding space. However, the excessive pursuit to the separability may break the underlying \emph{semantic structures} because some negative samples with high similarity scores to the anchor data might indeed contain similar semantics, to which we refer as pseudo negatives. In contrast, a large $\tau$ tends to treat all negative pairs equally (i.e., the degree of hardness-awareness is low) and is more tolerant to pseudo negative samples, which is beneficial for keeping local semantic structures.

Real-world data distributions always exhibit long tails~\citep{zhu2014capturing,feldman2020does} and the frequency of samples with different semantics can be extremely diverse. In Figure~\ref{fig:motivating_example}, we show some images with frequent or rare semantics from the CC3M dataset~\cite{sharma2018conceptual}. We further select two representative images, namely  ``a kitten in a basket" that contains frequent semantics and ``architectural details of a bridge" that contains rare semantics, and present the cosine similarities between these two images and other 100,000 random texts from the same dataset. Note that images with frequent semantics have much more similar samples. To improve feature qualities, samples with frequent semantics should be assigned with a \emph{large} $\tau$ to better capture the local semantic structure, while using a small $\tau$ will push semantically consistent samples away.  On the other hand, samples with rare semantics should have a \emph{small} $\tau$ to make their features more discriminative and separable. We refer to these effects as {\it semantics harmonizing}. Unfortunately, most existing CL methods treat the temperature parameter as a \emph{global} parameter, which does not accommodate different semantics and restricts their performance in real-world applications. 


In this paper, we propose a provable stochastic algorithm for optimizing a contrastive loss with individualized temperatures. First, inspired by \emph{distributionally robust optimization (DRO)}~\citep{namkoong2017variance,duchi2021statistics}, we design a novel \emph{robust global contrastive loss (RGCL)} for each anchor data. RGCL introduces a distributional variable for \emph{all negative samples} of each anchor data (this explains ``global'' in RGCL), and a KL divergence constraint between the distributional variable and the uniform distribution. We show that RGCL is hardness-aware by optimizing the distributional variable, and the KL constraint affects the degree of hardness-awareness. We further demonstrate that the dual formulation of RGCL induces a loss function that can be solved efficiently and contains an individualized learnable temperature parameter. 
In a spirit of stochastic optimization of a global contrastive loss (SogCLR)~\cite{yuan2022provable}, we propose an efficient optimization algorithm named \textbf{iSogCLR} for solving the dual formulation of RGCL by synthesizing advanced techniques of compositional optimization~\cite{wang2022finite} and of solving KL constrained DRO~\cite{qi2022stochastic}. 
We establish a convergence guarantee of our algorithm without large mini-batch sizes, which is similar to that of SogCLR for optimizing a global contrastive loss with a fixed $\tau$.  
Experiments on unimodal and bimodal datasets indicate that iSogCLR achieves superior performance, especially on imbalanced data. More in-depth analyses demonstrate the relationship between the semantics of data and their learned temperatures. An illustration of some images and their learned temperatures by our method is plotted in Figure~\ref{fig:motivating_example}. Besides, ablation studies show that iSogCLR is much less sensitive to its hyper-parameters. We summarize our contributions below:
\vspace{-2.5mm}
\begin{itemize}[leftmargin=*]
    \item We propose a new {\it robust contrastive loss} inspired by DRO, and study its properties and connections with existing softmax-based contrastive losses. 
    \vspace*{0.05in}
    \item We propose a novel and provable stochastic algorithm called iSogCLR for optimizing the robust contrastive losses with automatic temperature individualization.
    \vspace*{0.05in}\item We conduct comprehensive experiments on  unimodal and bimodal CL to demonstrate the superior performance of our method, and the relationship between the semantics of data and their learned temperatures.
\end{itemize}

\section{Related Work}
\vspace*{-0.02in}{\noindent\bf Self-supervised Learning.} SSL methods can be divided into two main categories: CL methods and non-CL methods. Although non-CL methods do not rely on negative samples and achieve comparable performance to CL methods, they often require additional projector~\citep{grill2020bootstrap}, stop gradient~\citep{chen2021exploring}, or momentum encoder~\cite{richemond2020byol}. Several non-CL methods~\cite{ermolov2021whitening,zbontar2021barlow,bardes2021vicreg} use information maximization techniques, which decorrelate the variables of each embedding to avoid an informational collapse. 
Nevertheless, CL methods remain a mainstream framework for SSL and are extended to many other fields~\citep{khaertdinov2021contrastive,aberdam2021sequence,eun2020learning}, and are shown to be better than non-CL methods in the human learning setting~\cite{zhuang2022how}.

{\noindent\bf CL Methods.} Pioneering methods~\citep{chen2020simple,he2020momentum} construct pairs and optimize InfoNCE loss~\citep{oord2018representation} in a direct way. Later, numerous works improve the performance by penalizing hard negatives~\citep{wu2020conditional,kalantidis2020hard,chen2021simpler,robinson2021contrastive,xie2022delving,zhang2022dual}, generating better views~\citep{tamkin2020viewmaker,tian2020makes,ge2021robust,wang2021contrastive}, using prototypes (cluster centers) for contrasting~\citep{caron2020unsupervised,li2020prototypical}, and handling pseudo negative samples~\citep{chuang2020debiased,dwibedi2021little}. Recently,~\citet{haochen2021provable} propose a novel loss based on spectral decomposition on the population graph with accuracy guarantees. To address the issue that most CL methods rely on large batch sizes,~\citet{yuan2022provable} propose a provable algorithm named SogCLR for optimizing a global contrastive loss, and achieve promising results without large mini-batch sizes. Although great progress has been made, most methods ignore the imbalance of semantics in real-world data, and lack the ability to adapt to different types of semantics automatically. 

{\noindent\bf Non-CL Methods.} Non-CL methods employ the augmented-view-based paradigm as contrastive methods. But they only consider positive image view pairs, and use different methodologies to avoid all outputs of the network collapse to a constant. A representative class of methods aims to avoid collapse by using tricks inspired knowledge distillation~\citep{hinton2015distilling}. Specifically, a student network is trained to predict the outputs of a teacher network, and the weights for the teacher network are updated by different strategies~\citep{grill2020bootstrap,chen2021exploring}. An alternative class of methods~\citep{ermolov2021whitening,zbontar2021barlow,bardes2021vicreg} relies on maximizing the information content of embeddings. They prevent informational collapse by decorrelating every pair of variables of the embedding vectors. Although these methods achieving promising results on popular tasks, theoretical studies for these methods are still lacking. Besides, it is difficulty to directly extend these methods in bimodal self-supervised learning tasks where the model architectures and data are completely different on different modalities.

{\noindent\bf Bimodal Contrastive Learning.} Vision-and-language pretraining (VLP) is a rapidly growing field. Due to its effectiveness, CL has been extended to representative works such as CLIP~\citep{radford2021learning} and ALIGN~\citep{jia2021scaling}, which are pretrained on millions of web-crawled image-text pairs and achieve astounding results. Later, DeCLIP~\citep{li2021supervision}, FILIP~\citep{yao2021filip}, SLIP~\citep{mu2022slip} and CyCLIP~\citep{goel2022cyclip} improve CLIP by introducing more supervisions or bringing in fine-grained cross-modal interactions. 
Our algorithm tackles a fundamental problem on optimizing individualized temperatures for a contrastive loss. Thus it can be applied in bimodal setting seamlessly. Besides, because the web-crawled bimodal data often exhibits long tail distributions~\citep{wang2022vlmixer}, we observe that our algorithm is more suitable in such scenario and achieves great improvements compared with baselines.

{\noindent\bf Optimizing $\tau$ in CL.} The impact of $\tau$ on the success of CL is remarkable and noticed in prior works.~\citet{wang2021understanding} show that temperature controls the strength of penalties on hard negative samples and describe a uniformity-tolerance dilemma when choosing temperature parameter.~\citet{zhang2022dual,khaertdinov2021contrastive} propose to improve negative mining in CL by using different temperatures for positive and negative samples, where temperatures can be fixed values or input-dependent functions.~
In bimodal CL, CLIP~\citep{radford2021learning} proposes to treat $\tau$ as a learnable variable, which is adopted by later works~\cite{goel2022cyclip,li2021align}. However, this approach was never rigorously justified.~\citet{zhang2021temperature} show that input-dependent learnable $\tau$ is effective to estimate the uncertainty in out-of-distribution detection, but with the cost of sacrificing the performance on downstream tasks. Different from previous methods that set or learn temperatures heuristically, we present a new view of the contrastive loss based on DRO, which explains the role of $\tau$ mathematically, and enables automatic optimization of individualized temperatures.

{\noindent\bf Distributionally Robust Optimization.} DRO has been extensively studied in machine learning and statistics~\citep{bertsimas2018data,staib2019distributionally,duchi2021statistics}. Mathematically, DRO seeks a model that performs well regardless of perturbing the sample distribution within an uncertainty set, which is specified by a divergence measure between the perturbed distribution and the observed empirical distribution~\citep{ben2013robust,blanchet2019robust,duchi2021statistics}. 
Recent works~\cite{qi2020attentional,qi2021stochastic,qi2022stochastic,levy2020large,jin2021non,gurbuzbalaban2022stochastic,zhu2023distributionally} have proposed efficient stochastic algorithms for solving different DRO formulations. Our algorithm is inspired by that of~\citet{qi2022stochastic}, which considers a similar DRO problem with the uncertainty set specified by a KL constraint, and proposes efficient dual-free algorithms with convergence guarantees.  However, different from their work that considers an ordinary compositional objective with only \emph{one} KL constraint,  we deal with a more complex \emph{coupled} compositional function~\citep{wang2022finite} and \emph{many} KL constraints for all anchor data, which complicate the convergence analysis.

\section{Preliminaries}

Let $\D=\{\x_1,\ldots,\x_n\}$ denote a set of training images with size $n$. $\mathcal{P}$ denotes a set of data augmentation operators. Denoted by $\S^-_i=\{\A(\x):\forall\A\in\mathcal{P},\forall\x\in\D\setminus\x_i\}$ the set of negative data for the anchor image $\x_i$. Let $E(\cdot)$ denote the image encoder. For bimodal tasks, let $\D^{\prime}=\{(\x_1,\t_1),\ldots,(\x_n,\t_n)\}$ denote $n$ image-text pairs. Let $\T^-_i=\{\t_j\in\D', j\neq i\}$ be the set of negative texts for the anchor image $\x_i$, and $\mathcal I^-_i=\{\x_j\in\D', j\neq i\}$ be the set of negative images for the anchor text $\t_i$. Let $E_I(\cdot)$ and $E_T(\cdot)$ denote the encoder for images and texts in bimodal CL, respectively. Let $\w$ denote the model parameters. 
Denote by $\Delta_n$ a simplex of dimension $n$ and by $\delta_{\Omega}(\cdot)$ a dirac function that returns zero if input belongs to the set $\Omega$ or infinity otherwise.  Let $\text{KL}(\cdot, \cdot)$ denote the KL divergence. 

For unimodal CL, a \emph{global contrastive loss (GCL)}~\citep{yuan2022provable} for the $i$-th image $\x_i$ can be defined as:
{\setlength\abovedisplayskip{1pt}
\setlength\belowdisplayskip{1pt}
\begin{equation}
\hspace*{-0.1in}    \ell_{\text{GCL}}(\x_i)=-\tau\log\frac{\exp(E(\A(\x_i))^{\top} E(\A^{\prime}(\x_i))/\tau)}{\sum_{\z\in\S^-_i}\exp\left(E(\A(\x_i))^{\top} E(\z)/\tau\right)},
\label{eq:uni_cl_loss}
\end{equation}}where $\A, \A'\in\mathcal{P}$. Compared with a contrastive loss defined over mini-batch samples~\cite{chen2020simple,he2020momentum}, GCL multiplies $\tau$ on the right side to ensure the gradient is not illy scaled. Besides, GCL considers all negative samples $\S^-_i$ for $\x_i$ in the denominator, enabling us to analyze the optimization error and design algorithms to control the error.

To simplify~(\ref{eq:uni_cl_loss}) and facilitate our statements,  we define the following auxiliary function:
{\setlength\abovedisplayskip{1pt}
\setlength\belowdisplayskip{1pt}
\begin{equation}
h_i(\z)\!:=\!E(\A(\x_i))^{\top}\!E(\z)\!-\!E(\A(\x_i))^{\top}\!E(\A^{\prime}(\x_i)).
\label{eq:aux_func_uni_cl_loss}
\end{equation}}In fact, 
$h_i(\z)$ measures the \emph{hardness score} of $\z$ with respect to $\x_i$. Then~(\ref{eq:uni_cl_loss}) can be rewritten as:
{\setlength\abovedisplayskip{1pt}
\setlength\belowdisplayskip{1pt}
\begin{align}
\ell_{\text{GCL}}(\x_i)
&=\tau\log\sum\nolimits_{\z\in\S^-_i}\exp(h_i(\z)/\tau).
\label{eq:uni_cl_loss_simp}
\end{align}}For bimodal tasks, we consider the two-way GCL, which for the $i$-th image-text pair is defined as
{\setlength\abovedisplayskip{1pt}
\setlength\belowdisplayskip{1pt}
\begin{equation*}
\begin{aligned}
    \ell(\x_i,\t_i) =  & -\tau\log\frac{\exp(E_I(\x_i)^{\top}E_T(\t_i)/\tau)}{\sum_{\t\in\T^-_i} \exp(E_I(\x_i)^{\top}E_T(\t)/\tau)}  \\
    &   -\tau\log\frac{\exp(E_I(\x_i)^{\top}E_T(\t_i)/\tau)}{\sum_{\x\in\mathcal I^-_i} \exp(E_I(\x)^{\top}E_T(\t_i)/\tau)},
\end{aligned}
\end{equation*}}where the first term is an image-to-text contrastive loss, i.e., try to predict $\t_i$ from $\t\in\D'$ based on $\x_i$, and the second term is a symmetrical text-to-image contrastive loss.

Similar to~(\ref{eq:uni_cl_loss_simp}), we can simplify $\ell(\x_i,\t_i)$ as:
{\setlength\abovedisplayskip{3pt}
\setlength\belowdisplayskip{3pt}
\begin{equation*}
\small
\ell(\x_i,\t_i)\!=\!\tau\log\!\underset{\t\in\T^-_i}{\sum}\!\exp\!\left[\!\frac{h_{\x_i}(\t)}{\tau}\!\right]\!+\tau\!\log\!\underset{\x\in\mathcal I_i^-}{\sum}\!\exp\!\left[\!\frac{h_{\t_i}(\x)}{\tau}\!\right],
\end{equation*}}where $h_{\x_i}(\t)=E_I(\x_i)^{\top}E_T(\t)-E_I(\x_i)^{\top}E_T(\t_i)$ and $h_{\t_i}(\x)=E_I(\x)^{\top}E_T(\t_i)-E_I(\x_i)^{\top}E_T(\t_i)$. 
Our algorithm 
is applicable to both unimodal and bimodal CL. 

A general DRO formulation is given by~\citep{levy2020large}:  
{\setlength\abovedisplayskip{2pt}
\setlength\belowdisplayskip{2pt}
\begin{equation}
    \min_{\w}\max_{\p\in\mathcal U} \sum\nolimits_{i=1}^n \p_i l_i(\w) - \lambda D(\p,\boldsymbol{1}/n),
\label{eq:literature_dro}
\end{equation}}where $\p$ is a distributional variable, $l_i(\w)$ is the loss on sample $i$, $\mathcal U\subset\Delta_n$ is an uncertainty set of the distributional variable specified by some divergence constraint $D(\p,\boldsymbol{1}/n)$. Maximizing the objective over $\p$ leads to larger weights on samples with larger losses, which actually finds \emph{the worst case} loss. DRO then minimizes the worst-case loss to make models achieve the robustness against potential distribution shifts.

\section{Robust Global Contrastive Objectives}
We first introduce a novel \emph{robust global contrastive loss (RGCL)}, including its formulation and properties. 
Then we convert RGCL into a simpler equivalent minimization form with individualized temperatures by Lagrangian duality theory. We further give a theoretical explanation of our objective for optimizing temperatures. Due to the limited space, we describe our method in unimodal setting. For our method in bimodal setting, we present its final minimization form and defer detailed derivation to Appendix~\ref{app:bimodal_alg}.

\subsection{Formulations}
Motivated by DRO, we define the following loss for an anchor data $\x_i$ with a set of $m$ negative samples $\S_i^-$:
{\setlength\abovedisplayskip{1pt}
\setlength\belowdisplayskip{1pt}
\begin{align}
\hspace*{-0.1in}\ell_{\text{RGCL}}(\x_i):=
& \max_{\mathbf{p} \in \Delta_m} \sum\nolimits_{\z_j \in \mathcal{S}^-_i}\p_j h_i(\z_j)-\tau_0 \text{KL}(\mathbf{p},\boldsymbol{1}/m) \notag\\
&\text {s.t.} \quad \text{KL}(\mathbf{p},\boldsymbol{1}/m) \leq \rho,\label{eq:RCL}
\end{align}}where $m=|\S_i^-|$, $\rho>0$ is a hyperparameter, and $\tau_0$ is a small positive value by default. 
There are several features of~(\ref{eq:RCL}). (i) Similar to the GCL (\ref{eq:uni_cl_loss}), we use \emph{all} negative samples to define the loss. Hence our loss is referred to robust \emph{global} contrastive loss. (ii) We mainly use the KL constraint $\text{KL}(\mathbf{p},\boldsymbol{1}/n)\!\leq\!\rho$ to define an uncertainty set of $\p$. The small KL regularization term $\tau_0 \text{KL}(\mathbf{p},\boldsymbol{1}/n)$ is added to make the loss function  smooth hence facilitate the optimization~\cite{qi2022stochastic}.  
(iii) Different from~(\ref{eq:literature_dro}) that considers a distribution $\p$ over all samples, RGCL considers a distribution $\p$ over all negative samples for \emph{each anchor data}. 

To better illustrate the key properties of RGCL, we plot the contours of $\ell_{\text{RGCL}}(\x_i)$ with two negative data in Figure~\ref{fig:dro_contour}. For $[h_i(\z_1), h_i(\z_2)]$ values in~(\ref{eq:RCL}), we consider two settings, i.e., $[0.0,-1.0]$ and $[-0.5,-0.5]$, which represent $\z_1$ is very similar to the anchor data, and two equally dissimilar negative samples, respectively. The dashed lines in Figure~\ref{fig:dro_contour} are the boundaries of the KL constraints with different $\rho$ values, whose intersections with the red simplex lines define the feasible regions. From these figures, we observe that (i) $\ell_{\text{RGCL}}(\x_i)$ is also \emph{hardness-aware}. By maximizing over $\p$, harder negative samples will have larger weights (e.g., $\p^*=(0.8, 0.2)$ for the first setting when $\rho=0.2$). (ii) If the hardness of two negative samples are similar, then their weights tend to be similar too (for the second setting). (iii) The constraint $\text{KL}(\mathbf{p},\!\boldsymbol{1}/m)\!\leq\!\rho$ actually affects the \emph{degree of hardness-awareness}. In the left of Figure~\ref{fig:dro_contour}, note that if $\rho$ gets larger, the optimal $\p$ will be more non-uniform, i.e., the degree of hardness-awareness will increase. 

\begin{figure}
\begin{minipage}[c]{0.235\textwidth}
\centering\includegraphics[width=1\textwidth]{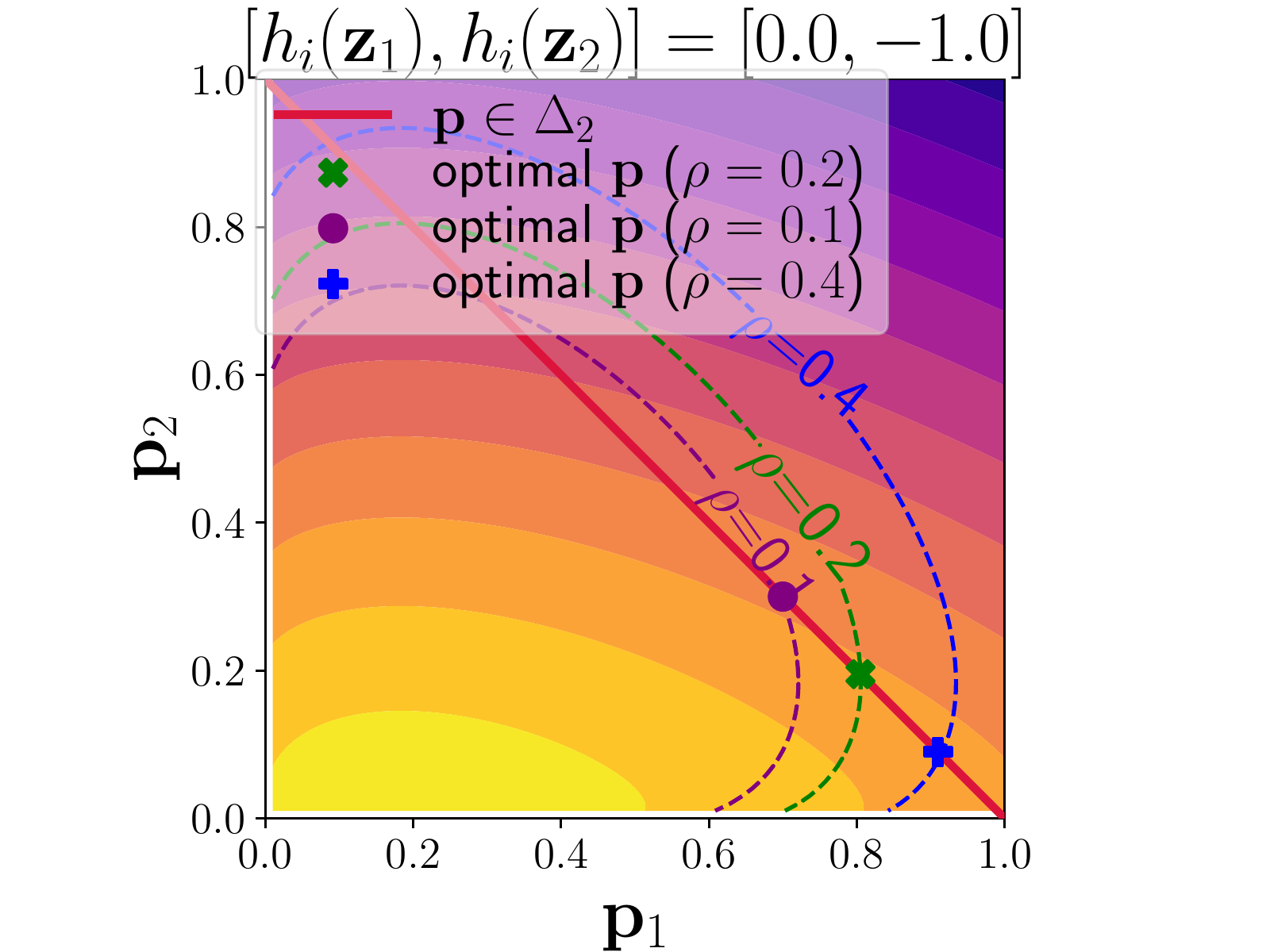}
\end{minipage}
\begin{minipage}[c]{0.235\textwidth}
\centering\includegraphics[width=1\textwidth]{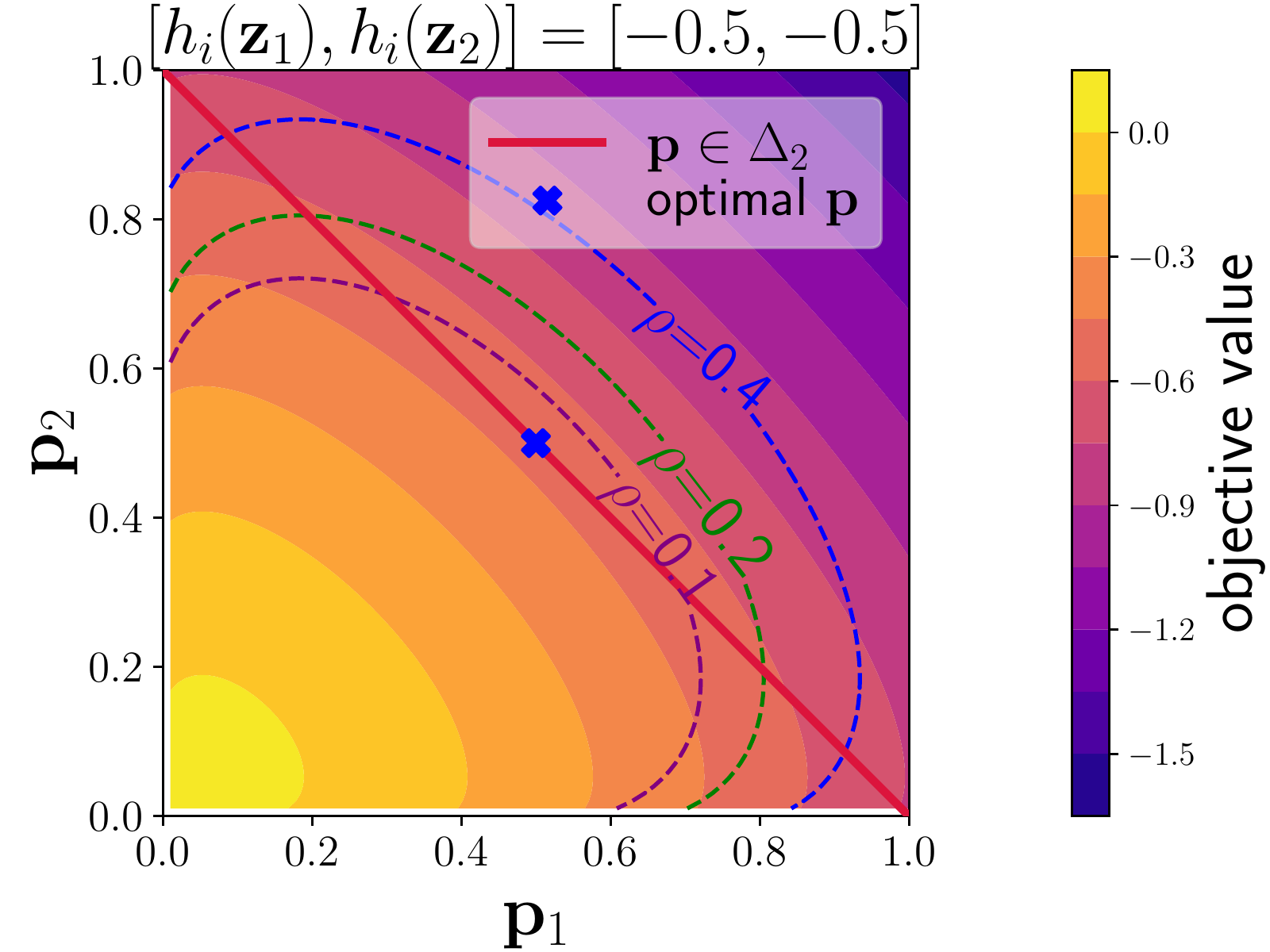}
\end{minipage}
\caption{Contours of $\ell_{\text{RGCL}}(\x_i)\ (|\S_i^-|=2)$ in~(\ref{eq:RCL}) for two $h_i=[h_i(\z_1), h_i(\z_2)]$ vectors: $[0.0,-1.0]$ and $[-0.5,-0.5]$. One can observe that $\ell_{\text{RGCL}}(\x_i)$ is hardness-aware, harder sample ($h_i(\z_1)$ on the left) has larger weight ($\p_1=0.8$). Moreover, $\rho$ affects the degree of hardness-awareness. Larger $\rho$ means higher degree of hardness-awareness.
}
\label{fig:dro_contour}
\end{figure}

Next, we induce an \emph{equivalent} loss with an individualized learnable temperature parameter from~(\ref{eq:RCL}) and show the intuition about the effect of $\tau$ from the DRO view. Since directly optimizing~(\ref{eq:RCL}) is challenging due to maintaining the high-dimensional distributional variable $\p$, we follow~\citet{qi2022stochastic} and adopt the Lagrangian duality theory to convert $\ell_{\text{RGCL}}(\x_i)$ into its dual form (cf. Appendix~\ref{app:minmax_derivation}):

{\setlength\abovedisplayskip{0pt}
\setlength\belowdisplayskip{0pt}
\small
\begin{align}
& \max_{\mathbf{p} \in \Delta_m} \min_{\lambda \geq 0}\!\sum_{\mathbf{z}_j \in \mathcal{S}^-_i}\!p_j h_i(\z_j)\!-\!\tau_0\text{KL}(\mathbf{p},\!\boldsymbol{1}/m)\!-\!\lambda(\text{KL}(\mathbf{p},\!\boldsymbol{1}/m)\!-\!\rho)\notag \\
& {\Leftrightarrow}\!\min_{\lambda \geq 0}\left(\lambda\!+\!\tau_0\right)\!\log \sum_{\mathbf{z}\in\S^-_i} \exp\!\left(h_i(\z) / \lambda\right)\!-\!\left(\lambda+\tau_0\right) \log (m)\!+\!\lambda \rho\notag \\
& \Leftrightarrow\!\min\nolimits_{\tau \geq \tau_0}\tau\log \E_{\mathbf{z}\in\S^-_i} \exp\!\left(h_i(\z) / \tau\right)\!+\!(\tau-\tau_0)\rho, 
\label{eq:dual_form_RGCL}
\end{align}}where we first introduce a Lagrangian multiplier $\lambda$ for the KL constraint, 
and the last equality is due to a variable change $\tau=\lambda+\tau_0$. Notice that the Lagrangian multiplier $\lambda$ for the KL constraint becomes a learnable parameter $\tau$ for $\x_i$. Interestingly, if we fix $\tau$ (i.e., using a fixed KL regularization instead of the KL constraint in~(\ref{eq:RCL})), the above loss will reduce to $\ell_{\text{GCL}}\left(\mathbf{x}_i\right)$ in~(\ref{eq:uni_cl_loss}) up to a constant difference. Hence, RGCL introduces the flexibility to optimize individualized temperatures compared with the GCL. 

Based on the dual form of RGCL in~(\ref{eq:dual_form_RGCL}), we define a robust global contrastive objective (RGCO) for unimodal SSL : 
{\setlength\abovedisplayskip{1pt}
\setlength\belowdisplayskip{1pt}
\begin{align*}
\small
\min_{\w,\bftau\geq\tau_0}\!F(\w,\bftau)\!:=\!\frac{1}{n}\!\sum_{\x_i\in\D}\!\left\{\!\bftau_i\log\!\mathop{\E}_{\z\in\S^-_i}\!\exp\!\left(\frac{h_i(\z)}{\bftau_i}\right)\!+\!\bftau_i\rho\!\right\},
\label{eq:novel_cl_loss}
\end{align*}}where 
$\bftau_i$ is the individualized temperature for $\x_i$. The RGCO for bimodal SSL is defined similarly:
{\setlength\abovedisplayskip{0pt}
\setlength\belowdisplayskip{0pt}
\begin{equation*}
\begin{aligned}
\small
&\min_{\w,\bftau,\bftau'\geq\tau_0}\!F_{\text{B}}(\w,\bftau,\bftau')\!:=\!\frac{1}{n}\sum\nolimits_{(\x_i,\t_i)\in\D'}\Bigg[(\bftau_{i}+\bftau'_{i})\rho  +\\
&\bftau_{i}\log{\E}_{\t\in\T_i^-}\exp\!\left(\frac{h_{\x_i}(\t)}{\bftau_{i}}\right)\!+\!\bftau'_{i}\log{\E}_{\x\in\mathcal I^-_i}\!\exp\!\left(\frac{h_{\t_i}(\x)}{\bftau'_{i}}\!\right)\!\Bigg],
\label{eq:novel_cl_loss_bimodal}
\end{aligned}
\end{equation*}}with individualized temperatures $\bftau_{i}$ and $\bftau_{i}'$ for images and texts, respectively. A small constant can be added inside the log to ensure its smoothness and Lipschitz continuity as in~\cite{yuan2022provable}, which is assumed for analysis.

\subsection{An Intuitive Theoretical Explanation}

To answer why our RGCL can learn suitable temperatures for samples with different semantics intuitively, we compare our iSogCLR (the algorithm for optimizing RCGL) with CLIP, whose learned global $\tau$ is 0.01 on CC3M. Considering the representative cat and bridge images, we extract the features of them and 1000 random samples as their negatives for each method. Then we substitute these features into~(\ref{eq:RCL}) and solve the optimal $\p^*$ (i.e.,~$\p^*_j=\frac{\exp(h_i(\z_j)/\tau)}{\sum_{\z\in\S_i^-}\exp(h_i(\z_j)/\tau)}$, cf.~Appendix~\ref{app:minmax_derivation}) of each image for both methods. We plot the results in Figure~\ref{fig:dro_math_exp}. Combining these results with the formulation~(\ref{eq:dual_form_RGCL}), we have the following explanation.

\begin{figure}
\begin{minipage}[c]{0.235\textwidth}
\centering\includegraphics[width=1.1\textwidth]{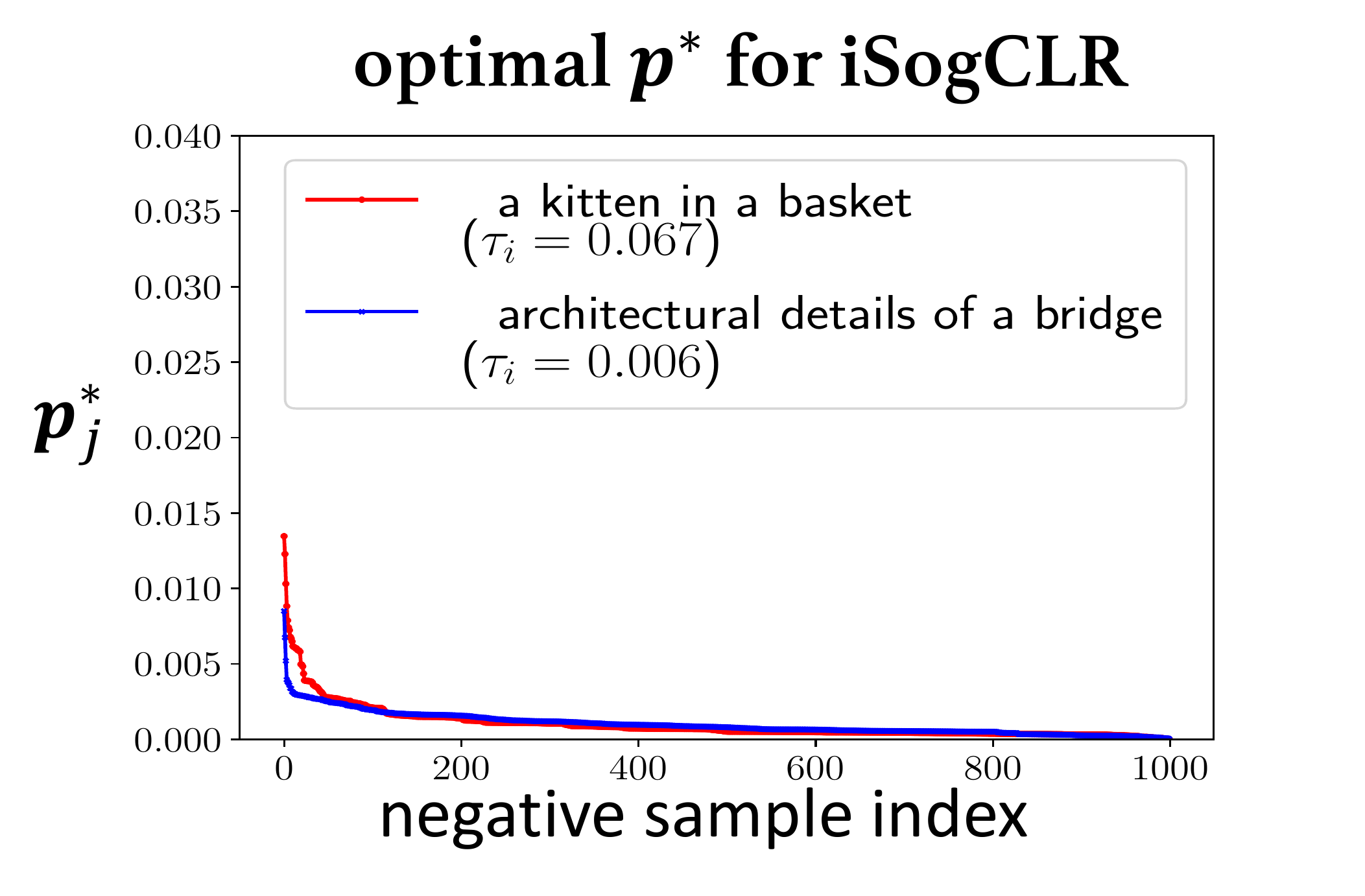}
\end{minipage}
\begin{minipage}[c]{0.235\textwidth}
\centering\includegraphics[width=1.1\textwidth]{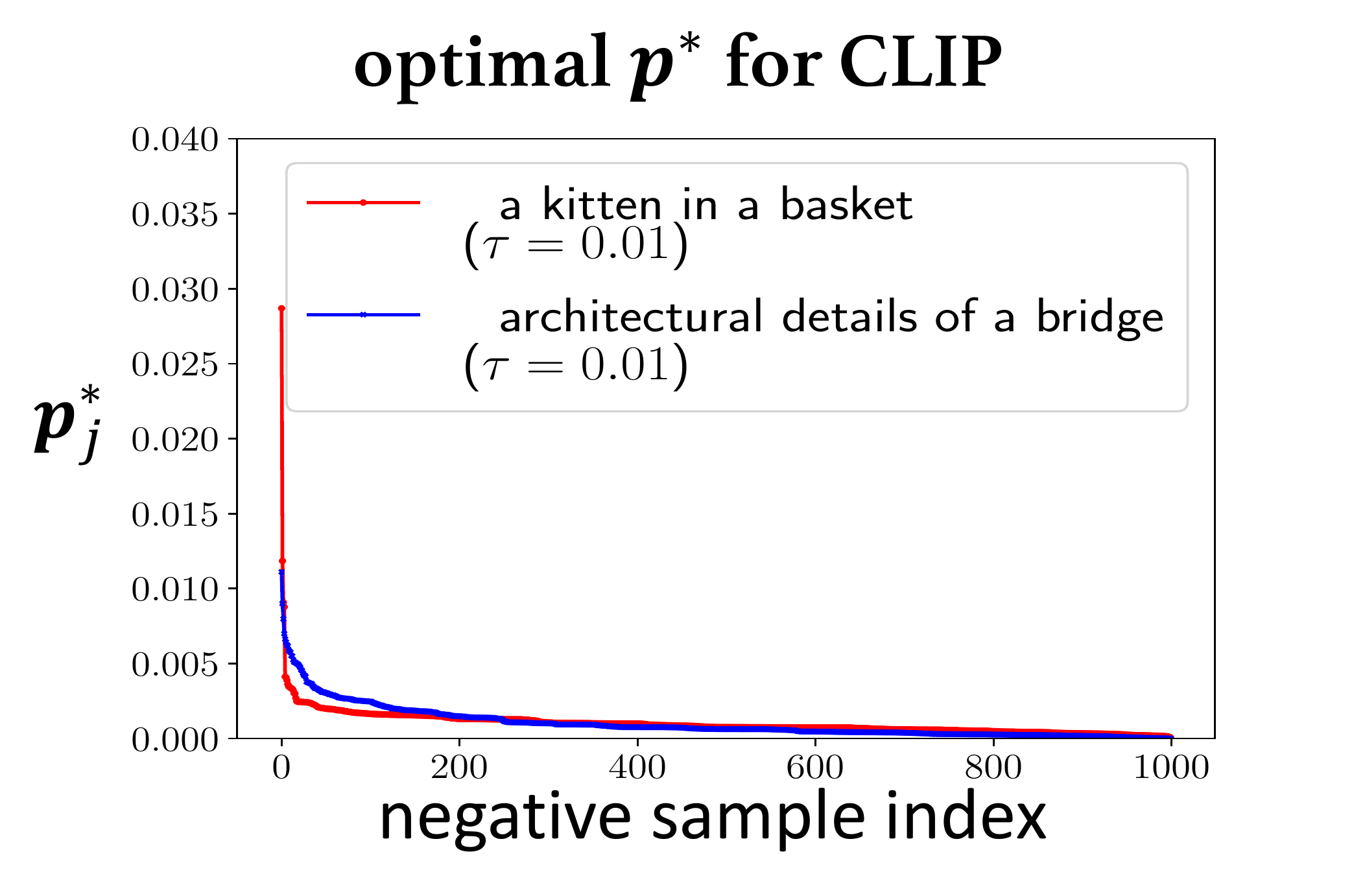}
\end{minipage}
\vspace{-0.2cm}
\caption{For the anchor images of cat and bridge, we select 1000 negative samples and solve~(\ref{eq:RCL}) for the optimal $\p^*$ by using $h_i$ values of iSogCLR  with learned $\bftau_i$ and CLIP with learned $\tau$. 
}
\label{fig:dro_math_exp}
\vspace{-0.0cm}
\end{figure}
{\bf For samples with frequent semantics}, due to the hardness-aware property of RGCL, their optimal $\p^*$ in~(\ref{eq:RCL}) tend to be more non-uniform, and $\text{KL}(\p,\boldsymbol{1}/n)\leq\rho$ is more likely to be violated. Therefore, their Lagrangian multipliers $\lambda$ in~(\ref{eq:dual_form_RGCL}) will be large to ``push'' $\p$ back to be closer to uniform.  Due to $\tau=\lambda+\tau_0$, their temperatures will be large. From Figure~\ref{fig:dro_math_exp}, it is notable that compared with the optimal $\p^*$ of the cat image from CLIP (the red line in the right), that from our RGCL (the red line in the left) is more \emph{uniform}.

{\bf For samples with rare semantics}, their optimal distributional variables $\p$ in~(\ref{eq:RCL}) tend to be more uniform, which makes $\text{KL}(\p,\boldsymbol{1}/n)\leq\rho$ being more likely to be satisfied. At this time, their Lagrangian multipliers $\lambda$ are probably small, and thus their temperatures are small. For example, the learned $\tau$ of the bridge image by iSogCLR is 0.006, which is smaller than the final learned $\tau\!=\!0.01$ in CLIP.


\section{iSogCLR for Stochastic Optimization}
In this section, we design a provable algorithm for optimizing $F(\w,\bftau)$. The algorithm for $F_{\text{B}}(\w,\bftau,\bftau')$ is similar and deferred to Appendix~\ref{app:bimodal_alg}. The new objective functions  $F(\w,\bftau)$ and $F_{\text{B}}(\w,\bftau,\bftau')$ are special cases of X-risks~\cite{DBLP:journals/corr/abs-2206-00439}, making the optimization of them much more challenging than traditional empirical risk minimization. Nonetheless, existing algorithms for deep X-risk optimization are not directly applicable due to the optimization over many temperature variables.  

Inspired by~\citet{yuan2022provable} for solving GCL, 
we cast $F(\w,\bftau)$ as a finite-sum coupled compositional optimization problem~\citep{wang2022finite}:
{\setlength\abovedisplayskip{0pt}
\setlength\belowdisplayskip{0pt}
\begin{equation}
\begin{aligned}
    \hspace*{-0.1in}& \min_{\w,\bftau\in\Omega} F(\w,\bftau):= \frac{1}{n}\sum_{\x_i\sim\D}\underbrace{f_i\left(\bftau_i, g_i(\w,\bftau_i;\S^-_i)\right)}_{F_i(\w,\bftau_i)},
\end{aligned}
\label{eq:objective}
\end{equation}}where $\bftau\in\Omega$ is to accommodate the constraint on $\bftau$ and
{\setlength\abovedisplayskip{2pt}
\setlength\belowdisplayskip{2pt}
\begin{equation*}
\begin{aligned}
& f_{i}(\bftau_i,\cdot)=\bftau_i\log(\cdot)+\bftau_i\rho, \\
&g_i(\w,\bftau_i;\S_i^-)=\E_{\z\in\S^-_i}\exp\left(h_i(\z)/{\bftau_i}\right).
\end{aligned}
\end{equation*}}
The gradients of $F$ w.r.t. $\w$ and $\bftau_i$ can be computed by:
{\setlength\abovedisplayskip{1pt}
\setlength\belowdisplayskip{1pt}
\small
\begin{align}
    \hspace*{-0.15in}&\nabla_{\w} F(\w,\bftau)=\frac{1}{n}\sum_{\x_i\in\D}\nabla_{\w}F_i(\w,\bftau_i)\label{eq:nabla_w_F} \\
    &=\frac{1}{n}\sum_{\x_i\in\D} \nabla_{g_i}f_i(\bftau_i,  g_i(\w,\bftau_i;\S_i^-))\nabla_{\w}g_i(\w,\bftau_i;\S_i^-), \notag\\
    & \nabla_{\bftau_i}\!F(\w,\bftau)\!=\!\frac{1}{n}\!\sum_{\x_j\in\D}\! \nabla_{\bftau_i}\!F_j(\w,\bftau_j)\!\stackrel{(a)}{=}\!\frac{1}{n}\!\nabla_{\bftau_i}\!F_i(\w,\bftau_i)\label{eq:nabla_tau_F}\\
    \hspace*{-0.2in}&=\frac{1}{n} \bigg(\frac{\bftau_i \nabla_{\bftau_i}g_i(\w,\bftau_i;\S_i^-)}{g_i(\w,\bftau_i;\S_i^-)} + \log(g_i(\w,\bftau_i;\S_i^-)) + \rho\bigg)\notag,
\end{align}}where $(a)$ holds because for $j\neq i$, $F_j(\w,\bftau_j)$ does not involve $\bftau_i$. Note that the major cost for computing~(\ref{eq:nabla_w_F}) and~(\ref{eq:nabla_tau_F}) lies at computing $g_i(\w,\bftau_i;\S_i^-)$ and its gradients w.r.t.~$\w$ and $\bftau_i$, involving all samples in $\S_i^-$. At each iteration, we only sample a random mini-batch of $B$ samples $\B=\{\x_1,\x_2,\ldots,\x_B\}$, and compute an unbiased estimator of $g_i(\w,\bftau_i;\S_i^{-})$ for each $\x_i\in\B$ by:
{\setlength\abovedisplayskip{2pt}
\setlength\belowdisplayskip{2pt}
\begin{equation}
g_i(\w,\bftau_i;\B_i)= \frac{1}{|\B_i|}\!\sum\nolimits_{\z\in\B_i}\exp(h_i(\z)/\bftau_i),
\label{eq:g_stochastic_estimator}
\end{equation}}where $\B_i\!=\!\left\{\A(\x),\A'(\x)\!:\!\A,\A^{\prime}\!\in\!\mathcal{P}, \x\!\in\!\B\!\setminus\!\x_i\right\}$ contains the negative samples of $\x_i$ in $\B$. 
However, directly substituting $g_i(\w,\bftau_i;\B_i)$ as the estimator of $g_i(\w,\bftau_i;\S_i^-)$ into~(\ref{eq:nabla_w_F}) and~(\ref{eq:nabla_tau_F}) will yield biased estimators because~(\ref{eq:nabla_w_F}) and~(\ref{eq:nabla_tau_F}) are \emph{non-linear} w.r.t. $g_i(\w,\bftau_i;\S_i^-)$. The optimization error will be large when the batch size is small~\citep{yuan2022provable}.


To control the approximation error and provide a convergence guarantee, we borrow a technique from~\citet{yuan2022provable} by using a \emph{moving average estimator} to keep track of $g_i(\w,\bftau_i;\S_i^-)$ for each $\x_i\in\D$. To this end, we maintain a scalar $\s_i$ for each $\x_i$ and update it at the $t$-th iteration by:
{\setlength\abovedisplayskip{2pt}
\setlength\belowdisplayskip{2pt}
\begin{equation}
    \s_i^{t+1}=(1-\beta_0)\s_i^t+\beta_0 g_i(\w_t,\bftau_i^t; \B_i),
\label{eq:update_s}
\end{equation}}where $\beta_0\in(0,1)$. Intuitively, when $t$ increases, $\w_{t-1}$ and $\bftau^{t-1}$ are getting close to $\w_t$ and $\bftau^t$, hence \textbf{the previous value of $\s_i^t$ is useful for estimating $g_i$}. 
With these stochastic estimators, we compute the gradients of~(\ref{eq:objective}) in terms of $\w_t$ and $\bftau_i^t$ with \emph{controllable} approximation error by:
{\setlength\abovedisplayskip{0pt}
\setlength\belowdisplayskip{0pt}
\begin{align}
\hspace{-10mm}
\hspace*{-0.11in}&G(\bftau_i^t) = \frac{1}{n} \left[\frac{\bftau_i^t}{\s_i^t}\nabla_{\bftau_i}g_i(\w,\bftau_i;\B_i) + \log(\s_i^t) + \rho\right]\label{eq:grad_tau},\\
\hspace*{-0.11in}&G(\w_t) = \frac{1}{B}\sum_{\x_i\in\B}\frac{\bftau_i^t}{\s_i^t}\nabla_{\w}g_i(\w,\bftau_i;\B_i).\label{eq:grad_w}
\end{align}}
The complete procedure is presented 
in Algorithm~\ref{algo:sogclr_dro}, named \textbf{iSogCLR} with \textbf{i} standing for individualization of temperatures. 
In step 1, we initialize all $\bftau_i$ to $\tau_{\text{init}}$. We implement the momentum update for $\bftau_i^{t+1}$ and $\w_{t+1}$ in Step 8, 9 and Step 12, 13, respectively, where $\beta_1\in(0,1)$ is the momentum parameter. The momentum-style update can be replaced by an Adam-style update using adaptive step sizes and the same convergence rate can be established~\citep{guo2021stochastic}. 

In terms of the additional memory cost, while it scales with the number of samples, it typically is not a significant concern in practical applications. Firstly, the additional memory cost is still small compared with the number of model parameters. For example, the additional memory cost for 1 million samples is $2\times\frac{10^6\times 4}{1024^2}=7.63$MB. Secondly, the GPU memory usage can be optimized by storing variables $\s_i$, $\u_i$ and $\bftau_i$ in the CPU memory. To further minimize the impact on training time, one can employ an \emph{asynchronous} strategy to transfer data. Specifically, before the $t$-th iteration, we can prefetch $\s_i^t$, $\u_i^t$ and $\bftau_i^t$ from the CPU and transfer them to the GPU. After forward propagation, we conduct back-propagation and asynchronously copy the updated $\s_i^{t+1}$, $\u_i^{t+1}$ and $\bftau_i^{t+1}$ back to the CPU memory and fetch a new batch of them from CPU. By utilizing high-bandwidth CPU-GPU interconnects, such as PCIe4 or NVLink, the time required to transfer these variables can effectively overlap with the time of back-propagation. This approach facilitates fast training and reduces GPU memory consumption.

\begin{algorithm}[t]
\caption{iSogCLR}\label{algo:sogclr_dro}
\begin{algorithmic}[1]
\REQUIRE $\beta_0, \beta_1, \eta$
\STATE Initialize $\w_1,\s^1, \u^1, \v_1$, $\bftau^1=\boldsymbol{\tau_{\text{init}}}$
\FOR{$t=1,2,\dots,T$}
\STATE Draw a batch of $B$ samples denoted by $\B\subset\D$
\FOR{$\x_i\in\B$}
\STATE Compute $g_i(\w_t,\bftau_i^t;\B_i)$ according to~(\ref{eq:g_stochastic_estimator})
\STATE Update $\s_i^{t+1}$ according to~(\ref{eq:update_s})
\STATE Compute $G(\bftau_i^t)$ according to~(\ref{eq:grad_tau})
\STATE Update $\u_i^{t+1}=(1-\beta_1)\u_i^t+\beta_1 G(\bftau_i^t)$
\STATE Update $\bftau_i^{t+1}=\Pi_{\Omega}\left[\bftau_i^t - \eta \u_i^{t+1} \right]$
\ENDFOR
\STATE Compute gradient estimator $G(\w_t)$ according to~(\ref{eq:grad_w})
\STATE Compute $\v_{t+1}=(1-\beta_1)\v_t + \beta_1 G(\w_t)$
\STATE Update $\w_{t+1}=\w_t - \eta \v_{t+1}$ \text{(or Adam-style)}
\ENDFOR
\end{algorithmic}
\end{algorithm}

We highlight the differences between RGCO/iSogCLR and GCO/SogCLR~\cite{yuan2022provable} and our contributions of analysis: (i) RGCO has additional $n$ temperature variables $\bftau_i$; nevertheless we prove the same iteration complexity as SogCLR. (ii) the constraint $\bftau_i\geq \tau_0$ makes the analysis more complicated. In particular, to show $F(\w,\bftau)$ is smooth in terms of $(\w, \bftau)$, we follow~\citet{qi2022stochastic} to derive an upper bound $\tau_{\max}$ for the optimal $\bftau_*$ and use the constraint set $\Omega=\{\tau_0\leq \bf\tau\leq \tau_{\max}\}$ in the analysis. (iii) Due to the constraint on $\bftau$, we employ a different notion of stationary point using the regular subgradient (cf.~Appendix~\ref{sec:convergence_analysis}) of the non-smooth extended objective $\bar F(\w, \bftau)=F(\w, \bftau) + \delta_\Omega(\bftau)$, i.e., $\hat{\partial}\bar{F}(\w,\bftau)$, which also complicates the analysis. Finally, the convergence guarantee of iSogCLR is:
\vspace*{-0.02in}
\begin{thm}
    Under appropriate conditions and settings of parameters $\beta_0,\beta_1\!=\!\O(B'\epsilon^2)$, $\eta\!=\!\O\!\left(\!\frac{BB'\epsilon^2}{n}\!\right)$, where $B\!=\!|\B|,B'\!=\!|\B_i|$,  after $T\!=\!\O\!\left(\!\frac{n}{BB'\epsilon^4}\!\right)$ iterations Algorithm~\ref{algo:sogclr_dro} finds an $\epsilon$-stationary solution of the problem, i.e., $\E[\text{dist}(0, \hat\partial \bar F(\w_t,\bftau_t))^2]\!\leq\!\epsilon^2$ for a random $t\in\{1,\ldots,T\}$. 
\label{thm:short_stat}
\end{thm}
\vspace*{-0.3in}{\noindent\bf Remark:} The theorem indicates that iSogCLR has the same $\O\left(\frac{1}{\epsilon^4}\right)$ complexity as SogCLR~\citep{yuan2022provable}. We refer the interested readers to Appendix~\ref{sec:convergence_analysis} for the proof, where we also exhibit the conditions similar to~\citep{yuan2022provable}.

\section{Experiments}

\begin{table*}[t]
\vspace{-4mm}
\caption{Linear evaluation results with 400 pretraining epochs on six unimodal image datasets. We report the average top-1 accuracies (\%) and standard deviation over 3 runs with different random seeds. Full results are provided in Table~\ref{tab:unimodal_datasets_results_full_balance} and~\ref{tab:unimodal_datasets_results_full_imbalance} in Appendix~\ref{app:add-exp-res}.}
\vspace{-4mm}
\label{tab:unimodal_datasets_results_main}
\vskip 0.0in
\begin{center}
\begin{small}
\begin{sc}
\renewcommand{\arraystretch}{0.8}
\begin{tabular}{p{2.2cm}p{2cm}p{2cm}p{2cm}|p{2cm}p{2.2cm}p{2cm}}
\toprule
Method & CIFAR10 & CIFAR100 & ImageNet100 & CIFAR10-LT & CIFAR100-LT & iNaturalist \\
\midrule
SimCLR & 88.74$\pm$0.18 & 62.34$\pm$0.09 & 79.96$\pm$0.20  & 77.09$\pm$0.13 & 49.33$\pm$0.12 & 91.52$\pm$0.17 \\
Barlow Twins & 87.39$\pm$0.14 & 62.28$\pm$0.13 & 79.16$\pm$0.13 & 75.94$\pm$0.08 & 48.39$\pm$0.14 & 91.89$\pm$0.21 \\
FlatCLR & 88.61$\pm$0.10 & 63.27$\pm$0.07 & 80.24$\pm$0.16 & 77.96$\pm$0.12 & 52.61$\pm$0.06 & 92.54$\pm$0.09 \\
Spectral CL & 88.77$\pm$0.09 & 63.06$\pm$0.18 & 80.48$\pm$0.08 & 76.38$\pm$0.21 & 51.86$\pm$0.16 & 92.13$\pm$0.16 \\
SogCLR & 88.93$\pm$0.11 & 63.14$\pm$0.12 & 80.54$\pm$0.14 & 77.70$\pm$0.07 & 52.35$\pm$0.08 & 92.60$\pm$0.08 \\
VICReg & 88.96$\pm$0.16  & 62.44$\pm$0.13 & 80.16$\pm$0.22 & 75.05$\pm$0.09 & 48.43$\pm$0.13 & 93.03$\pm$0.14 \\
SimCo & 88.86$\pm$0.12 & 62.67$\pm$0.06 & 79.73$\pm$0.17 & 77.71$\pm$0.13 & 51.06$\pm$0.09 & 92.10$\pm$0.12 \\
iSogCLR & \textbf{89.24}$\pm$0.15 & \textbf{63.82}$\pm$0.14 & \textbf{81.14}$\pm$0.19 & \textbf{78.37}$\pm$0.16 & \textbf{53.06}$\pm$0.12 & \textbf{93.08}$\pm$0.19 \\
\bottomrule
\end{tabular}
\end{sc}
\end{small}
\end{center}
\vskip -0.2in
\end{table*}

In this section, we conduct experiments on unimodal and bimodal datasets and observe that our algorithm outperforms prior strong baselines. Moreover, in-depth analyses show that the samples with different semantics are indeed assigned with suitable temperatures. We also perform ablation studies to better understand the behaviors of iSogCLR. 
The code to reproduce the results in this paper is available at \url{https://github.com/zhqiu/contrastive-learning-iSogCLR/}.

In unimodal setting, We compare our iSogCLR with five CL methods: {\bf SimCLR}~\citep{chen2020simple}, {\bf FlatCLR}~\citep{chen2021simpler}, {\bf SimCo}~\citep{zhang2022dual}, {\bf Spectral CL}~\citep{haochen2021provable}, {\bf SogCLR}~\citep{yuan2022provable}, and two non-contrastive methods: {\bf Barlow Twins}~\citep{zbontar2021barlow} and {\bf VICReg}~\citep{bardes2021vicreg}. In bimodal setting, we compare with {\bf CLIP}~\citep{radford2021learning}, {\bf CyCLIP}~\citep{goel2022cyclip}, and {\bf SogCLR}. 

For fair comparison, we set the hyper-parameters of all methods using grid search. SimCLR, FlatCLR, and SogCLR contain $\tau$, which is tuned in a range of $\{0.1,0.3,0.5,0.7\}$. For other methods, 
we fine-tune their hyper-parameters around the recommended values in their  papers. Following~\citet{radford2021learning,goel2022cyclip}, $\tau$ is directly optimized in CLIP and CyCLIP. 
The detailed implementation and data information are in Appendix~\ref{app:exp-impl} and~\ref{app:exp-data}, respectively.

\begin{table*}[t]
\vspace{-4mm}
\caption{Results on two bimodal downstream tasks. For image-text retrieval on Flickr30K and MSCOCO, we compute IR@$1$ and TR@$1$ for the Recall@$1$ on image-retrieval (IR) and text-retrieval (TR). For classification tasks, we compute top-$1$ accuracy (\%). We report the average of scores and standard deviation over 3 runs with different random seeds. Full results are in Table~\ref{tab:bimodal_zs_retrieval_full_flickr},~\ref{tab:bimodal_zs_retrieval_full_coco}, and~\ref{tab:bimodal_zs_classification_full} in Appendix~\ref{app:add-exp-res}.}
\vspace{-4mm}
\label{tab:bimodal_results_main}
\vskip 0.2in
\begin{center}
\begin{small}
\begin{sc}
\renewcommand{\arraystretch}{0.8}
\begin{tabular}{p{1.4cm}p{1.7cm}p{1.7cm}p{1.7cm}p{1.7cm}|p{1.7cm}p{1.7cm}p{1.7cm}}
\toprule
\multirow{1}{*}{\thead{Method}} &
\multicolumn{2}{c}{\thead{Flickr30K Retrieval}} &
\multicolumn{2}{c}{\thead{MSCOCO Retrieval}} &
\multicolumn{3}{c}{\thead{Zero-shot Classification top-1 Acc}} \\
\cmidrule(lr){2-3}
\cmidrule(lr){4-5}
\cmidrule(lr){6-8}
& IR@1 & TR@1 & IR@1 & TR@1 & CIFAR10 & CIFAR100 & ImageNet1K  \\
\midrule
CLIP & 40.98$\pm$0.22  & 50.90$\pm$0.17  & 21.32$\pm$0.12  & 26.98$\pm$0.21 & 60.63$\pm$0.19 & 30.70$\pm$0.11 & 36.27$\pm$0.17 \\
CyCLIP & 42.46$\pm$0.13 & 51.70$\pm$0.23 & 21.58$\pm$0.19 & 26.18$\pm$0.24 & 57.19$\pm$0.20 & 33.11$\pm$0.14 & 36.75$\pm$0.21 \\
SogCLR & 43.32$\pm$0.18 & 57.18$\pm$0.20 & 22.43$\pm$0.13 & 30.08$\pm$0.22 & \textbf{61.09}$\pm$0.24 & 33.26$\pm$0.12 & 37.46$\pm$0.19 \\
iSogCLR & \textbf{44.36}$\pm$0.12 & \textbf{60.20}$\pm$0.26 & \textbf{23.27}$\pm$0.18 & \textbf{32.72}$\pm$0.13 & 58.91$\pm$0.15 & \textbf{33.81}$\pm$0.18 & \textbf{40.72}$\pm$0.23 \\
\bottomrule
\end{tabular}
\end{sc}
\end{small}
\end{center}
\vskip -0.2in
\end{table*}

\subsection{Unimodal Experiments}

{\bf Data.} We consider three balanced datasets: CIFAR10, CIFAR100, ImageNet100~\citep{wu2019large}, and three imbalanced datasets: CIFAR10-LT, CIFAR100-LT, iNaturalist2018~\citep{iNat18}. 
ImageNet100 is a subset with 100 classes from ImageNet1K~\citep{russakovsky2015imagenet}. CIFAR10-LT and CIFAR100-LT are created following the Long-Tailed (LT) imbalance setting~\citep{cui2019class} and widely used~\citep{cao2019learning,cui2019class,qi2020attentional}. 
The iNaturalist dataset is a large-scale dataset with dramatically different number of images per category. We use its official training and validation splits.

{\bf Setup.} The backbone network, initial learning rate and batch size are set to ResNet-18, 0.8, and 128 for CIFAR datasets. While for ImageNet100 and iNaturalist2018, they are set to ResNet-50, 1.2, and 256, respectively. The projection head has three linear layers, each with 8192 output units. The first two layers of the projector are followed by a BN layer and rectified linear units. 
We employ LARS optimizer~\citep{you2017large} (with a momentum of 0.9 and weight decay of 1e-4) and cosine learning rate schedule. We also use learning rate warm-up for 10 epochs, i.e., learning rate is gradually increased to the maximum value. We resize input images to 224$\times$224 and follow the same image augmentation strategies as in SimCLR~\citep{chen2020simple} including random crop, color distortion, and Gaussian blur. 
For linear evaluation, we train the last classification layer using SGD with Nesterov momentum with a batch size of 256 for 100 epochs. The initial learning rate is set to 30.0 and decayed by 0.2 at 40, 60, and 80 epochs. We tune $\beta_0$ and $\rho$ in our algorithm from $\{0.7,0.8,0.9\}$ and $\{0.1,0.2,0.3,0.4\}$, respectively. $\tau_{\text{init}}$ and $\tau_0$ are set to 0.7 and 0.05 by default.

{\bf Results.} We present partial results in Table~\ref{tab:unimodal_datasets_results_main} and full results in Table~\ref{tab:unimodal_datasets_results_full_balance},~\ref{tab:unimodal_datasets_results_full_imbalance} in Appendix~\ref{app:add-exp-res}. First, comparing iSogCLR, SogCLR and SimCLR, we observe that (i) SogCLR is generally better than SimCLR, showing the advantage of optimizing a GCL under limited mini-batch sizes; and (ii) iSogCLR outperforms SogCLR in all cases, confirming the effectiveness of individualized temperatures. In Figure~\ref{fig:in_depth_cifar10}, we visualize the learned embeddings from these three methods on CIFAR10, where each color represents a class. Note that the class boundaries of iSogCLR are more clear than that of others, indicating that iSogCLR indeed improves feature qualities. 
We also observe that iSogCLR outperforms prior strong baselines, e.g., VICReg, Spectral CL. Besides, iSogCLR achieves larger improvements on imbalanced data, e.g., has relative improvements of 2.37\% and 7.56\%  over SimCLR on CIFAR100 and CIFAR100-LT, respectively. 

\subsection{Bimodal Experiments}

{\bf Data.} We adopt Conceptual Captions 3M (CC3M)~\citep{sharma2018conceptual} dataset, which is widely used in vision-and-language pretraining~\citep{li2021supervision,mu2022slip,goel2022cyclip}. Because some links of the images in CC3M have expired, the number of pairs we downloaded is about 2.85M, which is smaller than 3.3M in original paper.
During evaluation, we use two common bimodal datasets: Flickr30K~\citep{plummer2015flickr30k}, MSCOCO~\citep{lin2014microsoft}, obtained from the well-known Karpathy split~\citep{karpathy2015deep}, 
and three standard image datasets: CIFAR10, CIFAR100, and ImageNet1K.

{\bf Setup.} Following recent studies on bimodal SSL~\citep{li2021align,dou2022coarse}, we adopt ResNet-50 and DistilBert~\citep{sanh2019distilbert} as the image and text encoder, which are initialized with weights from unimodal pretraining. 
Specifically, we use the ResNet-50 model pretrained on ImageNet from timm library~\citep{rw2019timm}. The DistilBert model comes from huggingface library~\citep{wolf-etal-2020-transformers}, which is pretrained on BookCorpus~\citep{Zhu_2015_ICCV} and English Wikipedia. The output embedding of each encoder is then transformed to a lower-dimensional (256-d) representation by a linear layer and normalized for computing contrastive loss. We use a batch size of 512 for 30 epochs pre-training, where the image resolution is 256$\times$256. We employ Adam-W optimizer~\citep{loshchilov2017decoupled} with cosine learning rate decay. The learning rate is warmed-up to 2e-4 in the first 1000 iterations and decayed to 1e-6 by a cosine decay scheduler. We employ Adam-W optimizer~\citep{loshchilov2017decoupled} with the weight decay of 0.02. 
We tune $\beta_0$ and $\rho$ from $\{0.7,0.8,0.9\}$ and $\{5.8,6.0,6.2,6.4\}$, respectively. $\tau_{\text{init}}$ and $\tau_0$ are set to 0.01 and 0.005 by default. We evaluate models on two downstream tasks: cross-modal retrieval and image classification in zero-shot setting, following the widely-used evaluation protocol~\citep{radford2021learning,goel2022cyclip}.  

\begin{figure}[t]
\begin{minipage}[c]{0.155\textwidth}
\centering\includegraphics[width=1\textwidth]{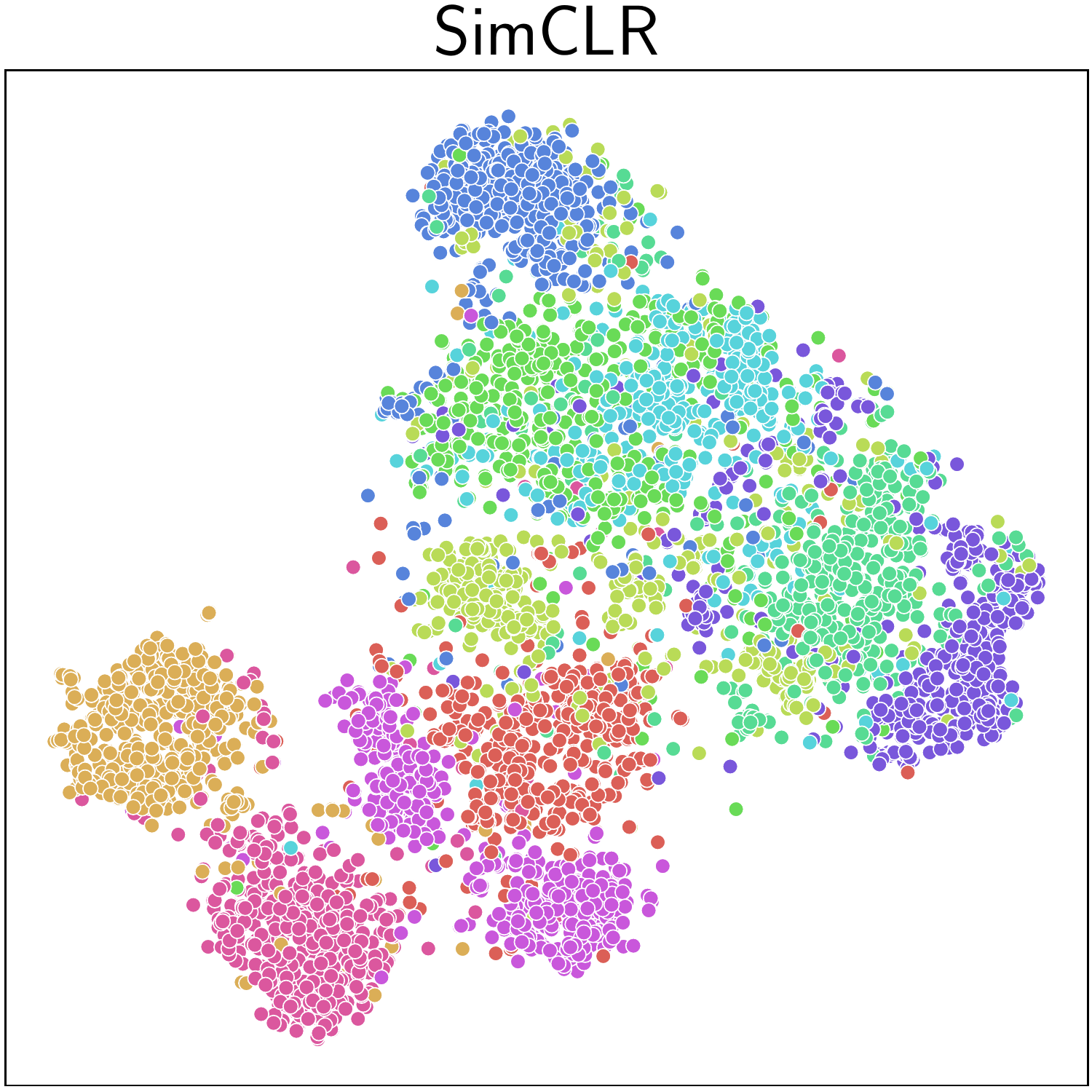}
\end{minipage}
\begin{minipage}[c]{0.155\textwidth}
\centering\includegraphics[width=1\textwidth]{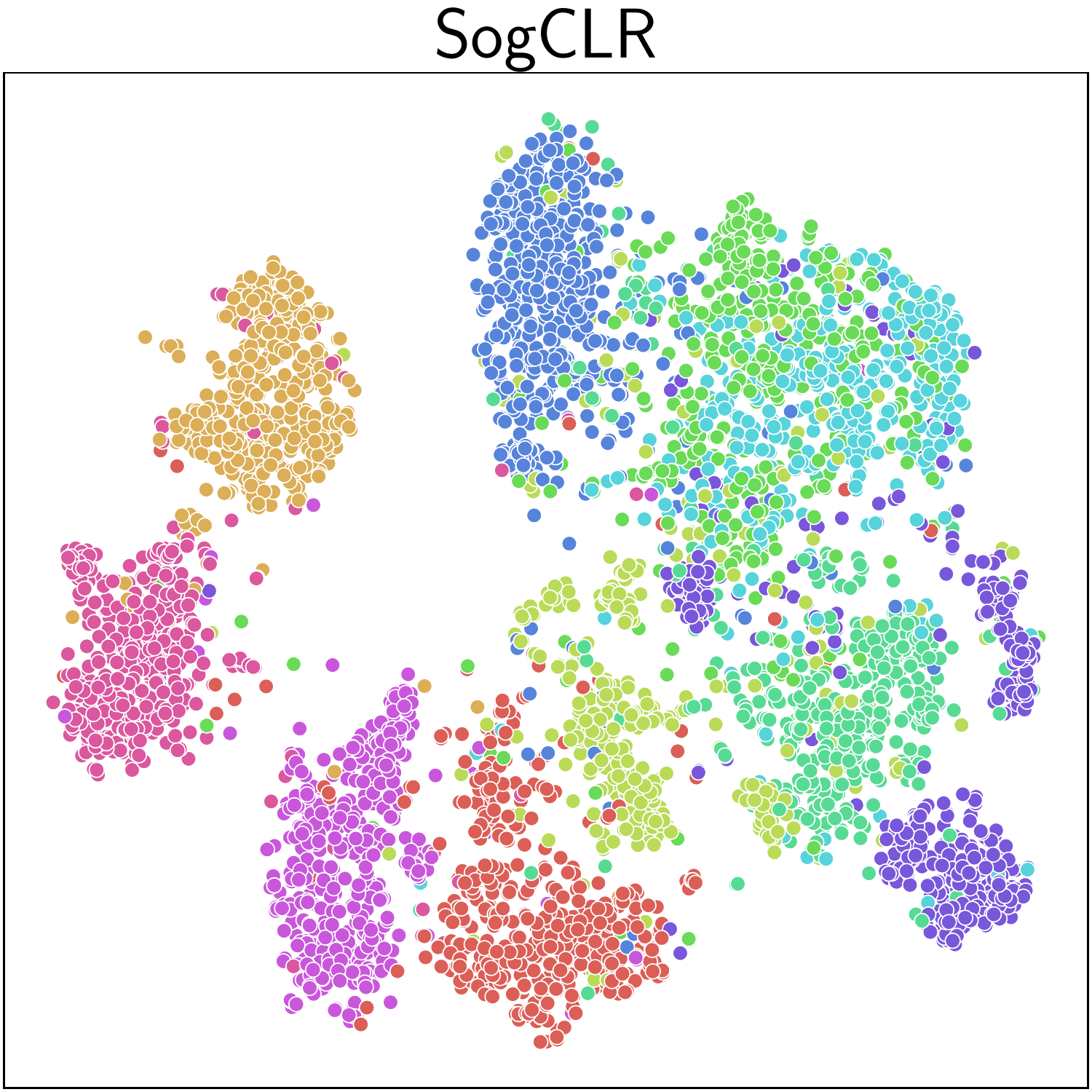}
\end{minipage}
\begin{minipage}[c]{0.155\textwidth}
\centering\includegraphics[width=1\textwidth]{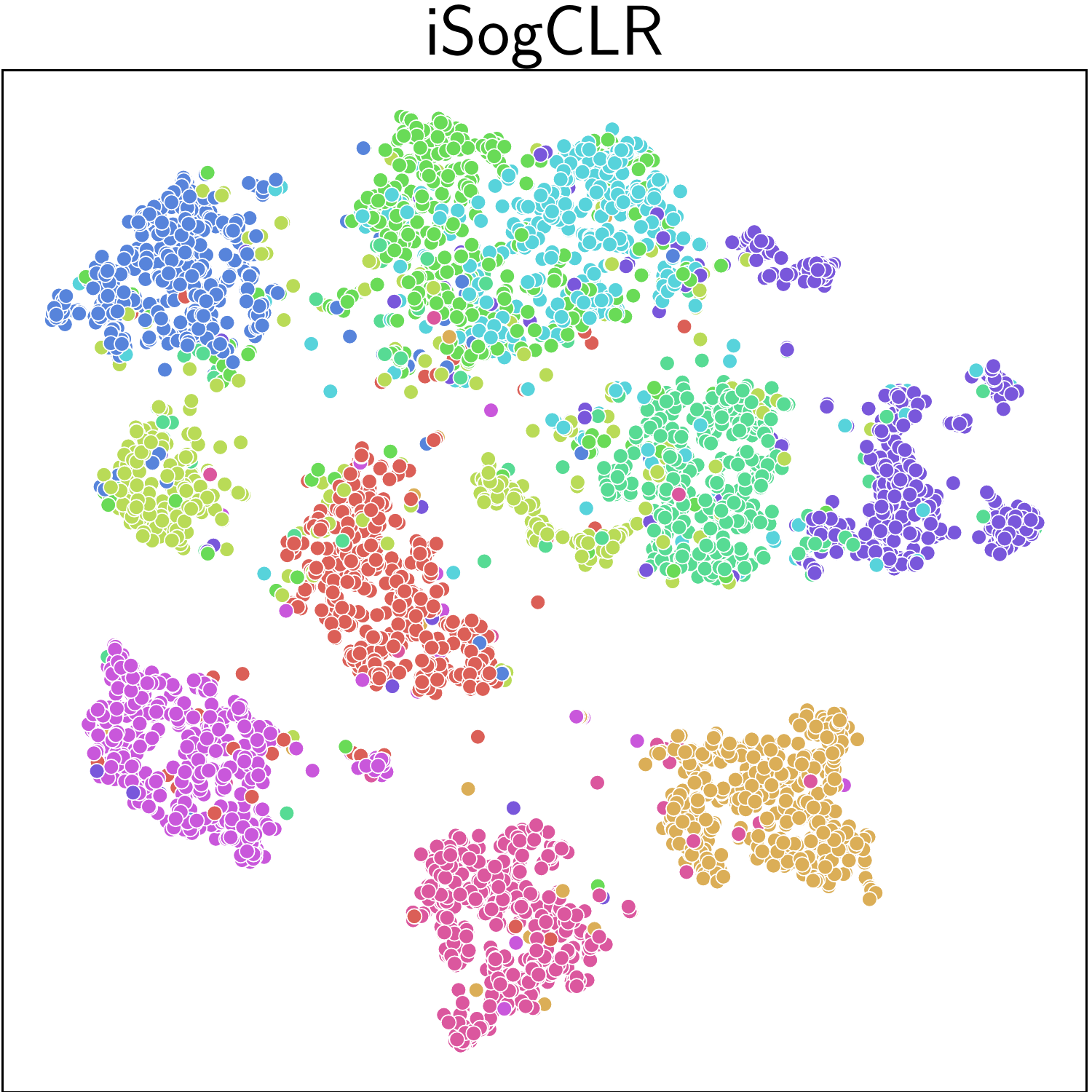}
\end{minipage}
\vspace{-0.2cm}
\caption{The arrangement of features (projected using t-SNE) for CIFAR10 samples learned by SimCLR, SogCLR and iSogCLR.}
\label{fig:in_depth_cifar10}
\vspace{0.05cm}

\begin{minipage}[c]{0.245\textwidth}
\centering\includegraphics[width=1\textwidth]{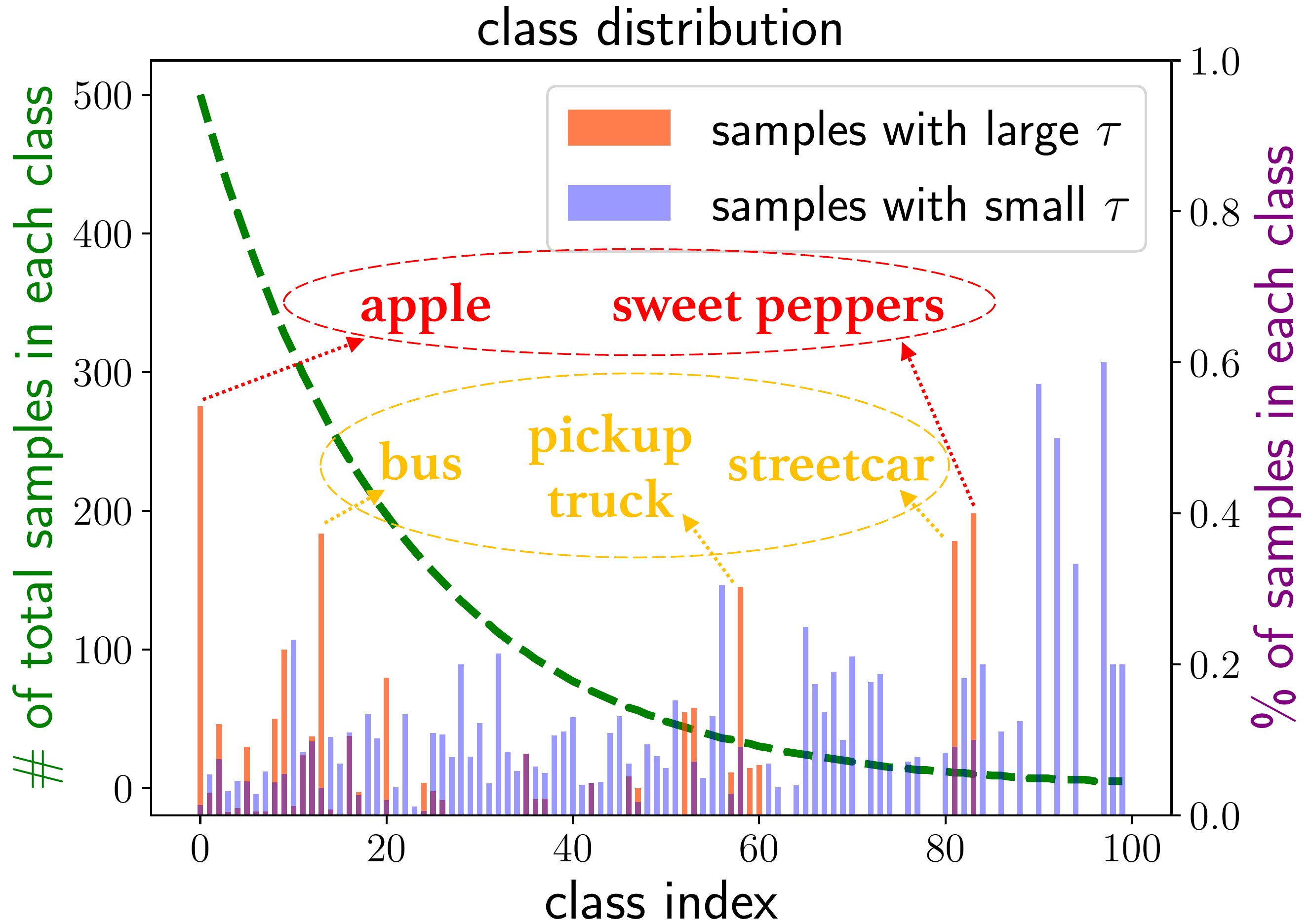}
\end{minipage}
\hspace*{0.05in}
\begin{minipage}[c]{0.22\textwidth}
\centering\includegraphics[width=1\textwidth]{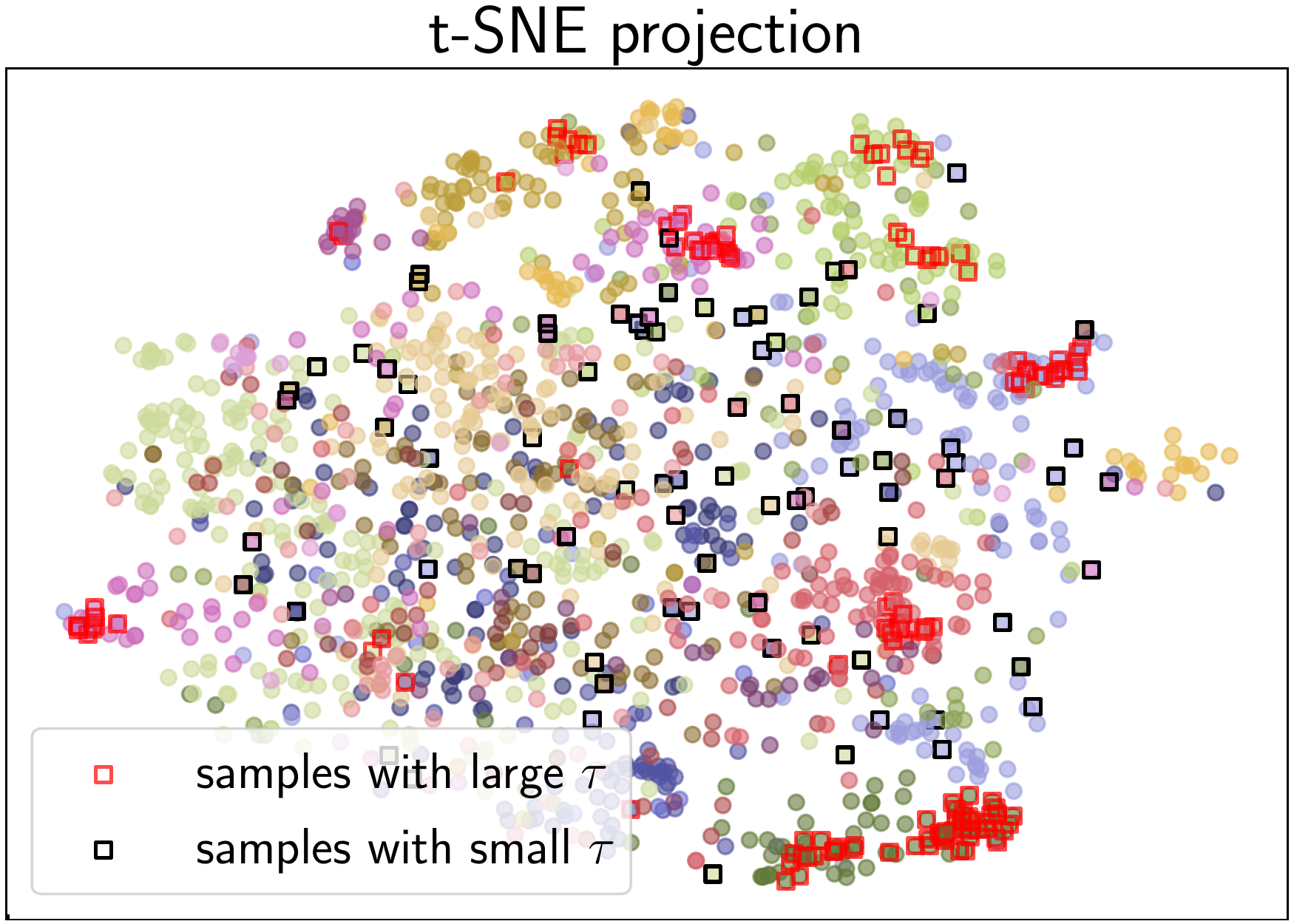}
\end{minipage}
\vspace{-0.2cm}
\caption{The class distributions and t-SNE projection for samples with large and small $\tau$ values in CIFAR100-LT. 
Left: The green dashed line and left axis denote the number of samples in each class, the red/blue bars and right axis denote the proportions of samples with large/small $\tau$ values in each class.
Right: Each color represents a \emph{superclass} in CIFAR100-LT.}
\label{fig:in_depth_cifar100_lt}
\vspace{-0.2cm}
\end{figure}

\begin{figure*}[t]
\begin{center}
\begin{minipage}[c]{0.48\textwidth}
\centering\includegraphics[width=1\textwidth]{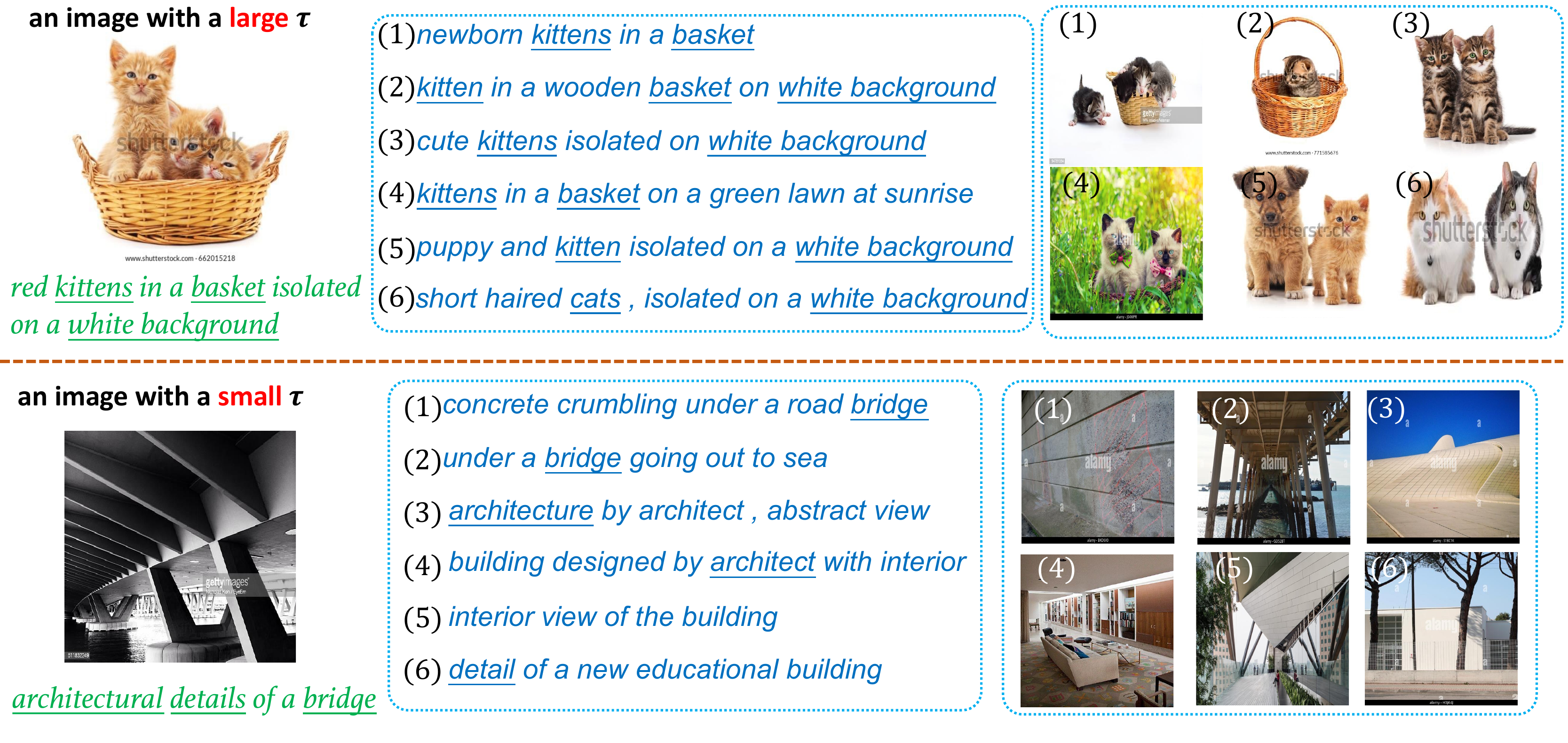}
\end{minipage}
\begin{minipage}[c]{0.48\textwidth}
\centering\includegraphics[width=1\textwidth]{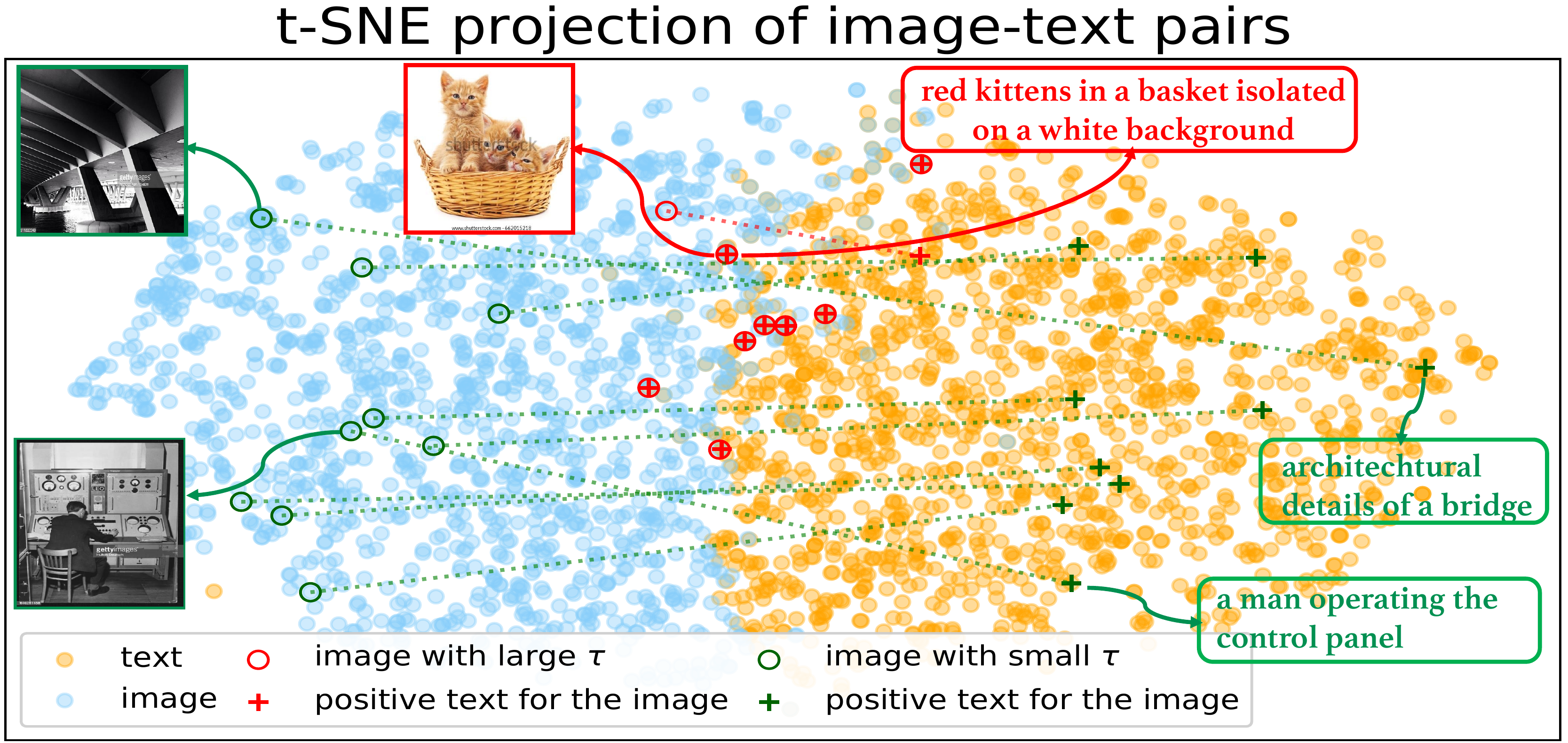}
\end{minipage}
\vspace{-0.3cm}
\caption{In-depth analyses on CC3M. Left: the contents of several hard negative image-text pairs of the cat and bridge images. Right: the tSNE of learned representations of sampled image-text pairs, with large and small temperatures marked by red and green, respectively.}
\label{fig:in_depth_cc3m}
\end{center}
\vspace{-0.7cm}
\end{figure*}

\begin{figure}[t]
\begin{minipage}[c]{0.235\textwidth}
\centering\includegraphics[width=1\textwidth]{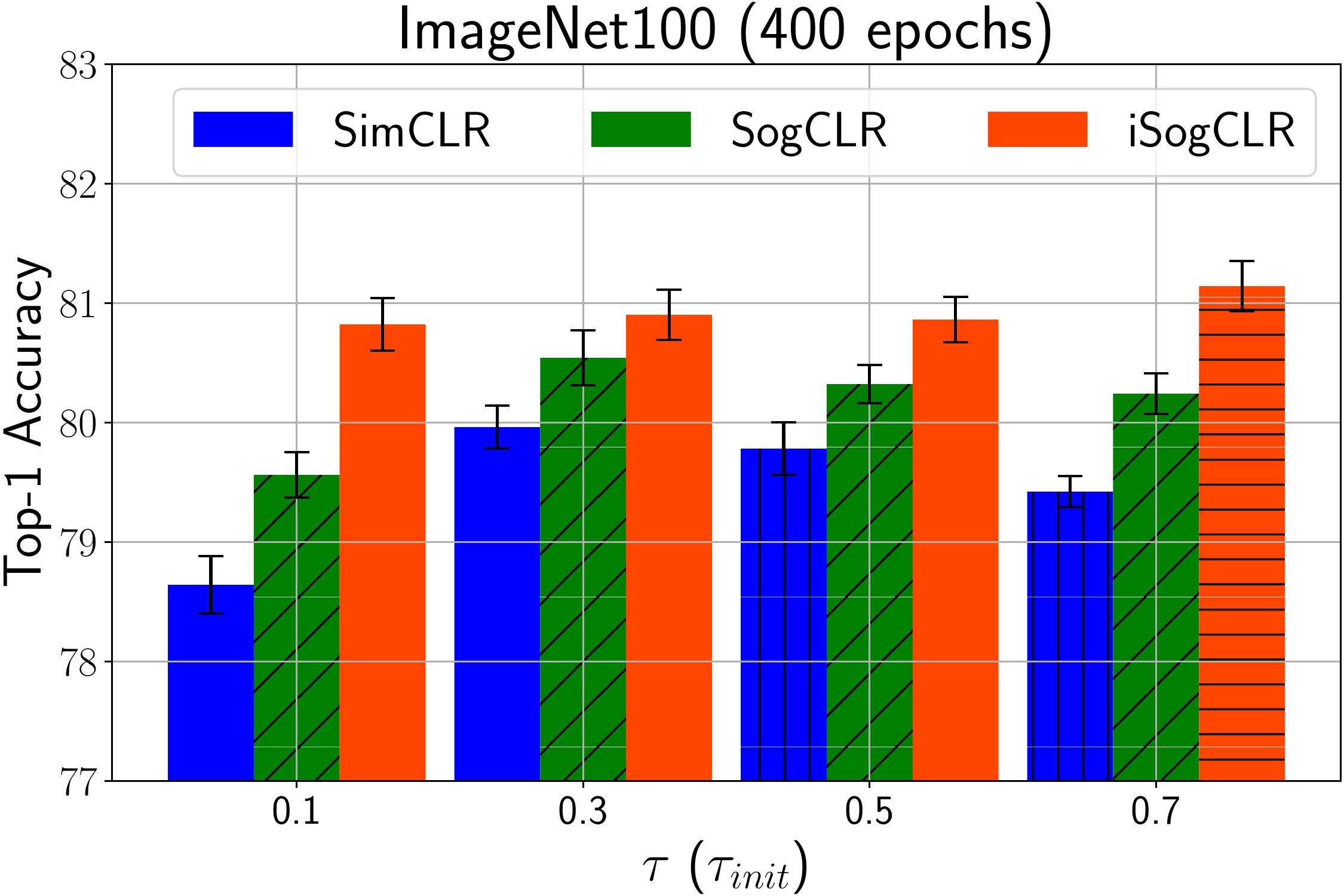}
\end{minipage}
\begin{minipage}[c]{0.235\textwidth}
\centering\includegraphics[width=1\textwidth]{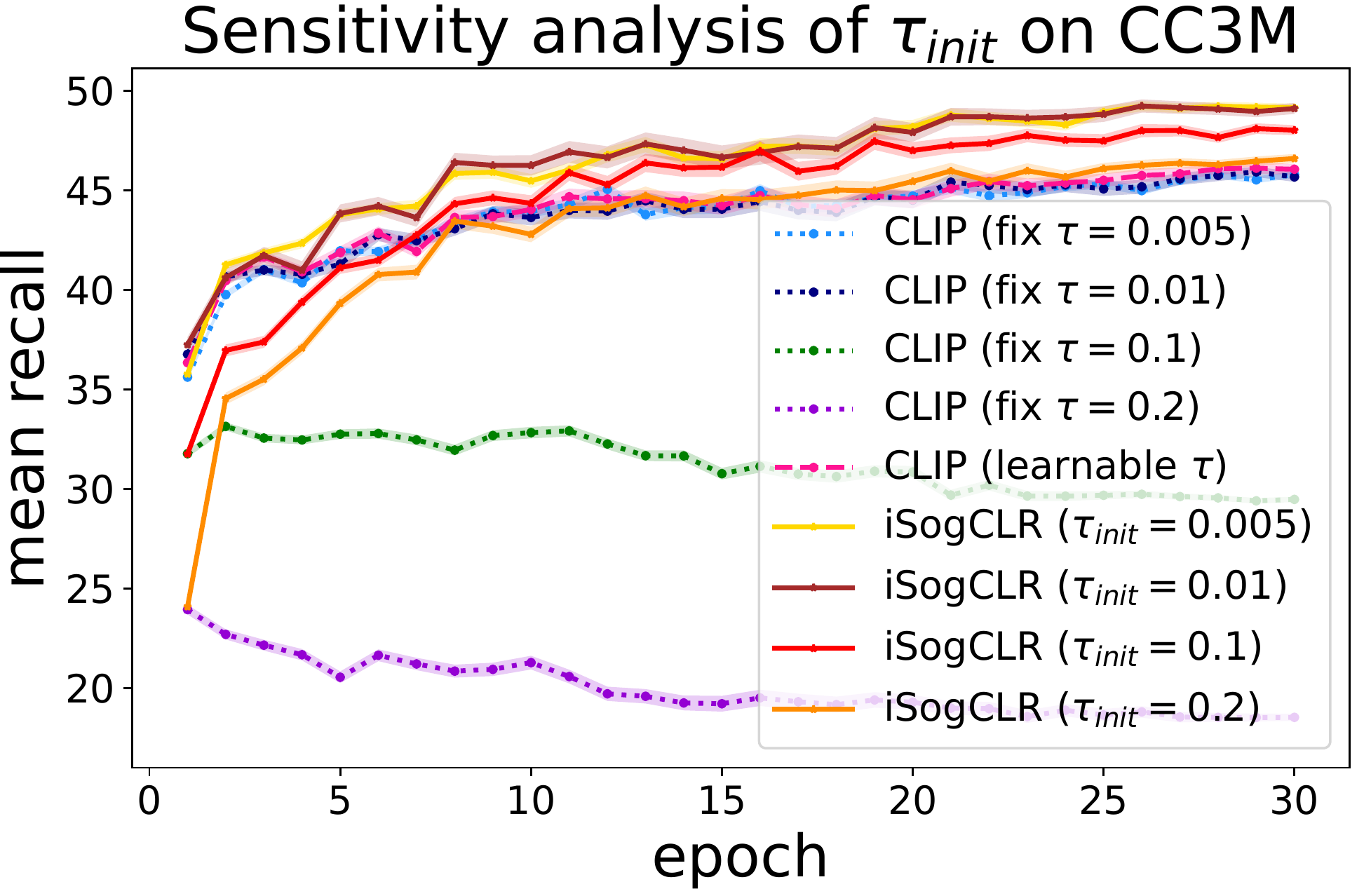}
\end{minipage}
\vspace{-0.2cm}
\caption{Effect of $\tau$ and $\tau_{\text{init}}$ on SimCLR/SogCLR and iSogCLR.}
\label{fig:ablation_tau_init}
\end{figure}

{\bf Results.} We present partial results in Table~\ref{tab:bimodal_results_main} and full results in Table~\ref{tab:bimodal_zs_retrieval_full_flickr},~\ref{tab:bimodal_zs_retrieval_full_coco},~\ref{tab:bimodal_zs_classification_full} in Appendix~\ref{app:add-exp-res}. Compared with baselines, our algorithm achieves significant improvements on both downstream tasks. Specifically, iSogCLR improves CLIP by 4\%$\sim$17\% and 2\%$\sim$8\% on image-text retrieval and zero-short classification, respectively. Large-scale bimodal data always contain long-tail underlying semantics~\citep{wang2022vlmixer}, thus the optimal $\tau$ of different samples may vary greatly. Hence iSogCLR with individualized temperatures is much more suitable than the methods with a global $\tau$.

\subsection{In-depth Analyses}
\label{sec:in_depth_analyses}

Here, we demonstrate that iSogCLR indeed assigns suitable temperatures to samples with different types of semantics. Specifically, we consider the following two scenarios.



{\bf Unimodal data.} 
We use CIFAR100-LT to study the characteristics of samples with different $\tau$ values. First, we select top-600 samples with large temperatures and bottom-600 samples with small temperatures. The \emph{class distributions} of these two groups of samples are in the left of Figure~\ref{fig:in_depth_cifar100_lt}.  We observe that samples with small $\tau$ account for a higher proportion of tail classes. Interestingly, although some of samples belong to tail classes, e.g., `sweet peppers', `streetcar' and `pickup truck', they are semantically similar to some head classes, e.g., `apple', `bus'. Thus these samples actually have frequent semantics and iSogCLR correctly assigns them large $\tau$ values. The right part of Figure~\ref{fig:in_depth_cifar100_lt} shows the projection of samples in these two groups. Note that most of samples with large $\tau$ values are in the centers of clusters, while most of samples with small $\tau$ values are separated from clusters. These results clearly show that iSogCLR makes samples with frequent semantics have large $\tau$ values to keep semantic structures, and makes samples with rare semantics have small $\tau$ values to be more discriminative.

{\bf Bimodal data.} First, we use the data of ``a kitten in a basket"  and  ``architectural details of a bridge"  for more illustrations. We show several hard negative pairs of the cat and bridge images in the left of Figure~\ref{fig:in_depth_cc3m}. 
One can observe that for the cat image, its hard negative pairs contain very similar semantics. For the bridge image, however, it has fewer hard negative pairs with similar semantics with it.  We also present learned features of 1500 random image-text pairs with highlights on several pairs with large and small $\tau$ values in Figure~\ref{fig:in_depth_cc3m} (right). Notice that images with large $\tau$ values are very close to their texts, while images with small $\tau$ values are far from their texts. The reason is that pairs with large $\tau$ values have frequent semantics, thus the model learns their patterns well and their features are well aligned. By contrast, the pairs with small $\tau$ values have rare semantics and their features are not learned so well. These results show that iSogCLR learns suitable temperatures for samples with different semantics. More samples in CC3M data are provided in Figure~\ref{fig:app_large_tau_images_cc3m} and~\ref{fig:app_small_tau_images_cc3m} in Appendix~\ref{app:add-exp-res}, showing that images with large $\tau$ values are related to frequent human activities or life scenes, while images with small $\tau$ values correspond to rare activities or scenes.


\begin{figure}[t]
\begin{minipage}{0.24\textwidth}
  \includegraphics[width=\linewidth]{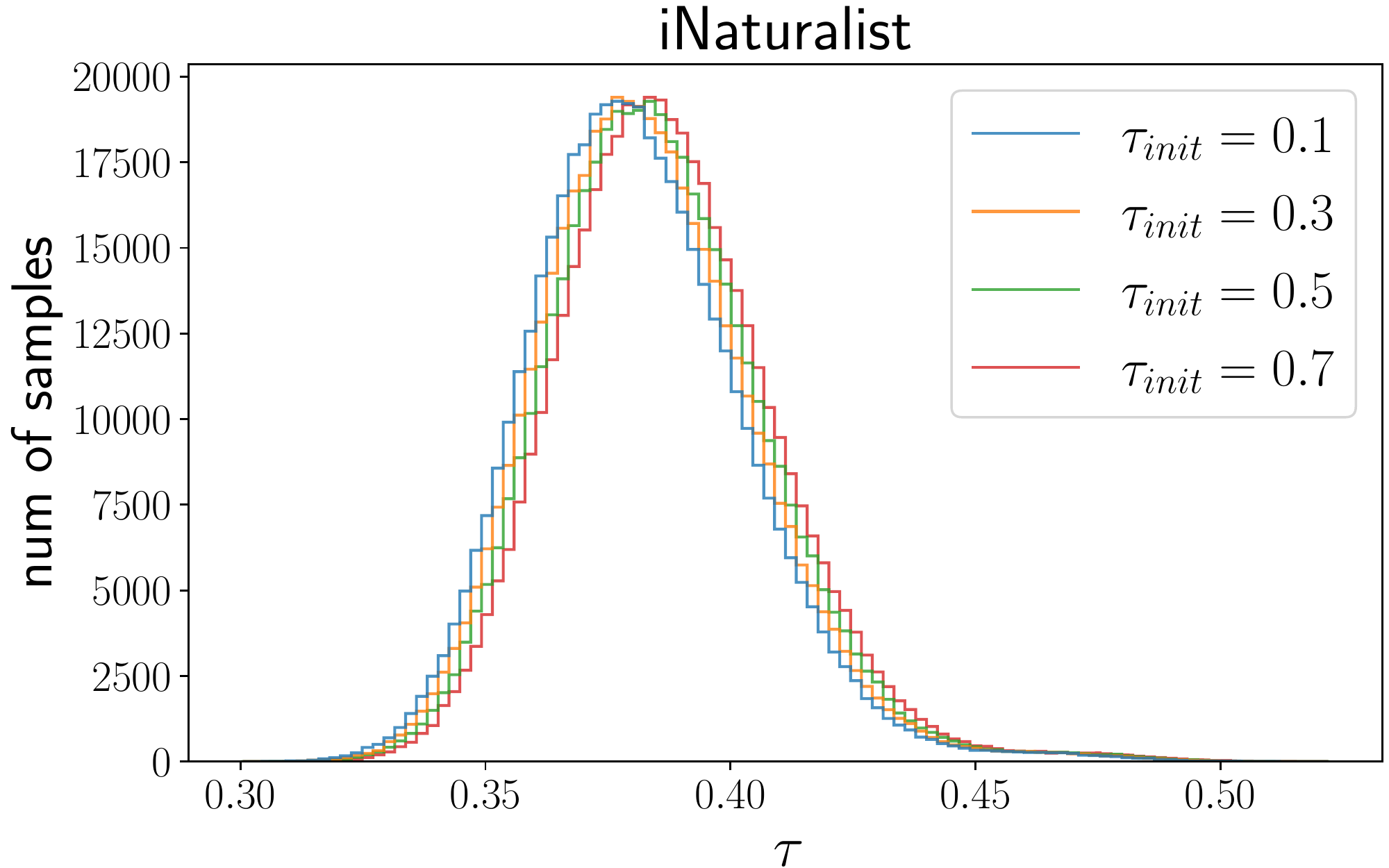}
\end{minipage}\hfil
\begin{minipage}{0.24\textwidth}
  \includegraphics[width=1.0\linewidth]{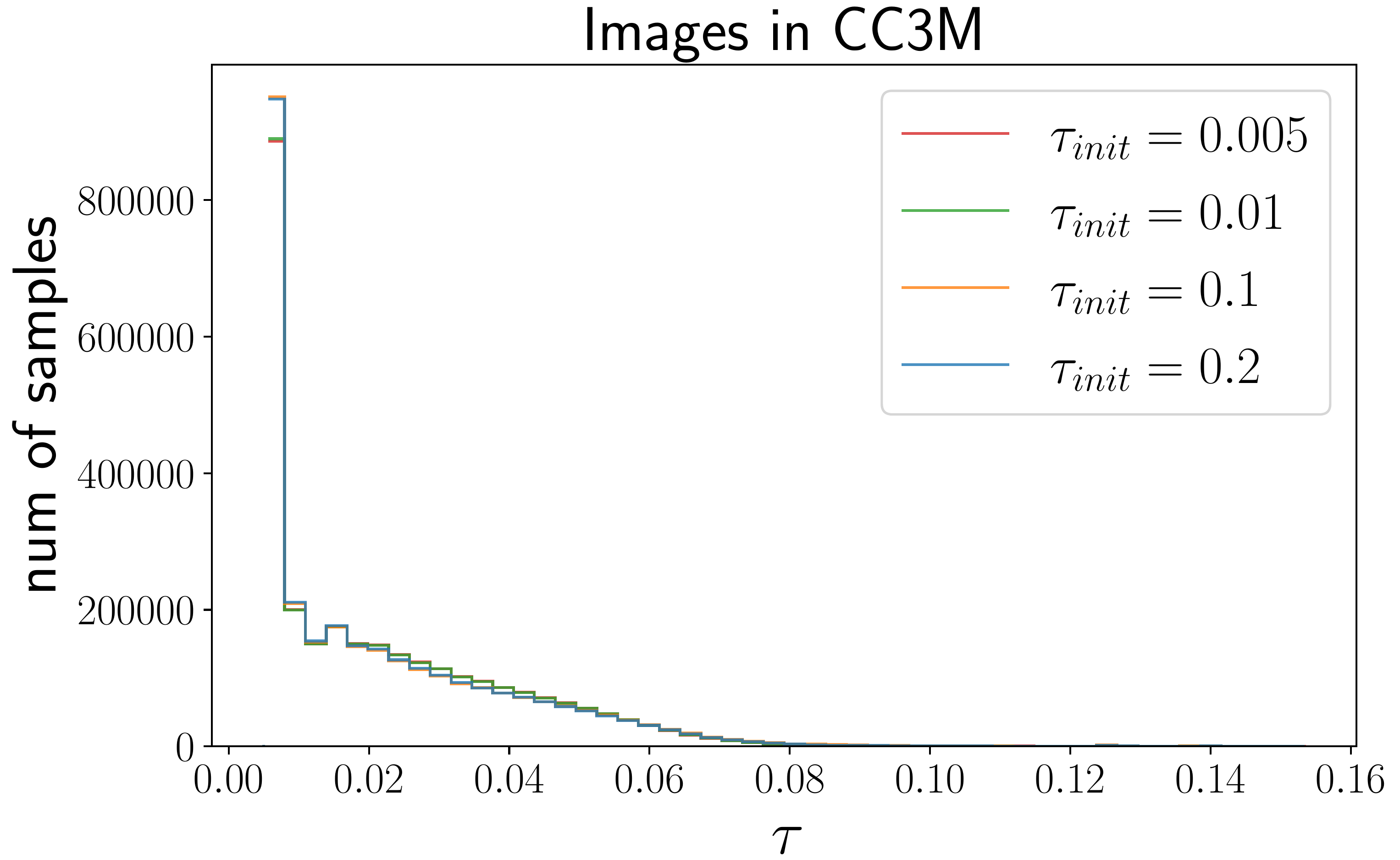}
\end{minipage}
\vspace{-0.25cm}
\caption{Final distributions of learned temperatures.}
\label{fig:ablation_properties_learn_tau_imagenet100}
\vskip -0.0in
\end{figure}

\subsection{Ablation Studies}
\vspace*{-0.05in}

In this section, we conduct extensive ablation studies to shed light on the behaviors of iSogCLR. First, we study the effect of different $\tau_{\text{init}}$ values on the final performance and convergence of iSogCLR, and provide the results in Figure~\ref{fig:ablation_tau_init}. From the left part of Figure~\ref{fig:ablation_tau_init}, one can observe that iSogCLR is not sensitive to $\tau_{\text{init}}$ and always outperforms SimCLR/SogCLR with a tuned $\tau$. More results are in Table~\ref{tab:ablation_temperature} in Appendix~\ref{app:add-exp-res}. The results on CC3M in the right of Figure~\ref{fig:ablation_tau_init} indicate that CLIP fails to converge when $\tau$ is fixed to a large value (e.g., 0.1, 0.2). On the contrary, regardless of the values of $\tau_{\text{init}}$, iSogCLR converges well and matches or outperforms CLIP with the tuned or learned $\tau$.

We further visualize the final distributions of learned temperatures on different datasets in Figure~\ref{fig:ablation_properties_learn_tau_imagenet100}. Note that these distributions are similar regardless of $\tau_{\text{init}}$ values. We also observe that the distributions on unimodal data usually follow a Gaussian distribution, while that on bimodal CC3M data has a long-tail (cf. Figure~\ref{fig:app_tau_dists} in Appendix~\ref{app:add-exp-res} for more results). It is interesting to observe that the learned temperatures for bimodal dataset are smaller than those for image datasets, which is consistent with the literature found by manual tuning~\citep{chen2020simple,liang2022mind}.  

Due to the limited space, more studies are provided in Appendix~\ref{app:add-exp-res}. In particular, the results about the effect of hyper-parameter $\rho$ in Table~\ref{tab:ablation_rho_full} indicate that iSogCLR is not sensitive to $\rho$. We also compare with other heuristic baselines with individualized learnable temperatures in Appendix~\ref{app:add-exp-res}, and show that our method is more advantageous for learning individualized temperatures (cf. Table~\ref{tab:compare_with_TaU_simclr}).

\section{Conclusion}
\vspace*{-0.05in}

In this work, we propose a novel method named iSogCLR for contrastive SSL with automatic temperature individualization. We first design a novel robust global contrastive objective based on DRO. Then we propose a provable stochastic algorithm. Theoretical and experimental results show that iSogCLR finds suitable temperatures for different samples. Comprehensive experiments demonstrate the effectiveness of iSogCLR on both unimodal and bimodal tasks.

\bibliography{tau-sogclr}
\bibliographystyle{icml2022}

\newpage
\appendix
\onecolumn

\section{Derivation of the Equivalent Minimization Form}
\label{app:minmax_derivation}

In this section, we present the detailed steps for the derivation of~(\ref{eq:dual_form_RGCL}). Recall the problem:
\begin{equation*}
\max_{\mathbf{p} \in \Delta} \min_{\lambda \geq 0}\!\sum_{\mathbf{z}_j \in \mathcal{S}^-_i}\!p_j h_i(\z_j)\!-\!\tau_0\text{KL}(\mathbf{p},\!\boldsymbol{1}/m)\!-\!\lambda(\text{KL}(\mathbf{p},\!\boldsymbol{1}/m)\!-\!\rho).
\end{equation*}
We first apply Sion's minimax theorem~\citep{sion1958general} and have:
\begin{equation*}
\min_{\lambda \geq 0} \max_{\mathbf{p} \in \Delta}\!\sum_{\mathbf{z}_j \in \mathcal{S}^-_i}\!p_j h_i(\z_j)\!-\!\tau_0\text{KL}(\mathbf{p},\!\boldsymbol{1}/m)\!-\!\lambda(\text{KL}(\mathbf{p},\!\boldsymbol{1}/m)\!-\!\rho),
\end{equation*}
which is equivalent to
\begin{equation*}
\min_{\lambda \geq 0} \max_{\mathbf{p} \in \Delta}\!\sum_{\mathbf{z}_j \in \mathcal{S}^-_i}\!p_j h_i(\z_j)\!-\!(\lambda+\tau_0)(\text{KL}(\mathbf{p},\!\boldsymbol{1}/m)-\rho)\!-\!\tau_0\rho.
\end{equation*}
Let $\tau=\lambda+\tau_0$, then we have
\begin{equation*}
\min_{\tau\geq \tau_0} \max_{\mathbf{p} \in \Delta}\!\sum_{\mathbf{z}_j \in \mathcal{S}^-_i}\!p_j h_i(\z_j)\!-\!\tau(\text{KL}(\mathbf{p},\!\boldsymbol{1}/m)-\rho)\!-\!\tau_0\rho.
\end{equation*}
Then, the original problem is equivalent to the following problem:
\begin{equation*}
\min_{\w}\min_{\tau\geq \tau_0} \max_{\mathbf{p} \in \Delta}\!\sum_{\mathbf{z}_j \in \mathcal{S}^-_i}\!p_j h_i(\z_j)\!-\!\tau(\text{KL}(\mathbf{p},\!\boldsymbol{1}/m)-\rho)\!-\!\tau_0\rho.
\end{equation*}
Next, we fix $\x=(\w^{\top},\tau)^{\top}$ and derive the optimal solution $\p^{*}(\x)$ that depends on $\x$ and solves the inner maximization problem. To this end, we consider the following problem
\begin{equation*}
\min_{\mathbf{p} \in \Delta}\!\sum_{\mathbf{z}_j \in \mathcal{S}^-_i}\!-p_j h_i(\z_j)\!+\!\tau\text{KL}(\mathbf{p},\!\boldsymbol{1}/m),
\end{equation*}
which has the same optimal solution as our original problem.
There are actually three constraints to handle, i.e., $p_i\geq 0,\forall{i}$, $p_i\leq 1,\forall{i}$ and $\sum_{i=1}^{m} p_i=1$. Note that the constraint $p_i\geq 0,\forall{i}$ is enforced by the term $p_i\log(p_i)$, otherwise the above objective will be infinity. Besides, the constraint $p_i\leq 1$ is automatically satisfied due to $\sum_{i=1}^{m}p_i=1$ and $p_i\geq 0,\forall{i}$. Hence, we only to explicitly tackle the constraint $\sum_{i=1}^{m}p_i=1$. To this end, we define the following Lagrangian function:
\begin{equation*}
    L_{\x}(\p,\mu) = \!\sum_{\mathbf{z}_j \in \mathcal{S}^-_i}\!-p_j h_i(\z_j)\!+\!\tau\left(\log m + \sum_{i=1}^{m}p_i\log(p_i)\right) + \mu\left(\sum_{i=1}^{m}p_i -1\right),
\end{equation*}
where $\text{KL}(\mathbf{p},\!\boldsymbol{1}/m)=\log m + \sum_{i=1}^{m}p_i\log(p_i)$, and $\mu$ is the Lagrangian multiplier for the constraint $\sum_{i=1}^{m}p_i=1$. The optimal solutions satisfy the KKT conditions:
\begin{equation*}
    -h_i(\z_j) + \tau(\log(p_j^*(\x)) + 1) + \mu=0\quad \text{and}\quad \sum_{i=1}^{m} p_i^*(\x)=1.
\end{equation*}
From the first equation, we can derive $p_j^*(\x)\propto\exp(h_i(\z_j)/\tau)$. Due to the second equation, we conclude that $p_j^*(\x)=\frac{\exp(h_i(\z_j)/\tau)}{\sum_{\z_j\in\mathcal{S}^-_i}\exp(h_i(\z_j)/\tau)}$. Plugging this optimal $\p^*$ into the inner maximization problem over $\p$, we have
\begin{equation*}
    \sum_{\z_j\in\mathcal{S}^-_i}p_j^*(\x)h_i(\z_j) -\tau\left(\log m + \sum_{i=1}^{m}p_i^*(\x)\log(p_i^*(\x)) \right)\!=\!\tau\log\left(\frac{1}{m}\sum_{\mathbf{z}_j \in \mathcal{S}^-_i}\exp\left(\frac{h_i(\z_j)}{\tau}\right)\right)\!=\!\tau\log\left(\E_{\mathbf{z}_j \in \mathcal{S}^-_i}\exp\left(\frac{h_i(\z_j)}{\tau}\right)\right).
\end{equation*}
Therefore, we get the following equivalent problem:
\begin{equation*}
    \min_{\tau\geq\tau_0} \tau\log\left(\E_{\mathbf{z}_j \in \mathcal{S}^-_i}\exp\left(\frac{h_i(\z_j)}{\tau}\right)\right) + (\tau-\tau_0)\rho,
\end{equation*}
which is the dual form in~(\ref{eq:dual_form_RGCL}) of the original RGCL. The dual form for RGCL in bimodal setting can be derived in a similar way.

\section{iSogCLR for Bimodal CL Setting}
\label{app:bimodal_alg}

Recall the RGCO for bimodal SSL:
\begin{equation*}
\min_{\w,\bftau,\bftau'\geq\tau_0}\!F_{\text{B}}(\w,\bftau,\bftau')\!:=\!\frac{1}{n}\sum\nolimits_{(\x_i,\t_i)\in\D'}\Bigg\{(\bftau_{i}+\bftau'_{i})\rho  +\left.\bftau_{i}\log{\E}_{\t\in\T_i^-}\exp\!\left(\frac{h_{\x_i}(\t)}{\bftau_{i}}\right)\!+\!\bftau'_{i}\log{\E}_{\x\in\mathcal I^-_i}\!\exp\!\left(\frac{h_{\t_i}(\x)}{\bftau'_{i}}\!\right)\!\right\},
\end{equation*}
where
\begin{align*}
&h_{\x_i}(\t)=E_I(\x_i)^{\top}E_T(\t)-E_I(\x_i)^{\top}E_T(\t_i),\\
&h_{\t_i}(\x)=E_I(\x)^{\top}E_T(\t_i)-E_I(\x_i)^{\top}E_T(\t_i).
\end{align*}
It is worth to mention that an image-text pair can be viewed as \emph{two views of the same underlying concept}. {\bf So essentially, bimodal RGCO is consistent with unimodal RGCO because they all construct positive (resp. negative) pairs from the the different views of the same (resp. different) concepts, and pull close positive pairs and push away negative pairs. The only difference is the bimodal loss gets views from different modalities while the unimodal loss gets views from different augmentations}. Our algorithm is general for softmax-base contrastive loss and does not mind how to extract the views. Therefore it is applicable to both unimodal and bimodal CL. 

The algorithm for optimizing $F_{\text{B}}(\w,\bftau,\bftau')$ is very similar to that for optimizing unimodal RGCO $F(\w,\bftau)$ in Algorithm~\ref{algo:sogclr_dro}. Note that we employ the subscript `v' and `t' to represent variables for \underline{v}isual images and \underline{t}exts, respectively. At each iteration, we sample a random mini-batch of $B^{\prime}$ image-text pairs $\B^{\prime}=\{\x_1,\t_1,\ldots,\x_{B^{\prime}},\t_{B^{\prime}}\}$. Then we compute the stochastic estimators of $g_{\x_i}(\w_t,\bftau_{\text{v},i};\mathcal{T}^{\prime}_i)$ and $g_{\t_i}(\w_t,\bftau_{\text{t},i};\mathcal{I}^{\prime}_i)$ by
\begin{align}
   & g_{\x_i}(\w_t,\bftau_{\text{v},i};\mathcal{T}^{\prime}_i) = \frac{1}{|\mathcal{T}^{\prime}_i|}\sum_{\t\in \mathcal{T}^{\prime}_i}\exp\left(\frac{h_{\x_i}(\t)}{\bftau_{\text{v},i}}\right),\label{eq:bimodal_isogclr_g1} \\ 
   & g_{\t_i}(\w_t,\bftau_{\text{t},i};\mathcal{I}^{\prime}_i) = \frac{1}{|\mathcal{I}^{\prime}_i|}\sum_{\x\in \mathcal{I}^{\prime}_i}\exp\left(\frac{h_{\t_i}(\x)}{\bftau_{\text{t},i}}\right),\label{eq:bimodal_isogclr_g2}
\end{align}
where $\mathcal{I}^{\prime}_i=\{\x_1,\ldots,\x_{B^{\prime}}\}\backslash\{\x_i\}$ and $\mathcal{T}^{\prime}_i=\{\t_1,\ldots,\t_{B^{\prime}}\}\backslash\{\t_i\}$. To control the approximation error, we maintain the following two moving average estimators:
\begin{align}
&\s_{\text{v},i}^{t+1}=(1-\beta_0)\s_{\text{v},i}^t+\beta_0 g_{\x_i}(\w_t,\bftau_{\text{v},i};\mathcal{T}^{\prime}_i),\label{eq:bimodal_isogclr_s1} \\
&\s_{\text{t},i}^{t+1}=(1-\beta_0)\s_{\text{t},i}^t+\beta_0 g_{\t_i}(\w_t,\bftau_{\text{t},i};\mathcal{I}^{\prime}_i).\label{eq:bimodal_isogclr_s2}
\end{align}
where $\beta_0\in(0,1)$. With these estimators, we can compute the gradients of $F_B(\w,\bftau)$ w.r.t. $\w$, $\bftau_{\text{v}}$, and $\bftau_{\text{t}}$ by
\begin{align}
&G(\bftau_{\text{v},i}^t) = \frac{1}{n} \left[\frac{\bftau_{\text{v},i}^t}{\s_{\text{v},i}^t}\nabla_{\bftau_{\text{v},i}}g_{\x_i}(\w_t,\bftau_{\text{v},i};\mathcal{T}^{\prime}_i) + \log(\s_{\text{v},i}^t) + \rho\right]\label{eq:grad_tau_bimodal_1},\\
&G(\bftau_{\text{t},i}^t) = \frac{1}{n} \left[\frac{\bftau_{\text{t},i}^t}{\s_{\text{t},i}^t}\nabla_{\bftau_{\text{t},i}}g_{\t_i}(\w_t,\bftau_{\text{t},i};\mathcal{I}^{\prime}_i) + \log(\s_{\text{t},i}^t) + \rho\right]\label{eq:grad_tau_bimodal_2},\\
&G(\w_t) = \frac{1}{|\B^{\prime}|}\sum_{\x_i,\t_i\in\B^{\prime}}\left(\frac{\bftau_{\text{v},i}^t}{\s_{\text{v},i}^t}\nabla_{\w}g_{\x_i}(\w_t,\bftau_{\text{v},i};\mathcal{T}^{\prime}_i) + \frac{\bftau_{\text{t},i}^t}{\s_{\text{t},i}^t}\nabla_{\w}g_{\t_i}(\w_t,\bftau_{\text{t},i};\mathcal{I}^{\prime}_i)\right).\label{eq:grad_w_bimodal}
\end{align}

We present the detailed steps of using the momentum-style update in Algorithm~\ref{algo:sogclr_dro_bimodal}. A similar convergence guarantee to Theorem~\ref{thm:short_stat} can be established for iSogCLR in bimodal setting. The momentum-style update can be replaced by an Adam-style update using adaptive step sizes, and the same convergence rate can be established.

\begin{algorithm}[t]
\caption{iSogCLR for Bimodal SSL}\label{algo:sogclr_dro_bimodal}
\begin{algorithmic}[1]
\REQUIRE $\beta_0, \beta_1, \eta$
\STATE Initialize $\w_1,\s_\text{v}^1, \s_\text{t}^1, \u_\text{v}^1,\u_\text{t}^1, \v_1$, $\bftau_\text{v}^1=\bftau_\text{t}^1=\boldsymbol{\tau_{\text{init}}}$
\FOR{$t=1,2,\dots,T$}
\STATE Draw a batch of $B^{\prime}$ samples denoted by $\B^{\prime}\subset\D^{\prime}$
\FOR{$\x_i\in\B^{\prime}$}
\STATE Compute $g_{\x_i}(\w_t,\bftau_{\text{v},i};\mathcal{T}^{\prime}_i)$ and $g_{\t_i}(\w_t,\bftau_{\text{t},i};\mathcal{I}^{\prime}_i)$ according to~(\ref{eq:bimodal_isogclr_g1}) and~(\ref{eq:bimodal_isogclr_g2}), respectively
\STATE Update $\s_{\text{v},i}^{t+1}$ and $\s_{\text{t},i}^{t+1}$ according to~(\ref{eq:bimodal_isogclr_s1}) and~(\ref{eq:bimodal_isogclr_s2}), respectively 
\STATE Compute $G(\bftau_{\text{v},i}^t)$ and $G(\bftau_{\text{i},i}^t)$ according to~(\ref{eq:grad_tau_bimodal_1}) and~(\ref{eq:grad_tau_bimodal_2}), respectively
\STATE Update $\u_{\text{v},i}^{t+1}=(1-\beta_1)\u_{\text{v},i}^t+\beta_1 G(\bftau_{\text{v},i}^t)$ and $\u_{\text{t},i}^{t+1}=(1-\beta_1)\u_{\text{t},i}^t+\beta_1 G(\bftau_{\text{t},i}^t)$
\STATE Update $\bftau_{\text{v},i}^{t+1}=\Pi_{\Omega}\left[\bftau_{\text{v},i}^t - \eta \u_{\text{v},i}^{t+1} \right]$ and $\bftau_{\text{t},i}^{t+1}=\Pi_{\Omega}\left[\bftau_{\text{t},i}^t - \eta \u_{\text{t},i}^{t+1} \right]$
\ENDFOR
\STATE Compute gradient estimator $G(\w_t)$ according to~(\ref{eq:grad_w_bimodal})
\STATE Compute $\v_{t+1}=(1-\beta_1)\v_t + \beta_1 G(\w_t)$
\STATE Update $\w_{t+1}=\w_t - \eta \v_{t+1}$ \text{(or Adam-style)}
\ENDFOR
\end{algorithmic}
\end{algorithm}

\section{Experiments}
\label{app:exp}

\subsection{Details of Implementation}
\label{app:exp-impl}

For experiments on unimodal image datasets, we compare our algorithm, iSogCLR, against the following methods. {\bf SimCLR}~\citep{chen2020simple} is a pioneering work that directly optimize InfoNCE loss~\citep{oord2018representation}. {\bf FlatCLR}~\citep{chen2021simpler} employs a variant of InfoNCE loss for better performance in the small-batch-size regime. {\bf Spectral CL}~\citep{haochen2021provable} is based on spectral decomposition on population graph and has provable accuracy guarantees. {\bf SogCLR}~\citep{yuan2022provable} utilizes variance reduction techniques to achieve promising performance and has provable convergence guarantees. {\bf SimCo}~\citep{zhang2022dual} improves negative mining in CL by using dual temperatures. {\bf Barlow Twins}~\citep{zbontar2021barlow} and {\bf VICReg}~\citep{bardes2021vicreg} are non-contrastive methods and aim to maximize the information content of embeddings. On bimodal visual-language datasets, we consider the following baselines. {\bf CLIP}~\citep{radford2021learning} is one of the most popular VLP framework. {\bf CyCLIP}~\citep{goel2022cyclip} try to improve CLIP by optimizing the features to be geometrically consistent on image and text space. {\bf SogCLR} can also be applied to solve bimodal SSL problems and is included in our comparison.

For unimodal experiments, we adopt a code base from GitHub\footnote{https://github.com/HobbitLong/SupContrast} and implement the baseline methods in our experiments based on their open source implementations. The backbone networks we use are ResNet-18 and ResNet-50 for experiments on CIFAR dataset and ImageNet100/iNaturalist, respectively. For the projection head, we employ that used by VICReg~\citep{bardes2021vicreg} for all methods. For bimodal experiments, we conduct experiments on the basis of ALBEF\footnote{https://github.com/salesforce/ALBEF}~\citep{li2021align}. We also implement bimodal CL baselines, e.g., CLIP, CyCLIP, and SogCLR, in the code base. We adopt ResNet-50 as the image encoder and DistilBert~\citep{sanh2019distilbert} as the text encoder. We train our models on Nvidia Tesla V100 GPU with 32GB memory and GTX 3090 GPU with 24GB memory.

\subsection{Details of Datasets}
\label{app:exp-data}

CIFAR-10 and CIFAR-100 are two widely-used image datasets. Both of them contain 50,000 images for training and 10,000 images for test. The full version of ImageNet contains 1000 classes (about 1.2M images) and we denote it as ImageNet-1K~\citep{russakovsky2015imagenet}. ImageNet-100~\citep{wu2019large} is a subset with randomly selected 100 classes (about 128K image) from ImageNet-1K. We also consider two imbalanced datasets: CIFAR100-LT and ImageNet-LT. We construct CIFAR100-LT following a widely-used strategy in the literature~\citep{cao2019learning,qi2022stochastic} with the imbalance ratio $\rho$=100, and keep the test set unchanged. The imbalance ratio $\rho$ is defined as the ratio between sample sizes of the most frequent and least frequent classes. The LT imbalance follows the exponentially decayed sample size between different classes. The iNaturalist species classification and detection dataset~\citep{iNat18} is a real-world large-scale dataset with 437,513 images from 8142 classes in its 2018 version.

Conceptual Captions 3M (CC3M) dataset~\citep{sharma2018conceptual} contains about 2.9 million image-caption pairs crawled from the Internet. Note that as time goes by, some images are not available. Thus the number of image-caption pairs we use in our experiments is smaller than that in the original papers. Each image in MSCOCO and Flickr30K datasets has about 5 captions. MSCOCO dataset~\citep{lin2014microsoft} contains 113K images and 567K captions, and Flickr30K dataset~\citep{plummer2015flickr30k} has 32K images and 158K captions. We employ the well-known Karpathy split~\citep{karpathy2015deep} for these two datasets.

\subsection{Additional Experimental Results}
\label{app:add-exp-res}

\begin{table*}[t]
\vspace{-4mm}
\caption{Linear evaluation (top-1 accuracy (\%)) under different training epochs on three balanced unimodal image datasets.}
\vspace{-0mm}
\label{tab:unimodal_datasets_results_full_balance}
\vskip 0.1in
\begin{center}
\begin{small}
\begin{sc}
\begin{tabular}{p{2.2cm}p{1.6cm}p{1.6cm}p{1.6cm}p{1.6cm}p{1.6cm}p{1.6cm}}
\toprule
\multirow{1}{*}{\thead{Method}} &
\multicolumn{2}{c}{\thead{CIFAR10}} &
\multicolumn{2}{c}{\thead{CIFAR100}} &
\multicolumn{2}{c}{\thead{ImageNet100}} \\
\cmidrule(lr){2-3}
\cmidrule(lr){4-5}
\cmidrule(lr){6-7}
& 400ep & 800ep & 400ep & 800ep  & 200ep & 400ep \\
\midrule
SimCLR & 88.74$\pm$0.18 & 89.64$\pm$0.12 & 62.34$\pm$0.09 & 64.78$\pm$0.14 & 78.84$\pm$0.18 & 79.96$\pm$0.20  \\
Barlow Twins & 87.39$\pm$0.14 & 88.39$\pm$0.16 & 62.28$\pm$0.13 & 64.33$\pm$0.13 & 77.02$\pm$0.14 & 79.16$\pm$0.13  \\
FlatCLR & 88.61$\pm$0.10 & 89.22$\pm$0.06 & 63.27$\pm$0.07 & 64.51$\pm$0.08 & 79.06$\pm$0.09 & 80.24$\pm$0.16 \\
Spectral CL & 88.77$\pm$0.09 & \textbf{90.30}$\pm$0.11 & 63.06$\pm$0.18 & 64.32$\pm$0.17 & 78.38$\pm$0.17 & 80.48$\pm$0.08 \\
SogCLR & 88.93$\pm$0.11 & 90.07$\pm$0.10 & 63.14$\pm$0.12 & 65.18$\pm$0.10 & 79.12$\pm$0.07  & 80.54$\pm$0.14  \\
VICReg & 88.96$\pm$0.16 & 89.90$\pm$0.12 & 62.44$\pm$0.13 & 64.18$\pm$0.09 & \textbf{79.58}$\pm$0.23 & 80.16$\pm$0.22  \\
SimCo & 88.86$\pm$0.12 & 89.79$\pm$0.15 & 62.67$\pm$0.06 & 64.74$\pm$0.12 & 77.36$\pm$0.16 & 79.73$\pm$0.17  \\
iSogCLR & \textbf{89.24}$\pm$0.15 & 90.25$\pm$0.09 & \textbf{63.82}$\pm$0.14 & \textbf{65.95}$\pm$0.07 & 79.42$\pm$0.15 & \textbf{81.14}$\pm$0.19 \\
\bottomrule
\end{tabular}
\end{sc}
\end{small}
\end{center}
\vskip -0.0in
\end{table*}

\begin{table*}[t]
\vspace{-4mm}
\caption{Linear evaluation (top-1 accuracy (\%)) under different training epochs on three imbalanced unimodal image datasets.}
\vspace{-0mm}
\label{tab:unimodal_datasets_results_full_imbalance}
\vskip 0.1in
\begin{center}
\begin{small}
\begin{sc}
\begin{tabular}{p{2.2cm}p{1.6cm}p{1.6cm}p{1.6cm}p{1.6cm}p{1.6cm}p{1.6cm}}
\toprule
\multirow{1}{*}{\thead{Method}} &
\multicolumn{2}{c}{\thead{CIFAR10-LT}} &
\multicolumn{2}{c}{\thead{CIFAR100-LT}} &
\multicolumn{2}{c}{\thead{iNaturalist}}  \\
\cmidrule(lr){2-3}
\cmidrule(lr){4-5}
\cmidrule(lr){6-7}
& 400ep & 800ep & 400ep & 800ep  & 200ep & 400ep \\
\midrule
SimCLR  & 77.09$\pm$0.13 & 78.36$\pm$0.07 & 49.33$\pm$0.12 & 51.89$\pm$0.09 & 90.79$\pm$0.14 & 91.52$\pm$0.17 \\
Barlow Twins & 75.94$\pm$0.08 & 77.12$\pm$0.14 & 48.39$\pm$0.14 & 50.74$\pm$0.15 & 90.57$\pm$0.22 & 91.89$\pm$0.21 \\
FlatCLR & 77.96$\pm$0.12 & 79.19$\pm$0.08 & 52.61$\pm$0.06 & 54.14$\pm$0.08 & 91.48$\pm$0.15 & 92.54$\pm$0.09 \\
Spectral CL & 76.38$\pm$0.21 & 78.63$\pm$0.13 & 51.86$\pm$0.16 & 53.46$\pm$0.17 & 91.28$\pm$0.11 & 92.13$\pm$0.16 \\
SogCLR & 77.70$\pm$0.07 & 79.16$\pm$0.09 & 52.35$\pm$0.08 & 53.58$\pm$0.13 & 91.89$\pm$0.18 & 92.60$\pm$0.08\\
VICReg & 75.05$\pm$0.09 & 77.84$\pm$0.15 & 48.43$\pm$0.13 & 51.68$\pm$0.06 & 92.18$\pm$0.06 & 93.03$\pm$0.14 \\
SimCo & 77.71$\pm$0.13 & 78.56$\pm$0.19 & 51.06$\pm$0.09 & 52.31$\pm$0.14 & 91.03$\pm$0.18 & 92.10$\pm$0.12 \\
iSogCLR & \textbf{78.37}$\pm$0.16 & \textbf{79.69}$\pm$0.08 & \textbf{53.06}$\pm$0.12 & \textbf{54.42}$\pm$0.18 & \textbf{92.33}$\pm$0.23 & \textbf{93.08}$\pm$0.19 \\
\bottomrule
\end{tabular}
\end{sc}
\end{small}
\end{center}
\vskip -0.0in
\end{table*}

{\bf Unimodal experimental results.} We present the full results on three balanced datasets and three imbalanced datasets in Table~\ref{tab:unimodal_datasets_results_full_balance} and Table~\ref{tab:unimodal_datasets_results_full_imbalance}, respectively. One can observe than our iSogCLR matches or outperforms prior strong baselines.

\begin{table*}[t]
\vspace{-4mm}
\caption{Zero-shot image-text retrieval (text-to-image and image-to-text) results (Recall@$k$), where $k\in\{1,5,10\}$, on Flickr30K dataset.}
\vspace{-4mm}
\label{tab:bimodal_zs_retrieval_full_flickr}
\vskip 0.2in
\begin{center}
\begin{small}
\begin{sc}
\begin{tabular}{p{1.3cm}p{1.8cm}p{1.8cm}p{1.8cm}p{1.8cm}p{1.8cm}p{1.8cm}}
\toprule
\multirow{1}{*}{\thead{Method}} &
\multicolumn{3}{c}{\thead{Image retrieval}} &
\multicolumn{3}{c}{\thead{Text retrieval}} \\
\cmidrule(lr){2-4}
\cmidrule(lr){5-7}
& R@1 & R@5 & R@10 & R@1 & R@5 & R@10  \\
\midrule
CLIP & 40.98$\pm$0.22 & 69.60$\pm$0.19 & 79.22$\pm$0.08 & 50.90$\pm$0.17 & 81.00$\pm$0.16 & 87.90$\pm$0.22 \\
CyCLIP & 42.46$\pm$0.13 & 69.56$\pm$0.16 & 78.74$\pm$0.21 & 51.70$\pm$0.23 & 79.90$\pm$0.18 & 88.40$\pm$0.11  \\
SogCLR & 43.32$\pm$0.18 & 71.06$\pm$0.13 & 79.54$\pm$0.19 & 57.18$\pm$0.20 & 81.03$\pm$0.26 & 88.62$\pm$0.18 \\
iSogCLR & \textbf{44.36}$\pm$0.12 & \textbf{72.64}$\pm$0.17 & \textbf{80.92}$\pm$0.13 & \textbf{60.20}$\pm$0.26 & \textbf{84.60}$\pm$0.21 & \textbf{90.50}$\pm$0.14 \\
\bottomrule
\end{tabular}
\end{sc}
\end{small}
\end{center}
\vskip -0.0in
\end{table*}

\begin{table*}[t]
\vspace{-4mm}
\caption{Zero-shot image-text retrieval (text-to-image and image-to-text) results (Recall@$k$), where $k\in\{1,5,10\}$, on MSCOCO dataset.}
\vspace{-4mm}
\label{tab:bimodal_zs_retrieval_full_coco}
\vskip 0.2in
\begin{center}
\begin{small}
\begin{sc}
\begin{tabular}{p{1.3cm}p{1.8cm}p{1.8cm}p{1.8cm}p{1.8cm}p{1.8cm}p{1.8cm}}
\toprule
\multirow{1}{*}{\thead{Method}} &
\multicolumn{3}{c}{\thead{Image retrieval}} &
\multicolumn{3}{c}{\thead{Text retrieval}} \\
\cmidrule(lr){2-4}
\cmidrule(lr){5-7}
& R@1 & R@5 & R@10 & R@1 & R@5 & R@10 \\
\midrule
CLIP & 21.32$\pm$0.12 & 45.52$\pm$0.17 & 57.30$\pm$0.16 & 26.98$\pm$0.21 & 54.86$\pm$0.15 & 66.86$\pm$0.19 \\
CyCLIP & 21.58$\pm$0.19 & 45.46$\pm$0.13 & 57.56$\pm$0.22 & 26.18$\pm$0.24 & 53.24$\pm$0.18 & 65.86$\pm$0.22 \\
SogCLR & 22.43$\pm$0.13 & 46.74$\pm$0.11 & 58.32$\pm$0.20 & 30.08$\pm$0.22 & 56.94$\pm$0.17 & 67.39$\pm$0.24 \\
iSogCLR & \textbf{23.27}$\pm$0.18 & \textbf{47.23}$\pm$0.24 & \textbf{59.07}$\pm$0.19 & \textbf{32.72}$\pm$0.13 & \textbf{59.52}$\pm$0.11 & \textbf{70.78}$\pm$0.21 \\
\bottomrule
\end{tabular}
\end{sc}
\end{small}
\end{center}
\vskip -0.0in
\end{table*}

\begin{table*}[t]
\vspace{-4mm}
\caption{Zero-shot top-$k$ classification accuracy (\%), where $k\in\{1,3,5\}$.}
\vspace{-4mm}
\label{tab:bimodal_zs_classification_full}
\vskip 0.2in
\begin{center}
\begin{small}
\begin{sc}
\begin{tabular}{p{1.2cm}p{1.6cm}p{1.6cm}p{1.6cm}p{1.6cm}p{1.6cm}p{1.6cm}p{1.6cm}p{1.6cm}p{1.6cm}}
\toprule
\multirow{1}{*}{\thead{Method}} &
\multicolumn{3}{c}{\thead{CIFAR10}} &
\multicolumn{3}{c}{\thead{CIFAR100}}  \\
\cmidrule(lr){2-4}
\cmidrule(lr){5-7}
& top-1 & top-3 & top-5 & top-1 & top-3 & top-5  \\
\midrule
CLIP & 60.63$\pm$0.19 & 87.29$\pm$0.12 & \textbf{95.02}$\pm$0.16 & 30.70$\pm$0.11 & 49.49$\pm$0.13 & 58.51$\pm$0.14 \\
CyCLIP & 57.19$\pm$0.20 & 85.02$\pm$0.14 & 93.94$\pm$0.23 & 33.11$\pm$0.14 & 52.99$\pm$0.17 & 61.01$\pm$0.22 \\
SogCLR & \textbf{61.09}$\pm$0.24 & \textbf{88.12}$\pm$0.19 & 94.92$\pm$0.18 & 33.26$\pm$0.12 & \textbf{52.46}$\pm$0.22 & 60.71$\pm$0.15  \\
iSogCLR & 58.91$\pm$0.15 & 86.27$\pm$0.24 & 93.43$\pm$0.11 & \textbf{33.81}$\pm$0.18 & 53.21$\pm$0.21 & \textbf{61.83}$\pm$0.19  \\
\bottomrule
\end{tabular}
\begin{tabular}{p{1.2cm}p{1.6cm}p{1.6cm}p{1.6cm}}
\toprule
\multirow{1}{*}{\thead{Method}} &
\multicolumn{3}{c}{\thead{ImageNet1K}} \\
\cmidrule(lr){2-4}
& top-1 & top-3 & top-5  \\
\midrule
CLIP & 36.27$\pm$0.17 & 51.03$\pm$0.17 & 56.84$\pm$0.22 \\
CyCLIP  & 36.75$\pm$0.21 & 51.32$\pm$0.18 & 57.08$\pm$0.23 \\
SogCLR  & 37.46$\pm$0.19 & 52.68$\pm$0.16 & 58.04$\pm$0.10 \\
iSogCLR  & \textbf{40.72}$\pm$0.23 & \textbf{54.38}$\pm$0.14 & \textbf{59.11}$\pm$0.17 \\
\bottomrule
\end{tabular}
\end{sc}
\end{small}
\end{center}
\vskip -0.0in
\end{table*}

{\bf Bimodal experimental results.} We provide the full results of the zero-shot image-text retrieval tasks on Flickr30K and MSCOCO in Table~\ref{tab:bimodal_zs_retrieval_full_flickr} and Table~\ref{tab:bimodal_zs_retrieval_full_coco}, respectively. It is notable that our method has large improvements compared with baselines. We also present the full results of the zero-shot classification tasks on three standard image datasets in Table~\ref{tab:bimodal_zs_classification_full}, and observe that our method achieves the best performance in most cases.

{\bf More ablation studies} 

{\bf Effect of $\tau_{\text{init}}$.} We present more ablation studies on the hyper-parameters of iSogCLR. In Table~\ref{tab:ablation_temperature}, we first present the effect of $\tau$ and $\tau_{\text{init}}$ on the performance of SimCLR and iSogCLR, respectively. One can observe that $\tau$ is an important hyper-parameter for SimCLR. SimCLR equiped with a tuned $\tau$ can be a strong baseline on many dataset. Besides, we find that our iSogCLR is not sensitive to $\tau_{\text{init}}$ in a range of 0.1$\sim$0.7. Moreover, iSogCLR with any $\tau_{\text{init}}$ is this range can outperforms SimCLR with a tuned $\tau$. These results demonstrate the effectiveness of our method.

\begin{table*}[t]
\vspace{-4mm}
\caption{The effect of $\tau$ ($\tau_{\text{init}}$) to SimCLR (iSogCLR). We report top-1 accuracy after pretraining for 400 epochs.}
\vspace{-4mm}
\label{tab:ablation_temperature}
\vskip 0.2in
\begin{center}
\begin{small}
\begin{sc}
\begin{tabular}{p{1.2cm}p{0.8cm}p{0.8cm}p{0.8cm}p{0.8cm}p{0.8cm}p{0.8cm}p{0.8cm}p{0.8cm}p{0.8cm}p{0.8cm}p{0.8cm}p{0.8cm}}
\toprule
\multirow{1}{*}{\thead{Method}} &
\multicolumn{4}{c}{\thead{CIFAR10}} &
\multicolumn{4}{c}{\thead{CIFAR100}} &
\multicolumn{4}{c}{\thead{ImageNet100}} \\
\cmidrule(lr){2-5}
\cmidrule(lr){6-9}
\cmidrule(lr){10-13}
& 0.1 & 0.3 & 0.5 & 0.7 & 0.1 & 0.3 & 0.5 & 0.7 & 0.1 & 0.3 & 0.5 & 0.7   \\
\midrule
SimCLR & 85.85 & 88.34 & 88.74 & 88.27 & 60.49 & 62.34 & 62.02 & 61.73 & 78.64 & 79.96 & 79.78 & 79.42 \\ 
iSogCLR & \textbf{89.00} & \textbf{89.17} & \textbf{89.24} & \textbf{89.23} & \textbf{63.30} & \textbf{63.73} & \textbf{63.41} & \textbf{63.50} & \textbf{80.82} & \textbf{80.90} & \textbf{80.86} & \textbf{81.14} \\ 
\bottomrule
\end{tabular}
\end{sc}
\end{small}
\end{center}
\vskip -0.0in
\end{table*}

{\bf Effect of $\rho$.} We provide the effect of $\rho$ on the performance of iSogCLR in Table~\ref{tab:ablation_rho_full}. We observe that although the parameter $\rho$ in RGCL affects the degree of hardness-awareness, this parameter does not have a big impact on the performance of iSogCLR in most cases. We believe the reason is that we introduce a learnable Lagrangian multiplier $\lambda$ for each KL constraint in our derivation. Thus the degree of hardness-awareness of each anchor data is largely affected by $\lambda$, i.e., the individualized temperature, which is flexible and updated during learning.

\begin{table}[t]
\vspace{-3mm}
\caption{Effect of $\rho$ on iSogCLR ($\tau_{\text{init}}$ is set to 0.3). We report the average top-1 accuracies (\%) for 400 epochs pretraining.}
\vspace{-0mm}
\label{tab:ablation_rho_full}
\vskip 0.2in
\begin{center}
\begin{small}
\begin{sc}
\renewcommand{\arraystretch}{0.8}
\begin{tabular}{p{2.2cm}p{1.0cm}p{1.0cm}p{1.0cm}p{1.0cm}}
\toprule
\diagbox[dir=NW]{Data}{$\rho$}  & 0.1 & 0.2 & 0.3 & 0.4 \\
\midrule
CIFAR10 & 88.98 & \textbf{89.03} & 88.99 & 88.75  \\
CIFAR100 & 63.02 & 63.12 & 63.27 & \textbf{63.82}  \\
ImageNet100 & 80.70 & \textbf{80.96} & 80.54 & 80.18  \\
CIFAR10-LT & 77.86 & 78.05 & 78.31 & \textbf{78.37}  \\
CIFAR100-LT & 52.60 & 52.75 & 52.92 & \textbf{53.04}  \\
iNaturalist & 92.13 & 92.30 & \textbf{92.79} & 92.66  \\
\bottomrule
\end{tabular}
\end{sc}
\end{small}
\end{center}
\vskip -0.0in
\end{table}

{\bf Effect of $\beta_0$.} Another hyper-parameter in iSogCLR is the moving average parameter $\beta_0$ for updating $\s^{t+1}$ in~(\ref{eq:update_s}). Following~\citet{yuan2022provable} (cf.~Table 8 in their paper), we tune this parameter in a range of $\{0.7,0.8,0.9\}$. We find that when $\beta_0$ of iSogCLR is set in this range, the performance of the algorithm does not differ much in most cases.

{\bf Comparing with other baselines containing individualized learnable parameters.} 

{\bf Unimodal TaU+SimCLR.} We first compare our method with TaU+SimCLR~\citep{zhang2021temperature}, which adopts the framework of SimCLR and optimizes an input-dependent temperature as the uncertainty for the input. Specifically, for an input $\x$,~\citet{zhang2021temperature} edit the encoder network to return $d+1$ entries, where the first $d$ entries are the embedding of $\x$, and the last entry (let $e$ denote its value) is used to compute a temperature for the input by $\frac{\text{sigmoid}(e)}{t}$ ($t$ is a fixed hyper-parameter). We implement TaU+SimCLR following the pseudo code in the paper~\citep{zhang2021temperature}, and present the results on CIFAR dataset in Table~\ref{tab:compare_with_TaU_simclr}. One can observe that iSogCLR outperforms TaU+SimCLR by large margins. TaU+SimCLR learns input-dependent $\tau$ to estimate the uncertainty in out-of-distribution detection effectively, but with the cost of sacrificing the performance on downstream tasks.

\begin{table*}[t]
\vspace{-4mm}
\caption{Comparison between TaU+SimCLR and iSogCLR. We report the top-1 accuracies (\%) after 400 epochs pretraining on CIFAR datasets.}
\vspace{-0mm}
\label{tab:compare_with_TaU_simclr}
\vskip 0.1in
\begin{center}
\begin{small}
\begin{sc}
\begin{tabular}{p{2.2cm}p{2cm}p{2cm}p{2cm}p{2.2cm}}
\toprule
Method & CIFAR10 & CIFAR100 & CIFAR10-LT & CIFAR100-LT  \\
\midrule
Tau+SimCLR & 86.80 & 59.35 & 76.41 & 49.62 \\
iSogCLR  & 89.24 & 63.82 & 78.37 & 53.06 \\
\bottomrule
\end{tabular}
\end{sc}
\end{small}
\end{center}
\vskip -0.0in
\end{table*}

{\bf Directly Optimizing CLIP with individualized temperatures.} Besides, we also try to implement a variant of CLIP with individualized learnable temperatures. Similar to the CLIP with a global learnable temperature, we construct a learnable temperature for each image or text, compute the loss on each pair using their own temperatures, and optimize them by the automatic differentiation in PyTorch. We initialize all temperature parameters to 0.01. However, we observe that this variant is hard to converge. Specifically, {\bf we observe that the average of learnable temperature parameters is getting larger and larger during training}. We believe the reason is this. Let us consider the ordinary bimodal contrastive loss on a image-text pair $(\x_i,\t_i)$:
\begin{equation*}
\ell(\x_i,\t_i)\!=\log\!\underset{\t\in\T^-_i}{\sum}\!\exp\!\left(\!\frac{h_{\x_i}(\t)}{\tau}\!\right)+\log\!\underset{\x\in\mathcal I_i^-}{\sum}\!\exp\!\left(\!\frac{h_{\t_i}(\x)}{\tau}\!\right),
\end{equation*}
where $h_{\x_i}(\t)=E_I(\x_i)^{\top}E_T(\t)-E_I(\x_i)^{\top}E_T(\t_i)$ and $h_{\t_i}(\x)=E_I(\x)^{\top}E_T(\t_i)-E_I(\x_i)^{\top}E_T(\t_i)$. If $\x_i$ are very similar to $\t_i$ (e.g., a pair with frequent semantics, or the encoders are good), then $h_{\x_i}(\t)$ and $h_{\t_i}(\x)$ are always negative. At this time, the larger the temperature, the smaller the loss function. Hence naively optimizing contrastive loss with individualized temperatures probably does not work.

\begin{figure*}[htb]
    \centering 
\begin{minipage}{0.25\textwidth}
  \includegraphics[width=\linewidth]{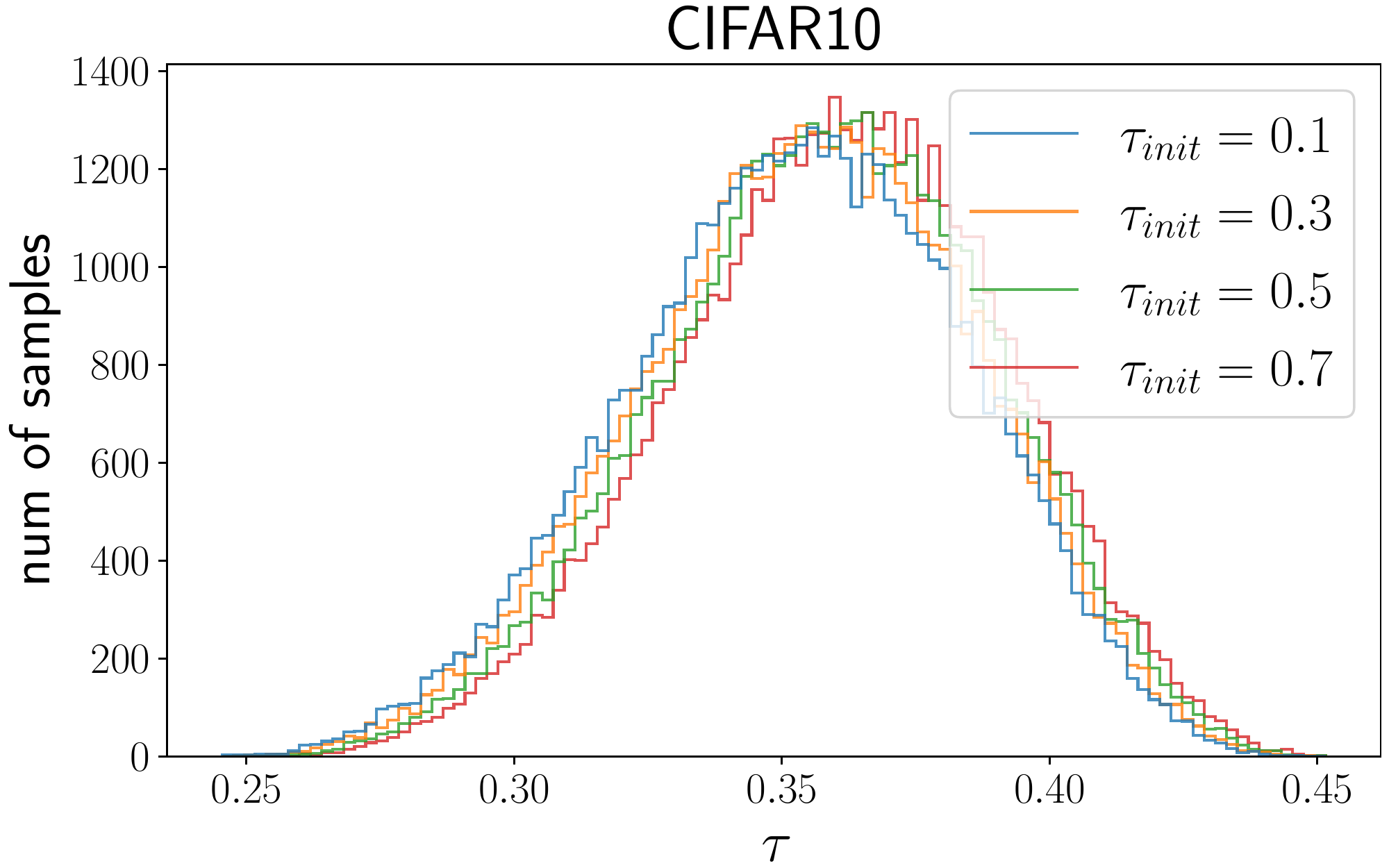}
\end{minipage}\hfil 
\begin{minipage}{0.25\textwidth}
  \includegraphics[width=1.0\linewidth]{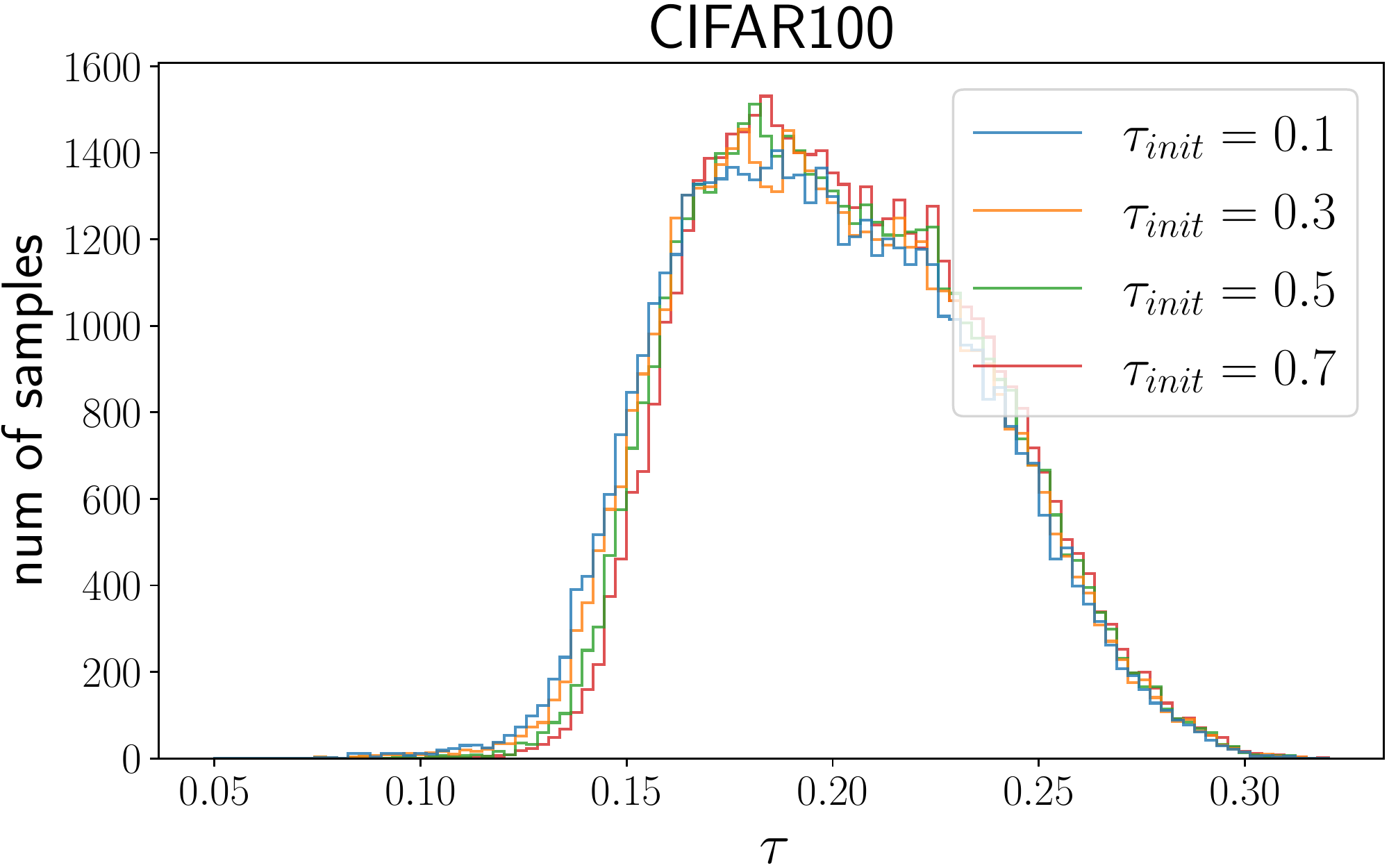}
\end{minipage}\hfil 
\begin{minipage}{0.25\textwidth}
  \includegraphics[width=1.0\linewidth]{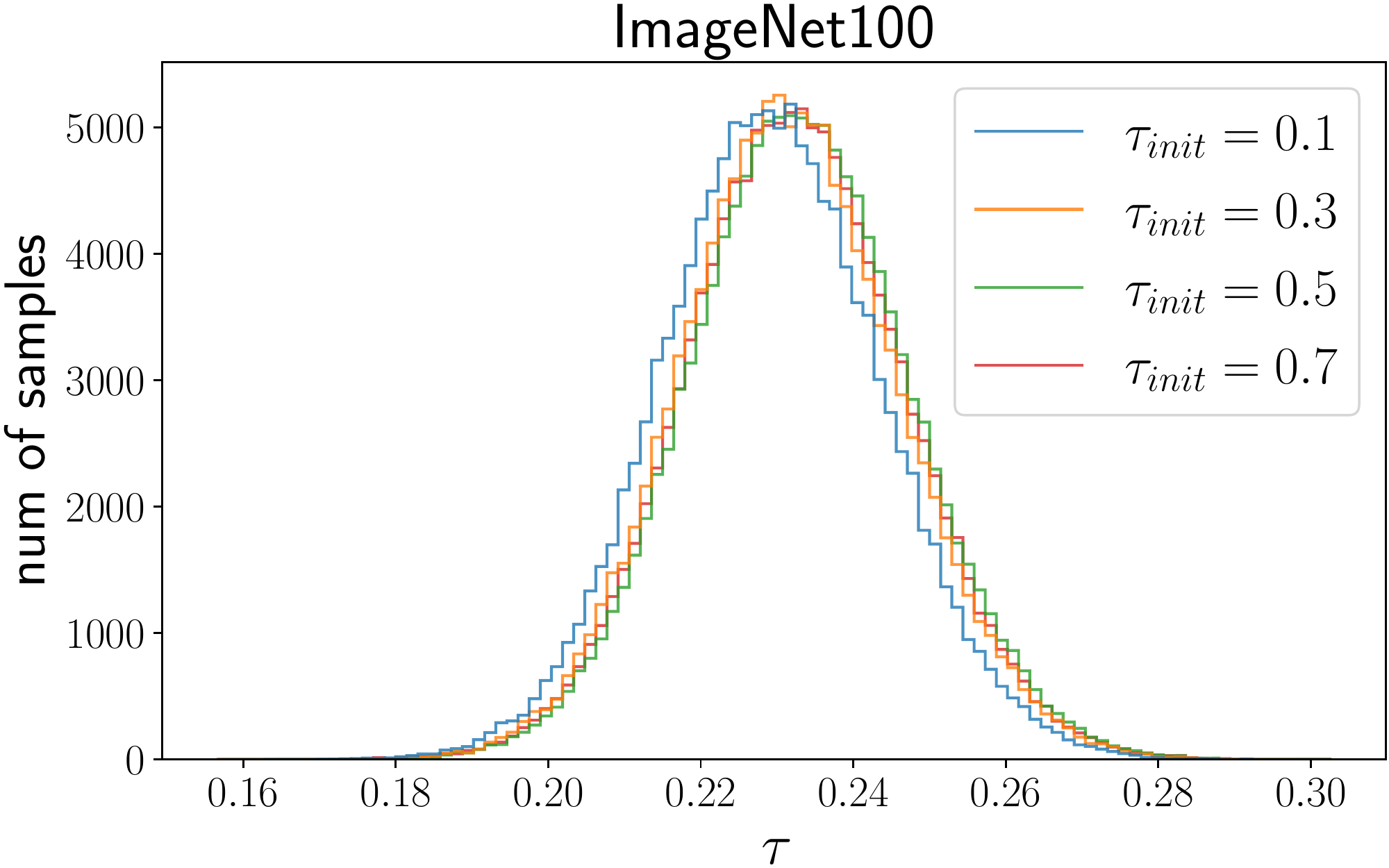}
\end{minipage}
\medskip
\begin{minipage}{0.25\textwidth}
  \includegraphics[width=1.0\linewidth]{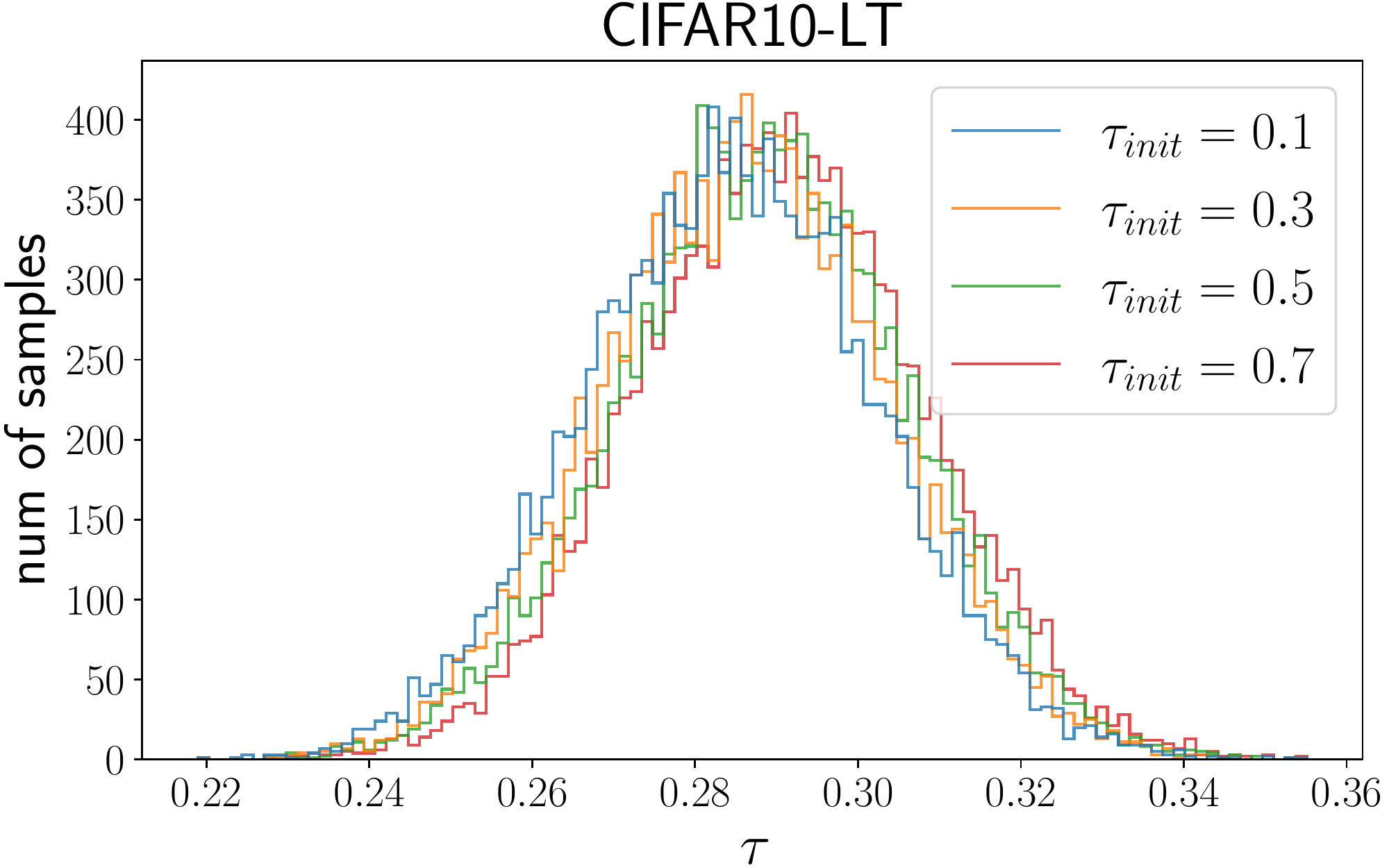}
\end{minipage}\hfil 
\begin{minipage}{0.25\textwidth}
  \includegraphics[width=1.0\linewidth]{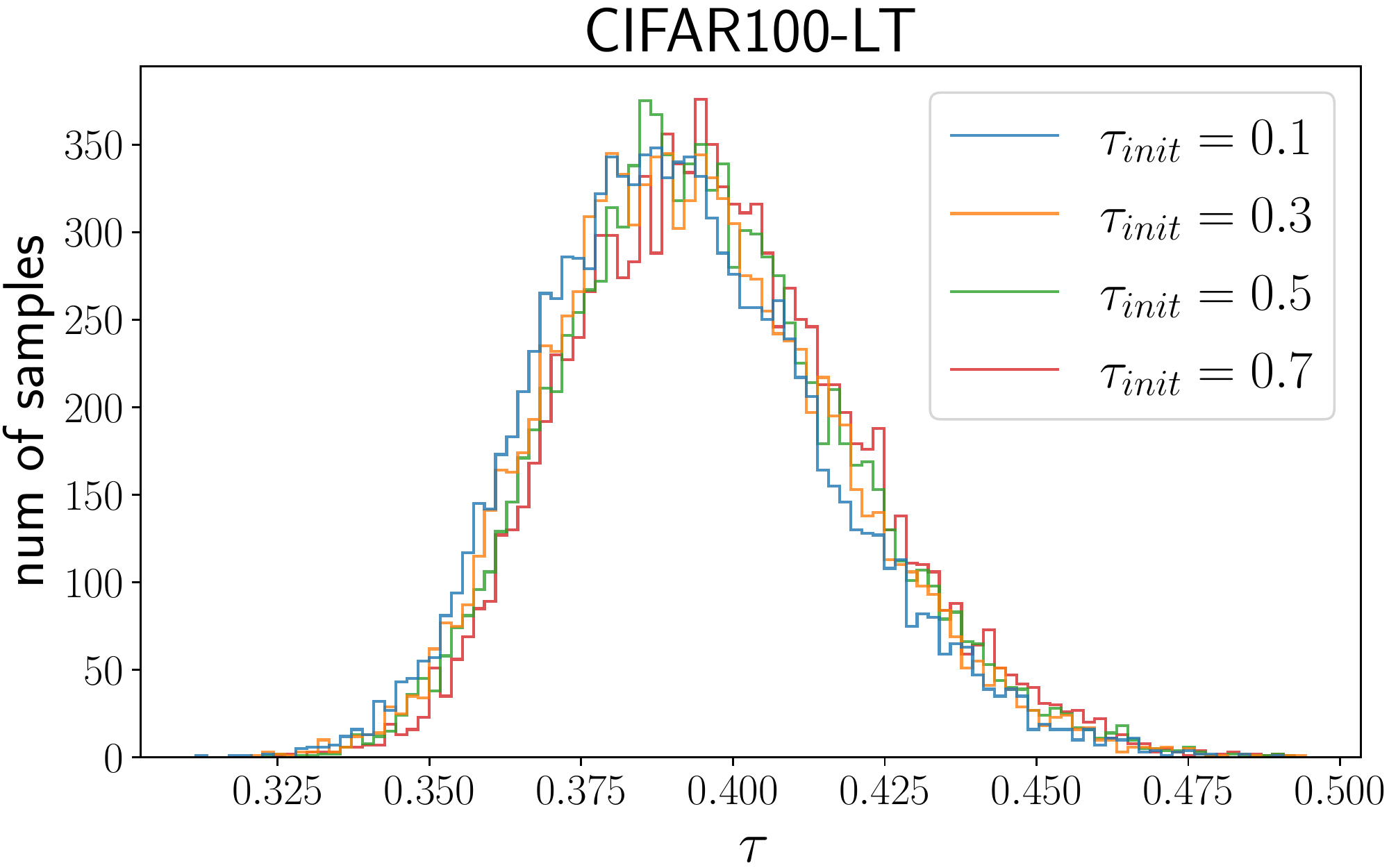}
\end{minipage}\hfil 
\begin{minipage}{0.25\textwidth}
  \includegraphics[width=\linewidth]{icml_appendix_final_tau_dists/iNat.pdf}
\end{minipage}
\medskip
\begin{minipage}{0.25\textwidth}
  \includegraphics[width=1.0\linewidth]{icml_appendix_final_tau_dists/cc3m_images_remove_line.pdf}
\end{minipage}\hfil 
\begin{minipage}{0.25\textwidth}
  \includegraphics[width=1.0\linewidth]{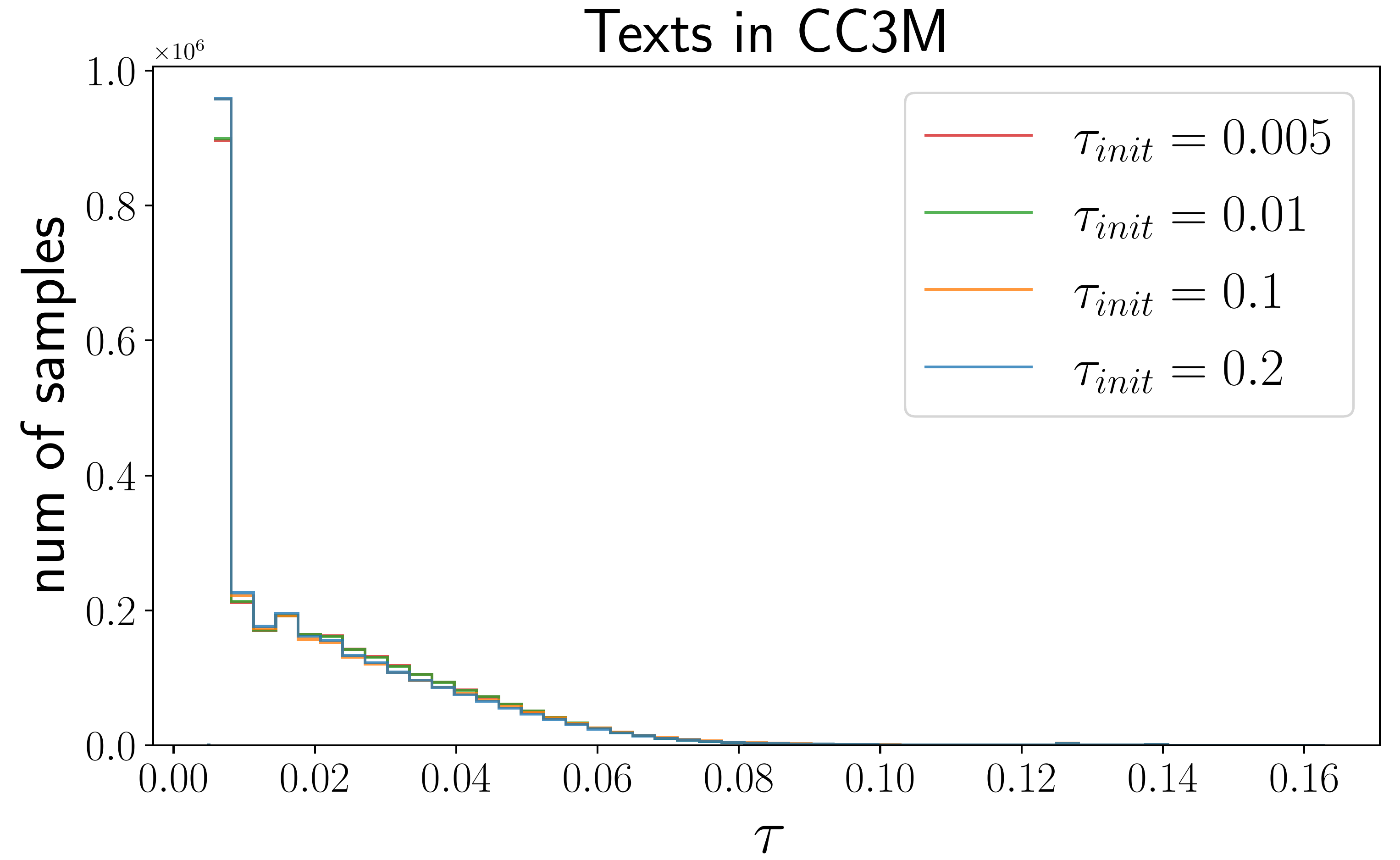}
\end{minipage}
\caption{Distributions of the final learned temperatures with different $\tau_{\text{init}}$ values on seven different datasets.}
\label{fig:app_tau_dists}
\end{figure*}

{\bf More results of the distributions of learned temperatures.} We present the final distributions of the learned temperatures with different $\tau_{\text{init}}$ values on all datasets in Figure~\ref{fig:app_tau_dists}. One can observe that the distributions for unimodal datasets are close to the Gaussian distribution. For CC3M dataset, we plot the distributions of learned temperatures of images and texts, respectively. We observe that these two distributions are very similar, and are close to the long-tail distribution with most samples have small temperatures.

{\bf More examples from CC3M dataset.} We present more images and texts with large and small learned temperatures in Figure~\ref{fig:app_large_tau_images_cc3m} and~\ref{fig:app_small_tau_images_cc3m}, respectively. One can observe that the images with large temperatures contain frequent semantics like person, house, animals, flowers, and natural landscape. While for images with small temperatures, their semantics could be abstract or rare in daily life.

\begin{figure*}
\centering
\includegraphics[width=0.8\linewidth]{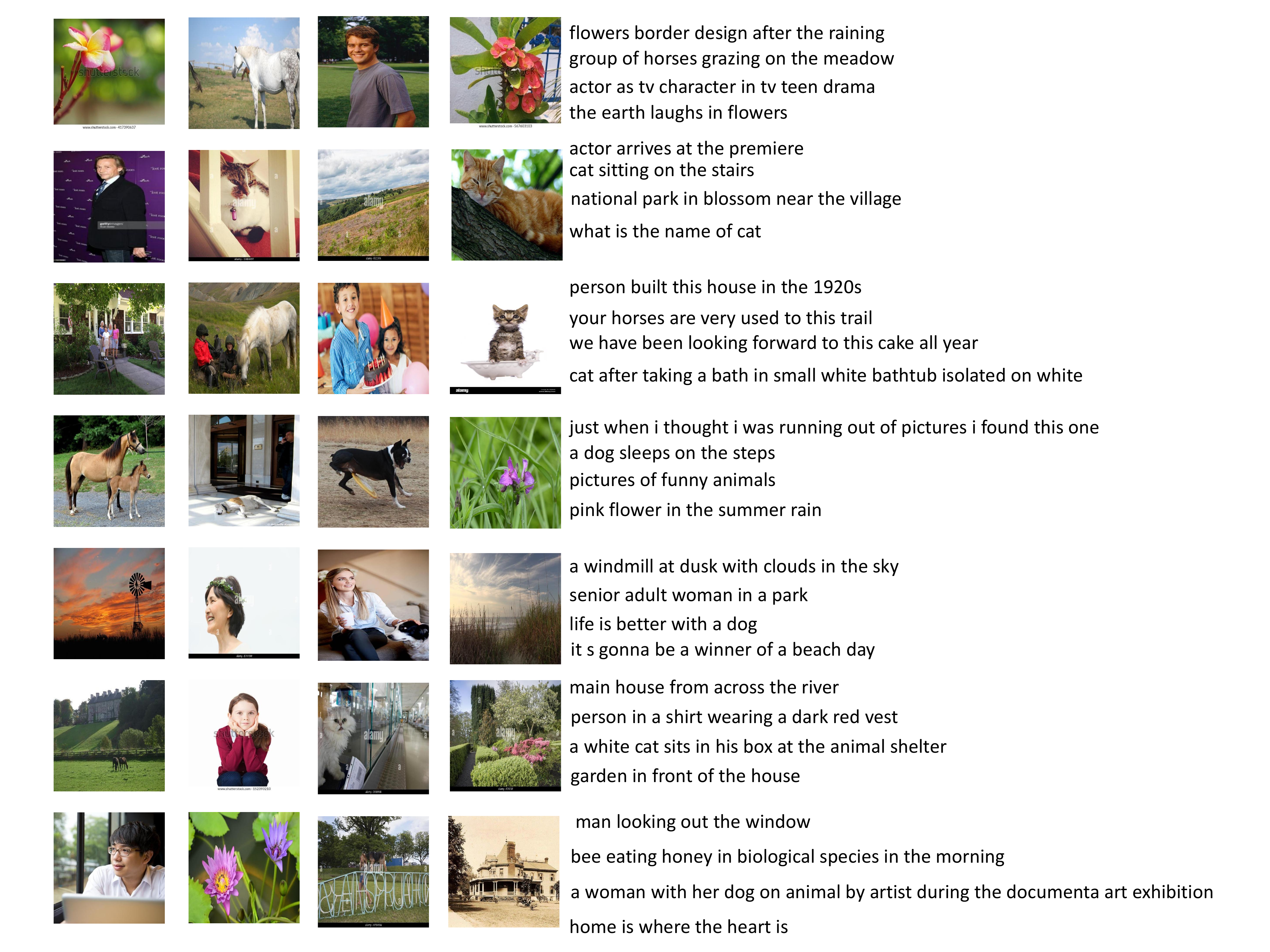}
\vspace{-5mm}
\caption{The images with large learned temperatures and their texts form CC3M. 
In general, they are very common in daily life, e.g., people, dogs, cats, flowers, houses, natural landscape, etc.}
\label{fig:app_large_tau_images_cc3m}
\end{figure*}

\begin{figure*}
\centering
\includegraphics[width=0.8\linewidth]
{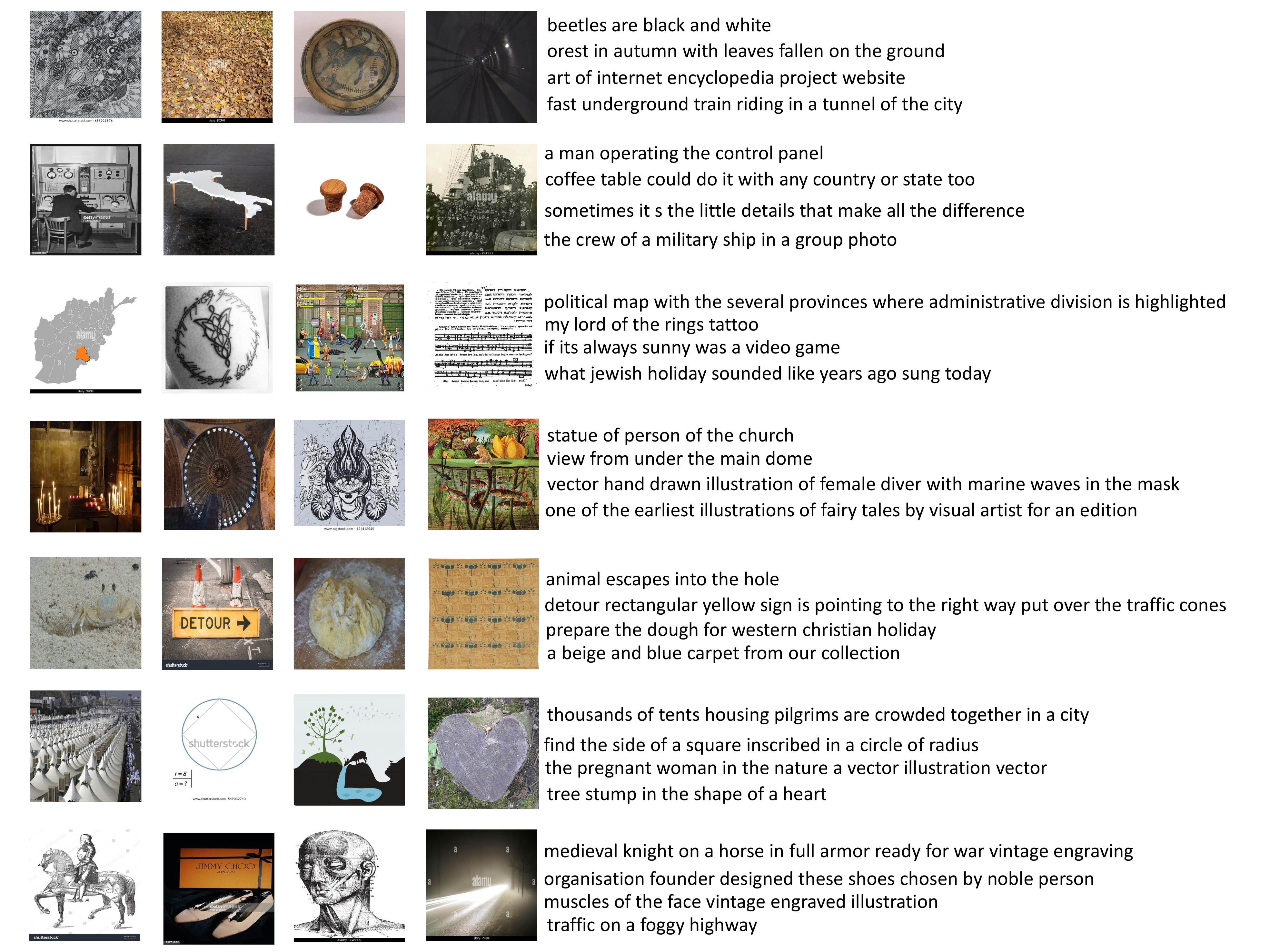}
\vspace{-5mm}
\caption{The images with small learned temperatures and their texts form CC3M. Most of them are not common in our lives or contain abstract concepts.}
\label{fig:app_small_tau_images_cc3m}
\end{figure*}

\section{Convergence Analysis}
\label{sec:convergence_analysis}

We first introduce some notations. Let $||\cdot||$ denote the Euclidean norm of a vector. We denote the combination of $\w$ and $\bftau$, i.e., $(\w^\top,\bftau^\top)^\top\in\R^{d+n}$ by $\z$. Recall that $h_i(\e
)\!=\!E(\A(\x_i))^{\top}\!E(\e)\!-\!E(\A(\x_i))^{\top}\!E(\A^{\prime}(\x_i))$, where {\bf we employ a new variable $\e$ in place of $\z$ used in~(\ref{eq:aux_func_uni_cl_loss}) to avoid conflicts}.

To simplify the notations, we use $g_i(\z)$ and $g_i(\z,\B)$ to represent $g_i(\w,\bftau_i;\S_i^-)$ and $g_i(\w,\bftau_i;\B_i)$, respectively. We can see that $\E_{\B}[g_i(\z,\B)]=g_i(\z)$. Then the objective~(\ref{eq:objective}) can be expressed as $F(\z)=F(\w,\bftau)=\frac{1}{n}\sum_{\x_i\in\D} f_{i}(\bftau_i,g_i(\z))$. We denote the batch sizes $B=|\B|$ and $B'=|\B_i|$.

Then we make the following standard assumptions regarding to problem~(\ref{eq:objective}).
\begin{ass}\label{ass:1}
\,There exists $R,\sigma,C_g,C_f,L_f,L_g,C$ such that
\begin{enumerate}[(i)]
\item The domain of model parameter $\w\in\mathcal{W}$ is bounded by $R$, i.e., for all $\w\in\mathcal{W}$, we have $||\w||\leq R$.
\item $\E_{\B}[||g_i(\z)-g_i(\z,\B)||^2]\leq\frac{\sigma^2}{B}$ and $\E_{\B}[||\nabla g_i(\z)-\nabla g_i(\z,\B)||^2]\leq\frac{\sigma^2}{B}$.
\item Functions $g_i$ and $f_i$ satisfy $||\nabla g_i||\leq C_g$ and $||\nabla f_i||\leq C_f$ for all $i$.
\item Functions $\nabla f_i(\cdot)$, $\nabla g_i(\cdot)$ are $L_f$,$L_g$-Lipschitz continuous for all $i$.
\item Functions $h_i(\e)$ is bounded by $C$ for all $i$, i.e., $|h_i(\e)|\leq C$.
\end{enumerate}
\end{ass}

{\noindent\bf Remark:} Assumption~\ref{ass:1}$(i)$ is also assumed by~\citet{levy2020large} and~\citet{qi2022stochastic}, and is mainly used for convex analysis. Assumption~\ref{ass:1}$(ii)$ assumes that the stochastic estimators of $g_i(\z)$ and $\nabla g_i(\z)$ have bounded variance. Assumption~\ref{ass:1}$(iii)$ and $(iv)$ are also standard for convergence analysis. Note that $E\left(\A\left(\x_i\right)\right)$, $E\left(\A^{\prime}\left(\x_i\right)\right)$ and $E\left(\e\right)$ are all \emph{normalized} vectors, thus their inner products are bounded and Assumption~\ref{ass:1}$(v)$ holds.

However, $F(\w,\bftau)$ is not necessarily smooth in terms of $\z=(\w^\top,\bftau^\top)^\top$ if $\bftau$ is unbounded. To address this concern, we have the following lemma:

\begin{lemma}
    The optimal solution of $\bftau_i^{*},i=1,2,\ldots,n$ to problem~(\ref{eq:objective}) is upper bounded by $\tilde{\tau}=\tau_0+C/\rho$, where $C$ is the upper bound for functions $h_i(\e)$ and $\rho$ is the constraint parameter.
\label{lemma:tau_upper_bound}
\end{lemma}

\begin{proof}
Recall the primal problem for each $\x_i\in\D$:
\begin{equation*}
    \p^{*} = \max_{\{\p\in\Delta,\text{KL}(\p, \boldsymbol{1}/m) \leq \rho\}} \sum_{\e_j \in \S_i^-} \p_j h_i(\e_j)-\tau_0 \text{KL}(\p, \boldsymbol{1}/m),
\end{equation*}
where $\p^{*}$ is the optimal value of the above problem.

Invoking dual variable $\bar{\lambda}_i$, we obtain the dual problem
\begin{equation*}
    \q^{*} = \min_{\bar{\lambda}\geq0} \max_{\p\in\Delta} \sum_{\e_j\in\S_i^-} \p_j h_i(\e_j)-\tau_0 \text{KL}(\p, \boldsymbol{1}/m)-\bar{\lambda}_i \left( \text{KL}(\p, \boldsymbol{1}/m) - \rho\right).
\end{equation*}

Set $\bar{\p}=(1/m,\ldots,1/m)$, a Slater vector satisfying $ \text{KL}(\bar{\p}, \boldsymbol{1}/m) - \rho \leq 0$. Applying Lemma 3 in~\citep{nedic2009subgradient}, we have
\begin{equation*}
    |\bar{\lambda}_i^*|\leq\frac{1}{\rho}\left(\q^{*}- \sum_{\e_j \in \S_i^-} \bar{\p}_j h_i(\e_j)-\tau_0 \text{KL}(\bar{\p}, \boldsymbol{1}/m)\right).
\end{equation*}

Since the primal problem is concave in terms of $\p$, we have $\p^{*}=\q^{*}$. Therefore,
\begin{equation}
\begin{aligned}
    |\bar{\lambda}_i^*|&\leq\frac{1}{\rho}\left(\p^{*}- \sum_{\e_j \in\S_i^-} \bar{\p}_j h_i(\e_j)\right) \\
    &\leq \frac{1}{\rho}\left(\sum_{\e_j \in\S_i^-} \p_j^{*} h_i(\e_j)-\tau_0 D(\p^{*}, \boldsymbol{1}/m)  - \sum_{\e_j \in\S_i^-} \bar{\p}_j h_i(\e_j)\right)\\
    &\leq \frac{C}{\rho},
\label{lemma:tau_upper_bound_1}
\end{aligned}
\end{equation}
where the last inequality is because $|h_i(\e_j)|\leq C$. Let $\bftau_i=\bar{\lambda}_i+\tau_0$, we have
\begin{equation*}
\q^{*} = \min_{\tau\geq\tau_0} \max_{\p\in\Delta} \sum_{\e_j \in\S_i^-} \p_j h_i(\e_j)-\tau \left( \text{KL}(\p, \boldsymbol{1}/m) - \rho\right) -\tau_0\rho. 
\end{equation*}
By~(\ref{lemma:tau_upper_bound_1}), we know that the optimal solution for above problem $|\bftau_i^{*}|\leq |\bar{\lambda}_i^{*}|+\tau_0\leq  \frac{C}{\rho}+\tau_0$, which completes the proof.
    
\end{proof}

Due the boundness of functions $h_i(\e)$ (cf.~Assumption~\ref{ass:1}$(v)$) and $\bftau_i$ (cf.~Lemma~\ref{lemma:tau_upper_bound}), we have the following lemma:

\begin{lemma}
    Functions $g_i(\z_t)$ and $g_i(\z_t,\B)$ are lower bounded by $\hat{g}=\exp(-C/\tilde{\tau})$, where $-C$ is the lower bound for functions $h_i(\e)$ and $\tilde{\tau}$ is the upper bound for $\bftau_i^{*}$.
\end{lemma}

\begin{proof}
    Recall the definitions of $h_i(\e)$, $g_i(\z_t)$ and $g_i(\z_t,\B)$:
\begin{equation*}
\begin{aligned}
    h_i(\e) &=\!E(\A(\x_i))^{\top}\!E(\e)\!-\!E(\A(\x_i))^{\top}\!E(\A^{\prime}(\x_i)), \\
    g_i(\z) &=g_i(\w,\bftau_i;\S_i^-)=\frac{1}{|\S_i^-|}\sum_{\e\in\S_i^-}\exp\left(\frac{h_i(\e)}{\bftau_i}\right),\\
    g_i(\z,\B)&=g_i(\w,\bftau_i;\B_i)=\frac{1}{\B_i}\sum_{\e\in\B_i}\exp\left(\frac{h_i(\e)}{\bftau_i}\right).
\end{aligned}
\end{equation*}

Using $\bftau_i\leq\tilde{\tau}$ and $h_i(\e)\geq -C$, we have $g_i(\z)\geq \exp\left(\frac{-C}{\tilde{\tau}}\right)$. Similarly, we have $g_i(\z,\B)\geq \exp\left(\frac{-C}{\tilde{\tau}}\right)$, which completes the proof.
    
\end{proof}

We will also see that the constraint on the domain of $\bftau$ guarantees the smoothness of $F(\w,\bftau)$, which is critical for the proposed algorithm to enjoy fast convergence rate.

\begin{lemma}\label{lem:L_F_smooth}
For all $\w\in\mathcal{W}$, $\bftau_i\in[\tau_0,\tilde{\tau}]$, and $i = 1,2,\dots,n$, $F_i(\z)=F_i(\w,\bftau_i)=f_i(\bftau_i,g_i(\z))$ is $L_F$-smooth for some constant $L_F$.
\end{lemma}

Note that Lemma~\ref{lem:L_F_smooth} naturally follows that function $F(\z)$ is also $L_F$-smooth.

\begin{proof}
    We have gradients
    \begin{equation*}
    \begin{aligned}
        \nabla_{\w} F_i(\w,\bftau_i)&=\nabla_{\w} g_i(\w,\bftau_i) \nabla_{g_i}f_i(\bftau_i,g_i(\w,\bftau_i))\\
        &= \frac{\bftau_i}{g_i(\w,\bftau_i)}\nabla_{\w} g_i(\w,\bftau_i)\\
        \nabla_{\bftau} F_i(\w,\bftau_i) &= \nabla_{\bftau} g_i(\w,\bftau_i) \nabla_{g_i}f_i(\bftau_i,g_i(\w,\bftau_i)) + \nabla_{\bftau} f_i(\bftau_i,g_i(\w,\bftau_i))\\
        &= \frac{\bftau_i}{g_i(\w,\bftau_i)}\nabla_{\bftau} g_i(\w,\bftau_i)+\nabla_{\bftau} f_i(\bftau_i,g_i(\w,\bftau_i))\\
        & = \frac{\bftau_i}{g_i(\w,\bftau_i)}\begin{pmatrix}
        0\\
        \vdots\\
        \nabla_{\bftau_i} g_i(\w,\bftau_i)\\
        \vdots\\
        0\\
        \end{pmatrix}
        +\begin{pmatrix}
        0\\
        \vdots\\
        \log(g_i(\w,\bftau_i))+\rho\\
        \vdots\\
        0\\
        \end{pmatrix}
    \end{aligned}
    \end{equation*}

    For any arbitrary $\z,\tilde{\z}$, we have
    \begin{equation*}
        \begin{aligned}
            &\|\nabla_{\z}F_i(\z)-\nabla_{\z}F_i(\tilde{\z})\|^2\\
            &= \|\nabla_{\w}F_i(\z)-\nabla_{\w}F_i(\tilde{\z})\|^2+\|\nabla_{\bftau}F_i(\z)-\nabla_{\bftau}F_i(\tilde{\z})\|^2\\
            &=\left\Vert\frac{\bftau_i}{g_i(\w,\bftau_i)}\nabla_{\w} g_i(\w,\bftau_i)-\frac{\tilde{\bftau}_i}{g_i(\tilde{\w},\tilde{\bftau}_i)}\nabla_{\w} g_i(\tilde{\w},\tilde{\bftau}_i)\right\Vert^2\\
            &\quad +\left\Vert\frac{\bftau_i}{g_i(\w,\bftau_i)}\nabla_{\bftau_i} g_i(\w,\bftau_i)+\log(g_i(\w,\bftau_i))-\left(\frac{\tilde{\bftau}_i}{g_i(\tilde{\w},\tilde{\bftau}_i)}\nabla_{\bftau_i} g_i(\tilde{\w},\tilde{\bftau}_i)+\log(g_i(\tilde{\w},\tilde{\bftau}_i))\right)\right\Vert^2
        \end{aligned}
    \end{equation*}
    Under assumption~\ref{ass:1}, we obtain
    \begin{equation*}
        \begin{aligned}
            &\left\Vert\frac{\bftau_i}{g_i(\w,\bftau_i)}\nabla_{\w} g_i(\w,\bftau_i)-\frac{\tilde{\bftau}_i}{g_i(\tilde{\w},\tilde{\bftau}_i)}\nabla_{\w} g_i(\tilde{\w},\tilde{\bftau}_i)\right\Vert^2\\
            &\leq 2\left\Vert\frac{\bftau_i}{g_i(\w,\bftau_i)}\big[\nabla_{\w} g_i(\w,\bftau_i)-\nabla_{\w} g_i(\tilde{\w},\tilde{\bftau}_i)\big]\right\Vert^2+2\left\Vert\left[\frac{\bftau_i}{g_i(\w,\bftau_i)}-\frac{\tilde{\bftau}_i}{g_i(\tilde{\w},\tilde{\bftau}_i)}\right]\nabla_{\w} g_i(\tilde{\w},\tilde{\bftau}_i)\right\Vert^2\\
            &\leq \frac{2\tilde{\tau}L_g}{\hat{g}}(\|\w-\tilde{\w}\|^2+\|\bftau_i-\tilde{\bftau}_i\|^2)+\frac{2\tilde{\tau}C_g^2}{\hat{g}^2}(\|\w-\tilde{\w}\|^2+\|\bftau_i-\tilde{\bftau}_i\|^2)
        \end{aligned}
    \end{equation*}
    and
    \begin{equation*}
    \begin{aligned}
        &\left\Vert\frac{\bftau_i}{g_i(\w,\bftau_i)}\nabla_{\bftau_i} g_i(\w,\bftau_i)+\log(g_i(\w,\bftau_i))-\left(\frac{\tilde{\bftau}_i}{g_i(\tilde{\w},\tilde{\bftau}_i)}\nabla_{\bftau_i} g_i(\tilde{\w},\tilde{\bftau}_i)+\log(g_i(\tilde{\w},\tilde{\bftau}_i))\right)\right\Vert^2\\
        &\leq 4\left(\frac{C_g}{\hat{g}} + \frac{\tilde{\tau}C_g^2}{\hat{g}^2} + \frac{\tilde{\tau}L_g}{\hat{g}} + \frac{C_g}{\hat{g}}\right)(\|\w-\tilde{\w}\|^2+\|\bftau_i-\tilde{\bftau}_i\|^2)
    \end{aligned}
    \end{equation*}

Define $L_F =  \frac{2\tilde{\tau}L_g}{\hat{g}} +\frac{2\tilde{\tau}C_g^2}{\hat{g}^2}+4\left(\frac{C_g}{\hat{g}} + \frac{\tilde{\tau}C_g^2}{\hat{g}^2} + \frac{\tilde{\tau}L_g}{\hat{g}} + \frac{C_g}{\hat{g}}\right)$, then $\|\nabla_{\z}F_i(\z)-\nabla_{\z}F_i(\tilde{\z})\|^2\leq L_f \|\z-\tilde{\z}\|^2$.
\end{proof}

Below, we let $\chi=\{\z|\w\in\W,\tau_0\leq\bftau_i\leq\tilde{\tau},i=1,2,\ldots,n\}$. $\delta_{\chi}(\z)=0$ if $\z\in\chi$, and $\delta_{\chi}(\z)=\infty$ if $\z\notin\chi$. Then problem~(\ref{eq:objective}) is equivalent to:
\begin{equation}
\min_{\z\in\R^{d+n}} \bar{F}(\z):=F(\z)+\delta_\chi(\z).
\label{eq:objective_equ}
\end{equation}

Now the update step of $\z_t$ can be written as $\z_{t+1}=\Pi_{\chi}(\z_t-\eta\d_{t+1})$, where $\Pi_{\chi}$ denotes the Euclidean projection onto the domain $\chi$, and $\d_{t+1}=(\v_{t+1}^\top,{\u^{t+1}}^\top)^\top$.

Since $\bar{F}$ is non-smooth, we define the regular subgradients as follows.

\begin{definition}[Regular Subgradient]
Consider a function $\Phi:\R^n\rightarrow \bar{\R}$ and $\Phi(\bar{\x})$ is finite. For a vector $\v\in\R^n$, $\v$ is a regular subgradient of $\Phi$ at $\bar{\x}$, written $\v\in\hat{\partial}\Phi(\bar{\x})$, if
\begin{equation*}
\liminf _{\mathbf{x} \rightarrow \overline{\mathbf{x}}} \frac{\Phi(\mathbf{x})-\Phi(\overline{\mathbf{x}})-\mathbf{v}^{\top}(\mathbf{x}-\overline{\mathbf{x}})}{\|\mathbf{x}-\overline{\mathbf{x}}\|} \geq 0.
\end{equation*}

\end{definition}

Since $F(\z)$ is differentiable, we use $\hat{\partial}\bar{F}(\z)=\nabla F(\z)+\hat{\partial}\delta_{\chi}(\z)$ (see Exercise 8.8 in~\citet{rockafellar2009variational}) in the analysis. The $\text{dist}(0,\hat{\partial}\bar{F}(\z))$ measures the distance between the origin and the regular subgradient set of $\bar{F}$ at $\z$. The oracle complexity is defined below:

\begin{definition}[Oracle Complexity]
    Let $\epsilon>0$ be a small constant, the oracle complexity is defined as the number of processing samples in order to achieve $\E[\text{dist}(0,\hat{\partial}\bar{F}(\z))]\leq\epsilon$ for a non-convex loss function or $\E[F(\z)-F(\z_{*})]\leq\epsilon$ for a convex loss function.
\end{definition}

To prove the main theorem, we present some required lemmas.

\begin{lemma}
    Under Assumption~\ref{ass:1}, run Algorithm~\ref{algo:sogclr_dro} with $\eta L_F\leq\frac{1}{4}$, and the output $\z_R$ of Algorithm~\ref{algo:sogclr_dro} satisfies
\begin{equation*}
    \E[dist(0,\hat{\partial}\bar{F}(\z_R))]\leq\frac{2+ 40L_F\eta}{T}\sum_{t=1}^{T}||\d_{t+1}-\nabla F(\z_t)||^2 + \frac{2\Delta}{\eta T} + \frac{40 L_F\Delta}{T},
\end{equation*}
where $\Delta:=\bar{F}(\z_1)-\inf_{\z\in\chi}\bar{F}(\z)$.
\label{lemma:F_smooth}
\end{lemma}

\begin{proof}
Recall the update of $\z_{t+1}$ is
\begin{equation*}
\begin{aligned}
    \z_{t+1} &= \Pi_{\chi}(\z_t-\eta\d_{t+1}) \\
    &= \argmin_{\z\in\R^{d+n}} \left\{\delta_{\chi}(\z) + \langle\d_{t+1},\z-\z_t\rangle + \frac{1}{2\eta}||\z-\z_t||^2 \right\}.
\end{aligned} 
\end{equation*}

Then by Exercise 8.8 and Theorem 10.1 of~\citet{rockafellar2009variational}, we know
\begin{equation*}
    -\d_{t+1}-\frac{1}{\eta}(\z_{t+1}-\z_t)\in \hat{\partial}\delta_{\chi}(\z_{t+1}),
\end{equation*}
which implies that
\begin{equation}
    \nabla F(\z_{t+1})-\d_{t+1}-\frac{1}{\eta}(\z_{t+1}-\z_t)\in\nabla F(\z_{t+1}) + \hat{\partial}\delta_{\chi}(\z_{t+1})= \hat{\partial}\bar{F}(\z_{t+1}).
\label{eq:lemma_F_smooth_core}
\end{equation}

By the update of $\z_{t+1}$, we also have 
\begin{equation*}
    \delta_{\chi}(\z_{t+1}) + \langle \d_{t+1},\z_{t+1}-\z_t \rangle + \frac{1}{2\eta}||\z_{t+1}-\z_t||^2 \leq \delta_{\chi}(\z_t).
\end{equation*}

Since $F(\z)$ is $L_F$-smooth, we have
\begin{equation*}
F(\z_{t+1})\leq F(\z_t) + \langle \nabla F(\z_t),\z_{t+1}-\z_t \rangle + \frac{L_F}{2}||\z_{t+1}-\z_t||^2.
\end{equation*}

Combining the above two inequalities, we obtain
\begin{equation*}
    \langle \d_{t+1}-\nabla F(\z_t),\z_{t+1}-\z_t \rangle + \frac{1}{2}\left(\frac{1}{\eta}-L_F\right)||\z_{t+1}-\z_t||^2 \leq \bar{F}(\z_t) - \bar{F}(\z_{t+1}).
\end{equation*}

Thus we have
\begin{equation*}
     \frac{1}{2}\left(\frac{1}{\eta}-L_F\right)||\z_{t+1}-\z_t||^2 \leq \bar{F}(\z_t) - \bar{F}(\z_{t+1}) - \langle \d_{t+1}-\nabla F(\z_t),\z_{t+1}-\z_t \rangle,
\end{equation*}
where the last inequality uses $\langle \a,\b\rangle\leq ||\a||^2 + \frac{||\b||^2}{4}$. Then by rearranging the above inequality and summing it across $t=1,2,\ldots,T$, we have
\begin{equation}
\begin{aligned}
    \sum_{t=1}^T\frac{1-2\eta L_F}{4\eta}||\z_{t+1}-\z_t||^2 &\leq \bar{F}(\z_1)-\bar{F}(\z_{T+1}) + \sum_{t=1}^T\eta||\d_{t+1}-\nabla F(\z_t)||^2\\
    &\leq \bar{F}(\z_1)-\inf_{\z\in\chi} \bar{F}(\z) + \sum_{t=1}^T\eta||\d_{t+1}-\nabla F(\z_t)||^2 \\
    &= \Delta+ \sum_{t=1}^T\eta||\d_{t+1}-\nabla F(\z_t)||^2 
\label{ineq:lemma_F_smooth_a}
\end{aligned}
\end{equation}

Using the same method in the proof of Theorem 2 in~\citep{xu2019non}, we obtain the following relationship:
\begin{equation}
\begin{aligned}
    \sum_{t=1}^T||\d_{t+1}-\nabla F(\z_{t+1})+\frac{1}{\eta}(\z_{t+1}-\z_t)||^2 &\leq 2\sum_{t=1}^T||\d_{t+1}-\nabla F(\z_t)||^2 + \frac{2\Delta}{\eta} \\
    &+ \left(2 L_F^2 +\frac{3L_F}{\eta}\right)\sum_{t=1}^{T}||\z_{t+1}-\z_t||^2
\label{ineq:lemma_F_smooth_b}
\end{aligned}
\end{equation}

Recalling $\eta L_F\leq \frac{1}{4}$ and combining~(\ref{ineq:lemma_F_smooth_a}) and~(\ref{ineq:lemma_F_smooth_b}), we have
\begin{equation}
\begin{aligned}
&\sum_{t=1}^T||\d_{t+1}-\nabla F(\z_{t+1})+\frac{1}{\eta}(\z_{t+1}-\z_t)||^2\\
&\stackrel{(a)}{\leq} 2\sum_{t=1}^T||\d_{t+1}-\nabla F(\z_t)||^2 + \frac{2\Delta}{\eta} + \frac{5L_F}{\eta}\left(\frac{4}{1-2\eta L_F}\right)\left(\eta\Delta + \sum_{t=1}^T \eta^2 ||\d_{t+1}-\nabla F(\z_t)||^2\right) \\
&\stackrel{(b)}{\leq} 2\sum_{t=1}^T||\d_{t+1}-\nabla F(\z_t)||^2 + \frac{2\Delta}{\eta} + 40L_F\Delta + 40\eta L_F\sum_{t=1}^T||\d_{t+1}-\nabla F(\z_t)||^2,
\label{ineq:lemma_F_smooth_c}
\end{aligned}
\end{equation}
where (a) is due to $(2 L_F^2 +\frac{3L_F}{\eta})\leq\frac{5L_F}{\eta}$ and (b) is due to $\frac{4}{1-2\eta L_F}\leq 8$.

Recalling~(\ref{eq:lemma_F_smooth_core}) and the output rule of Algorithm~\ref{algo:sogclr_dro}, we have
\begin{equation}
\E[\text{dist}(0,\hat{\partial}\bar{F}(\z_R))^2]\leq \frac{1}{T}\sum_{t=1}^T\E[||\d_{t+1}-\nabla F(\z_{t+1})+\frac{1}{\eta}(\z_{t+1}-\z_t)||^2].
\label{ineq:lemma_F_smooth_d}
\end{equation}

At last, we combine~(\ref{ineq:lemma_F_smooth_c}) and~(\ref{ineq:lemma_F_smooth_d}) and have
\begin{equation}
    \E[\text{dist}(0,\hat{\partial}\bar{F}(\z_R))^2]\leq  \frac{2+ 40\eta L_F}{T} \sum_{t=1}^T\E[||\d_{t+1}-\nabla F(\z_t)||^2] + \frac{2\Delta}{T\eta} + \frac{40L_F\Delta}{T}. 
\end{equation}

\end{proof}

\begin{lemma}
Under Assumption~\ref{ass:1}, run Algorithm~\ref{algo:sogclr_dro} and we have
\begin{equation*}
\begin{aligned}
\sum_{t=1}^T\E[||\d_{t+1}-\nabla F(\z_t)||^2]&\leq \Delta_{\v} + \Delta_{\u} + \left(\frac{4L_F^2}{\beta_1^2}+ \frac{72n^3L_F^2}{B^2\beta^2}\right)\sum_{t=1}^T\E[||\z_t-\z_{t-1}||^2] \\
&\quad + C_1\sum_{t=1}^T\E[||g(\z_t)-\s^{t+1}||^2] + \frac{C_2\beta_1}{B} T + \frac{C_3\beta}{B} T,
\end{aligned}
\end{equation*}
where $\Delta_{\v},\Delta_{\u},C_1,C_2,C_3$ are constants defined in the proof.

\label{lemma:gradient_error}
\end{lemma}

\begin{proof}
Recalling $\d_{t+1}=(\v_{t+1}^\top,{\u^{t+1}}^\top)^\top$ and $\nabla F(\z_t)=(\nabla_{\w}F(\z_t) ,\nabla_{\bftau}F(\z_t) )^\top$, we have
\begin{equation*}
   ||\d_{t+1}-\nabla F(\z_t)||^2 =||\v_{t+1}-\nabla_{\w}F(\z_t)||^2 + ||\u^{t+1}-\nabla_{\bftau}F(\z_t)||^2
\end{equation*}

We first establish the bound for $||\v_{t+1}-\nabla_{\w}F(\z_t)||^2$. Recall the define the following notations
\begin{equation*}
\begin{aligned}
&\nabla F(\z_t)=\frac{1}{n}\sum_{\x_i\in\S}\nabla_{\w}f_i(g_i(\z_t))\nabla_{\w}g_i(\z_t),\\
    &\nabla F(\z_t,\s^t)=\frac{1}{n}\sum_{\x_i\in\S}\nabla_{\w}f_i(\s_i^t)\nabla_{\w}g_i(\z_t),\\
    &\v_{t+1}=(1-\beta_1)\v_t + \beta_1 G(\w_t),\\
    &G(\w_t)=\frac{1}{B}\sum_{\x_i\in\B}\nabla_{\w}f_i(\s_i^t)\nabla_{\w}g_i(\z_t,\B).
\end{aligned}
\end{equation*}

By expansion, we have
\begin{equation}
\begin{aligned}
    &\E_t[||\nabla_{\w}F(\z_t)-\v_{t+1}||^2] \\
=&\E_t[||\nabla_{\w}F(\z_t)-(1-\beta_1)\v_t - \beta_1 G(\w_t)||^2] \\
=&\E_t[|| (1-\beta_1)(\nabla_{\w}F(\z_{t-1})-\v_t) + (1-\beta_1)(\nabla_{\w}F(\z_t)-\nabla_{\w}F(\z_{t-1})) \\
&\quad + \beta_1 (\nabla_{\w}F(\z_t) -\nabla_{\w}F(\z_t,\s^t)) + \beta_1 (\nabla_{\w}F(\z_t,\s^t)-G(\w_t)) ||^2]\\
\stackrel{(a)}{=}&|| (1-\beta_1)(\nabla_{\w}F(\z_{t-1})-\v_t) + (1-\beta_1)(\nabla_{\w}F(\z_t)-\nabla_{\w}F(\z_{t-1})) \\
&\quad + \beta_1 (\nabla_{\w}F(\z_t) -\nabla_{\w}F(\z_t,\s^t))||^2 + \beta_1^2 \E_t[||\nabla_{\w}F(\z_t,\s^t)-G(\w_t)||^2]\\
\stackrel{(b)}{\leq}& (1+\beta_1)(1-\beta_1)^2 ||\nabla_{\w}F(\z_{t-1}) -\v_t||^2 \\
&\quad+ 2\left(1+\frac{1}{\beta_1}\right)\left[ ||\nabla_{\w}F(\z_t)-\nabla_{\w}F(\z_{t-1})||^2 + \beta_1^2||\nabla_{\w}F(\z_t) -\nabla_{\w}F(\z_t,\s^t)||^2\right] \\
&\quad + \beta_1^2 \E_t[||\nabla_{\w}F(\z_t,\s^t)-G(\w_t)||^2] \\
\stackrel{(c)}{\leq}& (1-\beta_1) ||\nabla_{\w}F(\z_{t-1}) -\v_t||^2 + \frac{4 L_F^2}{\beta_1}||\z_t-\z_{t-1}||^2 + 4\beta_1||\nabla_{\w}F(\z_t)-\nabla_{\w}F(\z_t,\s^t)||^2\\
&\quad + \beta_1^2 \E_t[||\nabla_{\w}F(\z_t,\s^t)-G(\w_t)||^2],
\label{ineq:lemma_grad_error_1}
\end{aligned}
\end{equation}
where (a) is due to $\E_t[G(\w_t)]=\nabla_{\w} F(\z_t,\s^t)$, (b) is due to Young's inequality $||\a+\b||^2\leq(1+\gamma)||\a||^2+(1+\frac{1}{\gamma})||\b||^2$, and (c) is due to $\beta_1\leq 1 \rightarrow 1+\frac{1}{\beta_1}\leq\frac{2}{\beta_1}$.

Furthermore, one may bound $\E_t[||\nabla_{\w}F(\z_t)-\nabla_{\w}F(\z_t,\s^t)||^2]$ as follows:
\begin{equation}
\begin{aligned}
    &\E_t[||\nabla_{\w}F(\z_t)-\nabla_{\w}F(\z_t,\s^{t+1})||^2]\\
    =&\E_t\left[\left\Vert \frac{1}{n}\sum_{\x_i\in\D}\nabla_{\w}f_i(g_i(\z_t))\nabla_{\w}g_i(\z_t)- \frac{1}{n}\sum_{\x_i\in\D}\nabla_{\w}f_i(\s_i^t)\nabla_{\w}g_i(\z_t) \right\Vert^2\right] \\
    \leq&\frac{1}{n}\sum_{\x_i\in\D}C_g^2 L_f^2 \E_t[||g_i(\z_t)-\s_i^t||^2]\\
    =&\frac{C_g^2 L_f^2}{n} \E_t[||g(\z_t)-\s^t||^2].
\label{ineq:lemma_grad_error_1a}
\end{aligned}
\end{equation}

On the other hand, $\E_t[||\nabla_{\w}F(\z_t,\s^t)-G(\w_t)||^2]$ can be bounded by some constants:
\begin{equation}
\begin{aligned}
    &\E_t[||\nabla_{\w}F(\z_t,\s^t)-G(\w_t)||^2]\\
    =&\E_t\left[\left\Vert \frac{1}{n}\sum_{\x_i\in\D}\nabla_{\w}f_i(\s_i^t)\nabla_{\w}g_i(\z_t)- \frac{1}{B}\sum_{\x_i\in\B}\nabla_{\w}f_i(\s_i^t)\nabla_{\w}g_i(\z_t,\B)\right\Vert^2\right] \\
    \leq & \E_t \left[2\left\Vert \frac{1}{n}\sum_{\x_i\in\D}\nabla_{\w}f_i(\s_i^t)\nabla_{\w}g_i(\z_t) - \frac{1}{B}\sum_{\x_i\in\B}\nabla_{\w}f_i(\s_i^t)\nabla_{\w}g_i(\z_t) \right\Vert^2 \right. \\
    &\quad\left.2\left\Vert\frac{1}{B}\sum_{\x_i\in\B}\nabla_{\w}f_i(\s_i^t)\nabla_{\w}g_i(\z_t) - \frac{1}{B}\sum_{\x_i\in\B}\nabla_{\w}f_i(\s_i^t)\nabla_{\w}g_i(\z_t,\B) \right\Vert^2 \right]\\
    \leq & \frac{2 C_f^2 C_g^2}{B} + \frac{2 C_f^2 \sigma^2}{B'}.
\label{ineq:lemma_grad_error_1b}
\end{aligned}
\end{equation}

Substituting~(\ref{ineq:lemma_grad_error_1a}) and~(\ref{ineq:lemma_grad_error_1b}) into~(\ref{ineq:lemma_grad_error_1}), we have
\begin{equation}
\begin{aligned}
\E_t[||\nabla_{\w}F(\z_t)-\v_{t+1}||^2]&\leq (1-\beta_1) ||\nabla_{\w}F(\z_{t-1}) -\v_t||^2 + \frac{4 L_F^2}{\beta_1}\E_{t}[||\z_t-\z_{t-1}||^2] \\&+ \frac{4\beta_1 C_g^2 L_f^2}{n}\E_t[||g(\z_t)-\s^t||^2] + \frac{2\beta_1^2 C_f^2(C_g^2+\sigma^2)}{\min\{B,B'\}}.
\end{aligned}
\label{ineq:lemma_grad_error_1_res}
\end{equation}

Taking summation over $t=1,2,\ldots,T$, we obtain
\begin{equation}
\begin{aligned}
\sum_{t=1}^{T}\E[||\nabla_{\w}F(\z_{t})-\v_{t+1}||^2]&\leq \frac{1}{\beta_1}\Delta_{\v} + \frac{4L_F^2}{\beta_1^2}\sum_{t=1}^{T}\E[||\z_t-\z_{t-1}||^2] \\
&+\frac{4 C_g^2 L_f^2}{n}\sum_{t=1}^{T}\E[||g(\z_t)-\s^t||^2] + \frac{2\beta_1 C_f^2(C_g^2+\sigma^2)}{\min\{B,B'\}}T,
\end{aligned}
\label{ineq:lemma_grad_error_1_res_sum}
\end{equation}
where $\Delta_{\v}$ denotes $||\nabla_{\w}F(\z_0)-\v_{1}||^2$.

Next, we derive the bound for $||\u^{t+1}-\nabla_{\bftau}F(\z_t)||^2$. Note that
\begin{equation*}
    ||\u^{t+1}-\nabla_{\bftau}F(\z_t)||^2=\sum_{\x_i\in\D}||\u_i^{t+1}-\nabla_{\bftau_i}F(\z_t)||^2 = \sum_{\x_i\in\D}\left\Vert\u_i^{t+1}-\frac{1}{n}\nabla_{\bftau_i}F_i(\z_t)\right\Vert^2
\end{equation*}

Recall and define the following notations 
\begin{equation*}
\begin{aligned}
&\u_i^{t+1}=\begin{cases}(1-\beta)\u_i^t +\beta G(\bftau_i^t) \quad & \text{if }\x_i\in \B\\ \u_i^t & \text{o.w.}\end{cases},\quad \tilde{\u}_i^t:=(1-\beta)\u_i^t +\beta G(\bftau_i^t),\x_i\in\B, \\
&\quad\quad\quad\quad\quad\quad \nabla_{\bftau_i}F(\z_{t}) = \frac{1}{n} \left(\frac{\bftau_i^t}{g_i(\z_t)}\nabla_{\bftau_i}g_i(\z_t) + \log(g_i(\z_t)) + \rho\right), \\
&\quad\quad\quad\quad\quad\quad \nabla_{\bftau_i}F(\z_{t},\s_i^t) = \frac{1}{n} \left(\frac{\bftau_i^t}{\s_i^t}\nabla_{\bftau_i}g_i(\z_t) + \log(\s_i^t) + \rho\right), \\
&\quad\quad\quad\quad\quad\quad G(\bftau_i^t)=\frac{1}{n} \left(\frac{\bftau_i^t}{\s_i^t}\nabla_{\bftau_i}g_i(\z_t,\B) + \log(\s_i^t) + \rho\right).
\end{aligned}
\end{equation*}

Then we obtain
\begin{equation}
\begin{aligned}
&||\tilde{\u}_i^t - \nabla_{\bftau_i} F(\z_{t-1}) ||^2 \\
=&||(1-\beta)\u_i^t +\beta G(\bftau_i^t)-\nabla_{\bftau_i} F(\z_{t-1})||^2 \\
=&|| (1-\beta)(\u_i^t-\nabla_{\bftau_i} F(\z_{t-1})) +(1-\beta)(\nabla_{\bftau_i}F(\z_t) - \nabla_{\bftau_i}F(\z_{t-1})) \\
&\quad+ \beta (\nabla_{\bftau_i}F(\z_{t},\s_i^t)-\nabla_{\bftau_i}F(\z_{t})) + \beta (G(\bftau_i^t)-\nabla_{\bftau_i}F(\z_{t},\s_i^t))||^2 \\
\stackrel{(a)}{=}&|| (1-\beta)(\u_i^t-\nabla_{\bftau_i} F(\z_{t-1})) + (1-\beta)(\nabla_{\bftau_i}F(\z_t) - \nabla_{\bftau_i}F(\z_{t-1})) \\
&\quad+ \beta (\nabla_{\bftau_i}F(\z_{t},\s_i^t)-\nabla_{\bftau_i}F(\z_{t}))||^2  + \beta^2|| (\nabla_{\bftau_i}F(\z_{t},\s_i^t) - G(\bftau_i^t))||^2  \\
&\quad +2\langle (1-\beta)(\u_i^t-\nabla_{\bftau_i} F(\z_{t-1})) +(1-\beta)(\nabla_{\bftau_i}F(\z_t) - \nabla_{\bftau_i}F(\z_{t-1})) \\
&\quad+ \beta (\nabla_{\bftau_i}F(\z_{t},\s_i^t)-\nabla_{\bftau_i}F(\z_{t})) , \beta (G(\bftau_i^t)-\nabla_{\bftau_i}F(\z_{t},\s_i^t))\rangle\\
\stackrel{(b)}{\leq}&(1+\beta)(1-\beta)^2 ||\nabla_{\bftau_i} F(\z_{t-1})-\u_i^t||^2 \\
&\quad+2\left(1+\frac{1}{\beta}\right)\left[||\nabla_{\bftau_i}F(\z_t) - \nabla_{\bftau_i}F(\z_{t-1})||^2+ \beta^2||\nabla_{\bftau_i}F(\z_{t}) -\nabla_{\bftau_i}F(\z_{t},\s_i^t))||^2\right] \\
&\quad+ \beta^2|| (\nabla_{\bftau_i}F(\z_{t},\s_i^{t+1}) - G(\bftau_i^t))||^2  +2\langle (1-\beta)(\u_i^t-\nabla_{\bftau_i} F(\z_{t-1})) \\
&\quad+(1-\beta)(\nabla_{\bftau_i}F(\z_t) - \nabla_{\bftau_i}F(\z_{t-1})) + \beta (\nabla_{\bftau_i}F(\z_{t},\s_i^t)-\nabla_{\bftau_i}F(\z_{t})) , \beta (G(\bftau_i^t)-\nabla_{\bftau_i}F(\z_{t},\s_i^t))\rangle\\
\stackrel{(c)}{\leq}& (1-\beta)||\nabla_{\bftau_i} F(\z_{t-1})-\u_i^t||^2 + \frac{4 L_F^2}{n^2\beta}||\z_t-\z_{t-1}||^2 + 4\beta||\nabla_{\bftau_i}F(\z_{t}) -\nabla_{\bftau_i}F(\z_{t},\s_i^t))||^2 \\
&\quad+ \beta^2|| (\nabla_{\bftau_i}F(\z_{t},\s_i^t) - G(\bftau_i^t))||^2 +2\langle (1-\beta)(\u_i^t-\nabla_{\bftau_i} F(\z_{t-1})) \\
&\quad+(1-\beta)(\nabla_{\bftau_i}F(\z_t) - \nabla_{\bftau_i}F(\z_{t-1})) + \beta (\nabla_{\bftau_i}F(\z_{t},\s_i^t)-\nabla_{\bftau_i}F(\z_{t})) , \beta (G(\bftau_i^t)-\nabla_{\bftau_i}F(\z_{t},\s_i^t))\rangle,
\end{aligned}
\label{ineq:lemma_grad_error_2}
\end{equation}
where (b) is due to Young's inequality $||\a+\b||^2\leq(1+\gamma)||\a||^2+(1+\frac{1}{\gamma})||\b||^2$, and (c) is due to $\beta\leq 1 \rightarrow 1+\frac{1}{\beta}\leq\frac{2}{\beta}$. For simplicity, we denote the first term in the last inner product as $A_i^t$ and note that $A_i^t$ does not depend on the randomness of iteration $t$.

Subsequently, we derive the bound for $||\nabla_{\bftau_i}F(\z_{t}) -\nabla_{\bftau_i}F(\z_{t},\s_i^t))||^2$
\begin{equation}
\begin{aligned}
&||\nabla_{\bftau_i}F(\z_{t}) -\nabla_{\bftau_i}F(\z_{t},\s_i^{t+1}))||^2\\
=&\left\Vert\frac{1}{n} \left(\frac{\bftau_i^t}{g_i(\z_t)}\nabla_{\bftau_i}g_i(\z_t) + \log(g_i(\z_t))\right) -\frac{1}{n} \left( \frac{\bftau_i^t}{\s_i^t}\nabla_{\bftau_i}g_i(\z_t) + \log(\s_i^t)\right)\right\Vert^2 \\
\leq&\frac{2}{n^2}\left\Vert\frac{\bftau_i^t}{g_i(\z_t)}\nabla_{\bftau_i}g_i(\z_t)  - \frac{\bftau_i^t}{\s_i^t}\nabla_{\bftau_i}g_i(\z_t)\right\Vert^2 + \frac{2}{n^2}\left\Vert\ \log(g_i(\z_t)) - \log(\s_i^t)\right\Vert^2 \\
\leq& \frac{2(\tilde{\tau}^2 C_g^2 + \hat{g}^2)}{\hat{g}^4n^2}||\s_i^t-g_i(\z_t)||^2,
\label{ineq:lemma_grad_error_2a}
\end{aligned}
\end{equation}
where $\tilde{\tau}$ denotes the upper bound for $\bftau_i$ and $\hat{g}$ denotes the lower bound for $g_i$.


Substituting~(\ref{ineq:lemma_grad_error_2a})  into~(\ref{ineq:lemma_grad_error_2}), we have
\begin{equation}
\begin{aligned}
&||\tilde{\u}_i^t - \nabla_{\bftau_i} F(\z_{t-1}) ||^2\\
\leq&(1-\beta)||\nabla_{\bftau_i} F(\z_{t-1})-\u_i^t||^2 + \frac{4 L_F^2}{n^2\beta}||\z_t-\z_{t-1}||^2 + \frac{8\beta(\tilde{\tau}^2 C_g^2 + \hat{g}^2)}{\hat{g}^4n^2}||\s_i^t-g_i(\z_t)||^2 \\
&\quad+ \beta^2|| (\nabla_{\bftau_i}F(\z_{t},\s_i^t) - G(\bftau_i^t))||^2 +2\langle A_i^t,\beta (G(\bftau_i^t)-\nabla_{\bftau_i}F(\z_{t},\s_i^t)) \rangle
\end{aligned}
\end{equation}

\begin{equation*}
\begin{aligned}
&\E_t[||\u^{t+1}-\nabla_{\bftau}F(\z_{t-1})||^2]\\
&=\E_t\left[\sum_{x_i\in \B}\|\u_i^{t+1}-\nabla_{\bftau_i}F(\z_{t-1})\|^2+\sum_{x_i\not\in \B}\|\u_i^t-\nabla_{\bftau_i}F(\z_{t-1})\|^2\right]\\
&=\E_t\left[\sum_{x_i\in \B}\|\tilde{\u}_i^t-\nabla_{\bftau_i}F(\z_{t-1})\|^2\right]+\frac{n-B}{n}\|\u^t-\nabla_{\bftau}F(\z_{t-1})\|^2\\
&\leq \E_t\bigg[\sum_{x_i\in \B}(1-\beta)||\nabla_{\bftau_i} F(\z_{t-1})-\u_i^t||^2 + \frac{4 L_F^2}{n^2\beta}||\z_t-\z_{t-1}||^2 + \frac{8\beta(\tilde{\tau}^2 C_g^2 + \hat{g}^2)}{\hat{g}^4n^2}||\s_i^t-g_i(\z_t)||^2 \\
&\quad+ \beta^2|| (\nabla_{\bftau_i}F(\z_{t},\s_i^t) - G(\bftau_i^t))||^2 +2\langle A_i^t,\beta (G(\bftau_i^t)-\nabla_{\bftau_i}F(\z_{t},\s_i^t)) \rangle\bigg]+\frac{n-B}{n}\|\u^t-\nabla_{\bftau}F(\z_{t-1})\|^2\\
&\leq \frac{B}{n}(1-\beta)||\nabla_{\bftau} F(\z_{t-1})-\u^t||^2 + \frac{4 BL_F^2}{n^2\beta}||\z_t-\z_{t-1}||^2 + \frac{8B\beta(\tilde{\tau}^2 C_g^2 + \hat{g}^2)}{\hat{g}^4n^3}||\s^t-g(\z_t)||^2 \\
&\quad+ \beta^2\E_t\left[\sum_{x_i\in \B}|| (\nabla_{\bftau_i}F(\z_{t},\s_i^t) - G(\bftau_i^t))||^2\right] +2\E_t\left[\sum_{x_i\in \B}\langle A_i^t,\beta (G(\bftau_i^t)-\nabla_{\bftau_i}F(\z_{t},\s_i^t)) \rangle\right] \\
&\quad +\frac{n-B}{n}\|\u^t-\nabla_{\bftau}F(\z_{t-1})\|^2\\
&\stackrel{(a)}{\leq} (1-\frac{B\beta}{n})||\nabla_{\bftau} F(\z_{t-1})-\u^t||^2 + \frac{4 BL_F^2}{n^2\beta}||\z_t-\z_{t-1}||^2 + \frac{8B\beta(\tilde{\tau}^2 C_g^2 + \hat{g}^2)}{\hat{g}^4n^3}||\s^t-g(\z_t)||^2 \\
&\quad+ \frac{\beta^2B\tilde{\tau}^2 \sigma^2}{\hat{g}^2 B' n^2}  \\
\end{aligned}
\end{equation*}
where the last inequality uses the following facts
\begin{equation*}
    \begin{aligned}
        \E_t\left[\sum_{\x_i\in \B}|| \nabla_{\bftau_i}F(\z_{t},\s_i^t) - G(\bftau_i^t)||^2\right]
        &= \frac{1}{|\overline{\B}|}\sum_{\B\in \overline{\B}}\sum_{\x_i\in \B}|| \nabla_{\bftau_i}F(\z_{t},\s_i^t) - G(\bftau_i^t)||^2\\
        &=\frac{1}{|\overline{\B}|}\sum_{\x_i \in \D}|\overline{\B}_i|\frac{1}{|\overline{\B}_i|}\sum_{\B\in\overline{\B}_i}|| \nabla_{\bftau_i}F(\z_{t},\s_i^t) - G(\bftau_i^t)||^2\\
        &\leq \frac{1}{|\overline{\B}|}\sum_{\x_i \in \D}|\overline{\B}_i|\frac{1}{|\overline{\B}_i|}\sum_{\B\in\overline{\B}_i}\left\| \frac{1}{n}\frac{\bftau_i^t}{\s_i^t}\nabla_{\bftau_i}g_i(\z_t) -\frac{1}{n}\frac{\bftau_i^t}{\s_i^t}\nabla_{\bftau_i}g_i(\z_t,\B)\right\|^2\\
        &\leq \frac{1}{|\overline{\B}|}\sum_{\x_i \in \D}|\overline{\B}_i|\frac{\tilde{\tau}^2 \sigma^2}{\hat{g}^2 B' n^2} = \frac{B\tilde{\tau}^2 \sigma^2}{\hat{g}^2 B' n^2}
    \end{aligned}
\end{equation*}
and
\begin{equation*}
    \begin{aligned}
        &\E_t\left[\sum_{\x_i\in \B}\langle A_i^t,\beta (G(\bftau_i^t)-\nabla_{\bftau_i}F(\z_{t},\s_i^t)) \rangle\right]\\
        &=\frac{1}{|\overline{\B}|}\sum_{\B\in \overline{\B}}\sum_{\x_i\in \B}\langle A_i^t,\beta (G(\bftau_i^t)-\nabla_{\bftau_i}F(\z_{t},\s_i^t)) \rangle\\
        &=\frac{1}{|\overline{\B}|}\sum_{\x_i \in \D}|\overline{\B}_i|\frac{1}{|\overline{\B}_i|}\sum_{\B\in\overline{\B}_i}\langle A_i^t,\beta (G(\bftau_i^t)-\nabla_{\bftau_i}F(\z_{t},\s_i^t)) \rangle\\
        &=\frac{1}{|\overline{\B}|}\sum_{\x_i \in \D}|\overline{\B}_i|\frac{1}{|\overline{\B}_i|}\sum_{\B\in\overline{\B}_i}\left\langle A_i^t,\beta \left(\frac{1}{n}\frac{\bftau_i^t}{\s_i^t}\nabla_{\bftau_i}g_i(\z_t) -\frac{1}{n}\frac{\bftau_i^t}{\s_i^t}\nabla_{\bftau_i}g_i(\z_t,\B)\right) \right\rangle = 0
    \end{aligned}
\end{equation*}
where $\overline{\B}$ denotes the set of all possible batch $\B\subset \D$ of size $B$, and $\overline{\B}_i$ denotes $\{\B:\x_i\in \B, \B\in \overline{\B}\}$.

Furthermore,
\begin{equation*}
\begin{aligned}
\E_t[||\u^{t+1}-\nabla_{\bftau}F(\z_{t})||^2] &\stackrel{(a)}{\leq}  \left(1 + \frac{B\beta}{2n}\right)\E_t[||\u^{t+1}-\nabla_{\bftau}F(\z_{t-1})||^2] + \left(1 + \frac{2n}{B\beta}\right)||\nabla_{\bftau}F(\z_{t-1})-\nabla_{\bftau}F(\z_t)||^2 \\
&\stackrel{(b)}{\leq}  \left(1-\frac{B\beta}{2n}\right)||\nabla_{\bftau} F(\z_{t-1})-\u^t||^2 + \frac{8BL_F^2}{n^2\beta }||\z_t-\z_{t-1}||^2 \\
&\quad+\frac{16B\beta(\tilde{\tau}^2 C_g^2 + \hat{g}^2)}{n^3\hat{g}^4}||\s^t-g(\z_t)||^2 + \frac{2B\tilde{\tau}^2 \sigma^2\beta^2}{n^2\hat{g}^2 B'} + \frac{4L_F^2}{B\beta }||\z_t-\z_{t-1}||^2\\
&\stackrel{(c)}{\leq}  \left(1-\frac{B\beta}{2n}\right)||\nabla_{\bftau} F(\z_{t-1})-\u^t||^2+ \frac{2B\tilde{\tau}^2 \sigma^2\beta^2}{n^2\hat{g}^2 B'}\\
&\quad+\frac{16B\beta(\tilde{\tau}^2 C_g^2 + \hat{g}^2)}{n^3\hat{g}^4}||\s^t-g(\z_t)||^2  + \frac{36L_F^2}{B\beta}||\z_t-\z_{t-1}||^2,
\end{aligned}
\end{equation*}
where we use Young's inequality in (a), and use the assumption $\frac{B\beta}{2n}\leq 1$ in (b) and (c).

Taking summation over $t=1,2,\ldots,T$, we obtain
\begin{equation}
\begin{aligned}
\sum_{t=1}^{T}\E[||\u^{t+1}-\nabla_{\bftau}F(\z_{t})||^2] 
&\leq \frac{2n}{B\beta}\Delta_{\u} + \frac{72nL_F^2}{B^2\beta^2}\sum_{t=1}^{T}\E[||\z_t-\z_{t-1}||^2] \\
&+\frac{32(\tilde{\tau}^2 C_g^2 + \hat{g}^2)}{n^2\hat{g}^4}\sum_{t=1}^{T}\E[||\s^t-g(\z_t)||^2] + \frac{4\tilde{\tau}^2 \sigma^2\beta}{nB'\hat{g}^2}T,
\label{ineq:lemma_grad_error_2_res_sum}
\end{aligned}
\end{equation}
where $\Delta_{\u}$ denotes $||\nabla_{\bftau} F(\z_{0})-\u^1||^2$.

At last, we combine~(\ref{ineq:lemma_grad_error_1_res_sum}) and~(\ref{ineq:lemma_grad_error_2_res_sum}), and establish the following inequality:
\begin{equation*}
\begin{aligned}
&\sum_{t=1}^T\E[||\d_{t+1}-\nabla F(\z_t)||^2] =\sum_{t=1}^T\E[||\v_{t+1}-\nabla_{\w}F(\z_t)||^2] + \sum_{t=1}^T\E[||\u^{t+1}-\nabla_{\bftau}F(\z_t)||^2] \\
&\leq \frac{1}{\beta_1}\Delta_{\v} + \frac{2n}{B\beta}\Delta_{\u} + \left(\frac{4L_F^2}{\beta_1^2}+ \frac{72nL_F^2}{B^2\beta^2}\right)\sum_{t=1}^T\E[||\z_t-\z_{t-1}||^2] \\
&+ \left(\frac{4 C_g^2 L_f^2}{n}+ \frac{32(\tilde{\tau}^2 C_g^2 + \hat{g}^2)}{n^2\hat{g}^4}\right)\sum_{t=1}^T\E[||g(\z_t)-\s^t||^2] + \frac{2\beta_1 C_f^2(C_g^2+\sigma^2)}{\min\{B,B'\}}T + \frac{4\tilde{\tau}^2 \sigma^2\beta }{nB'\hat{g}^2}T \\
&\leq \frac{1}{\beta_1}\Delta_{\v} + \frac{2n}{B\beta}\Delta_{\u} + \left(\frac{4L_F^2}{\beta_1^2}+ \frac{72nL_F^2}{B^2\beta^2}\right)\sum_{t=1}^T\E[||\z_t-\z_{t-1}||^2] \\
&+ \frac{C_1}{n}\sum_{t=1}^T\E[||g(\z_t)-\s^t||^2] + \frac{C_2\beta_1}{\min\{B,B'\}} T + \frac{C_3\beta}{nB'} T,
\end{aligned}
\end{equation*}
where $C_1=\left(4 C_g^2 L_f^2+ \frac{32(\tilde{\tau}^2 C_g^2 + \hat{g}^2)}{\hat{g}^4}\right)$, $C_2=2 C_f^2(C_g^2+\sigma^2)$ and $C_3=\frac{4\tilde{\tau}^2 \sigma^2 }{\hat{g}^2}$.

\end{proof}

\begin{lemma}
Under Assumption~(\ref{ass:1}), run Algorithm~\ref{algo:sogclr_dro} and we have
\begin{equation*}
\begin{aligned}
    \sum_{t=1}^T \E[||\s^t-g(\z_t)||^2]\leq \frac{2n}{B\beta}\Delta_{\s} + \frac{8n^3 C_g^2}{B^2\beta^2}\sum_{t=1}^{T}\E[||\z_t-\z_{t-1}||^2] + \frac{4n\beta\sigma^2 T}{B'}.
\end{aligned}
\end{equation*}
where $\Delta_{\s}$ is a constant defined in the proof.

\label{lemma:s_tracking_error}
\end{lemma}

\begin{proof}
Recall and define the following notations:
\begin{equation*}
\begin{aligned}
&\s_i^{t+1}=\begin{cases}(1-\beta)\s_i^t +\beta g_i(\z_t,\B) \quad & \text{if }\x_i\in \B\\ \s_i^t & \text{o.w.}\end{cases},\quad \tilde{\s}_i^t:=(1-\beta)\s_i^t +\beta g_i(\z_t,\B),\x_i\in\B.
\end{aligned}
\end{equation*}

Then we obtain
\begin{equation*}
\begin{aligned}
   &||\tilde{\s}_i^t - g_i(\z_t)||^2 \\
   = & ||(1-\beta)\s_i^t +\beta g_i(\z_t,\B)- g_i(\z_t) ||^2 \\
=&||(1-\beta)(\s_i^t -g_i(\z_t)) + \beta (g_i(\z_t,\B)-g_i(\z_{t})) ||^2 \\
= & (1-\beta)^2\|(\s_i^t -g_i(\z_t))\|^2   + \beta^2 ||(g_i(\z_t,\B)-g_i(\z_{t}))||^2+2\langle(1-\beta)(\s_i^t -g_i(\z_t)) , \beta (g_i(\z_t,\B)-g_i(\z_{t})) \rangle .
\end{aligned}
\end{equation*}

Considering the randomness of iteration $t$, we have
\begin{equation*}
    \begin{aligned}
        &\E_t[||\s^{t+1}-g(\z_t)||^2] \\
        &= \E_t\left[\sum_{x_i\in \B}||\s_i^{t+1}-g_i(\z_t)||^2+\sum_{x_i\not\in \B}||\s_i^t-g_i(\z_t)||^2\right]\\
        &=\E_t\left[\sum_{x_i\in \B}||\tilde{\s}_i^t-g_i(\z_t)||^2\right] +\frac{n-B}{n}||\s^t-g(\z_t)||^2\\
        &=\E_t\bigg[\sum_{x_i\in \B}(1-\beta)^2\|(\s_i^t -g_i(\z_t))\|^2   + \beta^2 ||g_i(\z_t,\B)-g_i(\z_{t})||^2\\
        &\quad +2\langle(1-\beta)(\s_i^t -g_i(\z_t)) , \beta (g_i(\z_t,\B)-g_i(\z_{t})) \rangle\bigg]  +\frac{n-B}{n}||\s^t-g(\z_t)||^2\\
        &=\frac{B}{n}(1-\beta)||\s^t-g(\z_t)||^2 + \E_t\left[\sum_{x_i\in \B}\beta^2 ||g_i(\z_t,\B)-g_i(\z_{t})||^2\right]\\
        &\quad + \E_t\left[\sum_{x_i\in \B}2\langle(1-\beta)(\s_i^t -g_i(\z_t)) , \beta (g_i(\z_t,\B)-g_i(\z_{t})) \rangle\right]+\frac{n-B}{n}||\s^t-g(\z_t)||^2\\
       &\stackrel{(a)}{\leq}  (1-\frac{B\beta}{n})||\s^t-g(\z_t)||^2 + \frac{B\beta^2\sigma^2}{B'} \\
    \end{aligned}
\end{equation*}

where $(a)$ uses the following facts

\begin{equation*}
    \begin{aligned}
        \E_t\left[\sum_{x_i\in \B}\beta^2|| g_i(\z_t,\B)-g_i(\z_{t})||^2\right]
        &= \frac{1}{|\overline{\B}|}\sum_{\B\in \overline{\B}}\sum_{x_i\in \B}\beta^2|| g_i(\z_t,\B)-g_i(\z_{t})||^2\\
        &=\frac{1}{|\overline{\B}|}\sum_{x_i \in \D}|\overline{\B}_i|\frac{1}{|\overline{\B}_i|}\sum_{\B\in\overline{\B}_i}\beta^2|| g_i(\z_t,\B)-g_i(\z_{t})||^2\\
        &\leq \frac{1}{|\overline{\B}|}\sum_{x_i \in \D}|\overline{\B}_i| \frac{\beta^2\sigma^2}{B'}= \frac{B\beta^2\sigma^2}{B'}
    \end{aligned}
\end{equation*}
and
\begin{equation*}
    \begin{aligned}
        &\E_t\left[\sum_{x_i\in \B}2\langle(1-\beta)(\s_i^t -g_i(\z_t)) , \beta (g_i(\z_t,\B)-g_i(\z_{t})) \rangle\right]\\
        &=\frac{1}{|\overline{\B}|}\sum_{\B\in \overline{\B}}\sum_{x_i\in \B}2\langle(1-\beta)(\s_i^t -g_i(\z_t)) , \beta (g_i(\z_t,\B)-g_i(\z_{t})) \rangle\\
        &=\frac{1}{|\overline{\B}|}\sum_{x_i \in \D}|\overline{\B}_i|\frac{1}{|\overline{\B}_i|}\sum_{\B\in\overline{\B}_i}2\langle(1-\beta)(\s_i^t -g_i(\z_t)) , \beta (g_i(\z_t,\B)-g_i(\z_{t})) \rangle= 0
    \end{aligned}
\end{equation*}
where $\overline{\B}$ denotes the set of all possible batch $\B\subset \D$ of size $B$, and $\overline{\B}_i$ denotes $\{\B:x_i\in \B, \B\in \overline{\B}\}$.

Furthermore, we use Young's inequality and derive the following relationship:
\begin{equation*}
\begin{aligned}
\E_t[||\s^{t+1}-g(\z_{t+1})||^2] &\leq \left(1+\frac{B\beta}{2n}\right)\E_t[||\s^{t+1}-g(\z_t)||^2] +  \left(1+\frac{2n}{B\beta}\right)\E_t[||g(\z_t)-g(\z_{t+1})||^2] \\
&\stackrel{(a)}{\leq} \left(1-\frac{B\beta}{2n}\right)||\s^t -g(\z_t) ||^2  + \frac{2B\beta^2\sigma^2}{B'} + \frac{4n^2C_g^2}{B\beta}\E_t[||\z_t-\z_{t+1}||^2],
\end{aligned}
\end{equation*}
where (a) is due to $\frac{B\beta}{2n}\leq 1$.

Taking expectation over all randomness and taking summation over all $\x_i\in\S$ and $t=1,2,\ldots,T$, we obtain
\begin{equation*}
\begin{aligned}
    \sum_{t=1}^T \E[||\s^t-g(\z_t)||^2]\leq \frac{2n}{B\beta}\Delta_{\s} + \frac{8n^3 C_g^2}{B^2\beta^2}\sum_{t=1}^{T}\E[||\z_t-\z_{t-1}||^2] + \frac{4n\beta\sigma^2 T}{B'}.
\end{aligned}
\end{equation*}
where $\Delta_{\s}$ denotes $||\s^0-g(\z_0)||^2$.

\end{proof}

Now we present the proof for the convergence guarantee of Algorithm~\ref{algo:sogclr_dro}.

\begin{proof}
Now we present the proof for Theorem~\ref{thm:short_stat}. First of all, we establish the following relationship using~(\ref{ineq:lemma_F_smooth_a})
\begin{equation}
\sum_{t=1}^T\E[||\z_t-\z_{t-1}||^2]\leq 8\eta\Delta + 8\eta^2\sum_{t=1}^T\E[||\d_{t+1}-\nabla F(\z_t)||^2],
\label{ineq:main_proof_1}
\end{equation}
where we use $\eta L_F\leq\frac{1}{4}$.

On the other hand, we combine Lemma~(\ref{lemma:gradient_error}) and Lemma~(\ref{lemma:s_tracking_error}), and obtain
\begin{equation*}
\begin{aligned}
&\sum_{t=1}^T\E[||\d_{t+1}-\nabla F(\z_t)||^2] =\sum_{t=1}^T\E[||\v_{t+1}-\nabla_{\w}F(\z_t)||^2] + \sum_{t=1}^T\E[||\u^{t+1}-\nabla_{\bftau}F(\z_t)||^2] \\
&\leq \frac{1}{\beta_1}\Delta_{\v} + \frac{2n}{B\beta}\Delta_{\u} + \left(\frac{4L_F^2}{\beta_1^2}+ \frac{72nL_F^2}{B^2\beta^2}\right)\sum_{t=1}^T\E[||\z_t-\z_{t-1}||^2] \\
&+ \frac{C_1}{n}\bigg[\frac{2n}{B\beta}\Delta_{\s} + \frac{8n^3 C_g^2}{B^2\beta^2}\sum_{t=1}^{T}\E[||\z_t-\z_{t-1}||^2] + \frac{4n\beta\sigma^2 T}{B}\bigg]+ \frac{C_2\beta_1}{\min\{B,B'\}} T + \frac{C_3\beta}{nB'} T\\
&\leq \frac{1}{\beta_1}\Delta_{\v} + \frac{2n}{B\beta}\Delta_{\u}+ \frac{2C_1}{B\beta}\Delta_{\s}+\left(\frac{4L_F^2}{\beta_1^2}+ \frac{72nL_F^2}{B^2\beta^2}+\frac{8C_1n^2 C_g^2}{B^2\beta^2}\right)\sum_{t=1}^T\E[||\z_t-\z_{t-1}||^2]\\
&\quad + \frac{4C_1\beta\sigma^2 T}{B'}+ \frac{C_2\beta_1}{\min\{B,B'\}} T + \frac{C_3\beta}{nB'} T\\
&\leq \frac{1}{\beta_1}\Delta_{\v} + \frac{2n}{B\beta}\Delta_{\u}+ \frac{2C_1}{B\beta}\Delta_{\s}+ \frac{4C_1\beta\sigma^2 T}{B'}+ \frac{C_2\beta_1}{\min\{B,B'\}} T + \frac{C_3\beta}{nB'} T\\
&\quad +\left(\frac{4L_F^2}{\beta_1^2}+ \frac{72nL_F^2}{B^2\beta^2}+\frac{8C_1n^2 C_g^2}{B^2\beta^2}\right)\bigg[8\eta\Delta + 8\eta^2\sum_{t=1}^T\E[||\d_{t+1}-\nabla F(\z_t)||^2]\bigg]
\end{aligned}
\end{equation*}


By setting $\eta^2\leq \min\left\{\frac{\beta_1^2}{192L_F^2}, \frac{B^2\beta^2}{3456nL_F^2}, \frac{B^2\beta^2}{384C_1n^2 C_g^2}\right\}=O(\min\{\beta_1^2,\frac{B^2\beta^2}{n^2}\})$, we have
\begin{equation*}
    \begin{aligned}
        \sum_{t=1}^T\E[||\d_{t+1}-\nabla F(\z_t)||^2] 
        &\leq \frac{1}{\beta_1}\Delta_{\v} + \frac{2n}{B\beta}\Delta_{\u}+ \frac{2C_1}{B\beta}\Delta_{\s}+ \frac{4C_1\beta\sigma^2 T}{B'}+ \frac{C_2\beta_1}{\min\{B,B'\}} T + \frac{C_3\beta}{nB'} T\\
        &\quad +8\eta\Delta\frac{1}{16\eta^2} + \frac{1}{2}\sum_{t=1}^T\E[||\d_{t+1}-\nabla F(\z_t)||^2],
    \end{aligned}
\end{equation*}
which follows
\begin{equation}\label{ineq:main_proof_3}
    \begin{aligned}
        \sum_{t=1}^T\E[||\d_{t+1}-\nabla F(\z_t)||^2] 
        &\leq \frac{2}{\beta_1}\Delta_{\v} + \frac{4n}{B\beta}\Delta_{\u}+ \frac{4C_1}{B\beta}\Delta_{\s}+\frac{\Delta}{\eta}+ T\left(\frac{8C_1\beta\sigma^2 }{B'}+ \frac{2C_2\beta_1}{\min\{B,B'\}}  + \frac{2C_3\beta}{nB'} \right)
    \end{aligned}
\end{equation}

Combining~(\ref{ineq:main_proof_3}) and Lemma~(\ref{lemma:F_smooth}), using assumption $\eta L_F\leq\frac{1}{4}$ we obtain
\begin{equation*}
\begin{aligned}
&\E[\text{dist}(0,\hat{\partial}\bar{F}(\z_R))^2]\\
&\leq  \frac{2+ 40\eta L_F}{T} \sum_{t=1}^T\E[||\d_{t+1}-\nabla F(\z_t)||^2] + \frac{2\Delta}{T\eta} + \frac{40L_F\Delta}{T}\\
&\leq  12\bigg[\frac{1}{T}\left(\frac{2}{\beta_1}\Delta_{\v} + \frac{4n}{B\beta}\Delta_{\u}+ \frac{4C_1}{B\beta}\Delta_{\s}+\frac{3\Delta}{\eta} + 40L_F\Delta\right)+ \frac{8C_1\beta\sigma^2}{B'}+ \frac{2C_2\beta_1}{\min\{B,B'\}} + \frac{2C_3\beta}{nB'} \bigg]\\
\end{aligned}
\end{equation*}

By setting $\beta\leq \min\{\frac{B'\epsilon^2}{288C_1\sigma^2},\frac{nB'\epsilon^2}{72C_3}\}= O(B'\epsilon^2)$, $\beta_1\leq \frac{\min\{B,B'\}\epsilon^2}{72C_2}= O(B'\epsilon^2)$, we have
\begin{equation*}
    \begin{aligned}
        \E[\text{dist}(0,\hat{\partial}\bar{F}(\z_R))^2]\leq 12\bigg[\frac{1}{T}\left(\frac{2}{\beta_1}\Delta_{\v} + \frac{4n}{B\beta}\Delta_{\u}+ \frac{4C_1}{B\beta}\Delta_{\s}+\frac{3\Delta}{\eta} + 40L_F\Delta\right)\bigg] +\frac{2\epsilon^2}{3}.
    \end{aligned}
\end{equation*}
It implies that with 
\begin{equation*}
\begin{aligned}
    T&=\max\left\{\frac{360\Delta_{\v}}{\beta_1 \epsilon^2}, \frac{720n\Delta_{\u}}{B\beta \epsilon^2}, \frac{720C_1\Delta_{\s}}{B\beta \epsilon^2}+\frac{540\Delta}{\eta \epsilon^2} + \frac{7200L_F\Delta}{\epsilon^2}\right\}\\
    &= O\left(\max\left\{\frac{1}{\beta_1 \epsilon^2}, \frac{n}{B\beta \epsilon^2}\right\}\right)\\
    &=O\left( \frac{n}{BB' \epsilon^4 }\right)
\end{aligned}
\end{equation*}
we have $\E[\text{dist}(0,\hat{\partial}\bar{F}(\z_R))^2]\leq \epsilon^2$, which completes the proof.


\end{proof}


\end{document}